%% file: main.tex
\documentclass{article}

\PassOptionsToPackage{numbers, compress}{natbib}
\usepackage[final]{nips/neurips_2019}

\usepackage[utf8]{inputenc} % allow utf-8 input
\usepackage[T1]{fontenc}    % use 8-bit T1 fonts
\usepackage{hyperref}       % hyperlinks
\usepackage{url}            % simple URL typesetting
\usepackage{booktabs}       % professional-quality tables
\usepackage{amsfonts}       % blackboard math symbols
\usepackage{nicefrac}       % compact symbols for 1/2, etc.
\usepackage{microtype}      % microtypography
\usepackage{amsmath,bm}
\usepackage{amsthm}
\usepackage{amsmath,amsfonts,amssymb}
\usepackage{algorithm}% http://ctan.org/pkg/algorithms
\usepackage{algpseudocode}% http://ctan.org/pkg/algorithmicx
\usepackage{graphicx}
\usepackage{subfig}
\usepackage{listings}
\usepackage{xcolor}
\usepackage{mathrsfs}
\usepackage{mathtools}
\usepackage{csvsimple}
\usepackage{tabularx}
\newtheorem{theorem}{Theorem}[section]
\newtheorem*{remark}{Remark}
\newtheorem{lemma}[theorem]{Lemma}
\newtheorem*{problem*}{Problem}
\usepackage{thmtools,thm-restate} %
\usepackage{scalerel,stackengine}
\usepackage{booktabs}
\usepackage{verbatim}
\usepackage{wrapfig}

\algblock{Input}{EndInput}
\algnotext{EndInput}
\algblock{Output}{EndOutput}
\algnotext{EndOutput}

\stackMath
\newcommand\bighat[1]{%
\savestack{\tmpbox}{\stretchto{%
  \scaleto{%
    \scalerel*[\widthof{\ensuremath{#1}}]{\kern-.6pt\bigwedge\kern-.6pt}%
    {\rule[-\textheight/2]{1ex}{\textheight}}%WIDTH-LIMITED BIG WEDGE
  }{\textheight}%
}{0.5ex}}%
\stackon[1pt]{#1}{\tmpbox}%
}
\newcommand{\E}{\mathbb{E}}
\newcommand{\Var}{\mathrm{Var}}
\newcommand{\Cov}{\mathrm{Cov}}
\newcommand{\Prob}{\mathbb{P}}
\newcommand{\ksd}{\mathrm{KSD}}
\newcommand{\mmd}{\mathrm{MMD}}

\newcommand{\relmmd}{\textrm{RelMMD}}
\newcommand{\relksd}{\textrm{RelKSD}}
\newcommand{\relpsiksd}{\textrm{RelPSI-KSD}}
\newcommand{\relpsimmd}{\textrm{RelPSI-MMD}}
\newcommand{\relpsi}{\textrm{RelPSI}}
\newcommand{\relmul}{\textrm{RelMulti}}
\newcommand{\relmulksd}{\textrm{RelMulti-KSD}}
\newcommand{\relmulmmd}{\textrm{RelMulti-MMD}}

\DeclareMathOperator{\Tr}{tr}
\DeclarePairedDelimiter\floor{\lfloor}{\rfloor}
\newenvironment{proof-idea}{\noindent{\bf Proof Idea}\hspace*{1em}}{\qed\bigskip}

\newcommand{\wjsay}[1]{[\textcolor{red}{\textbf{WJ: }}\textcolor{red!60!black}{#1}]}
\newcommand{\jlsay}[1]{[\textcolor{blue}{\textbf{JL: }}\textcolor{blue!60!black}{#1}]}

\newcommand{\ourtitle}{Kernel Stein Tests for Multiple Model Comparison}

\title{\ourtitle}

\author{%
  Jen Ning Lim \\
  Max Planck Institute for Intelligent Systems\\
  {\small \url{jlim@tuebingen.mpg.de} }\\
  \And
  Makoto Yamada \\    
  Kyoto University, RIKEN AIP\\
  {\small \url{makoto.yamada@riken.jp} }\\
  \AND
  Bernhard Sch\"{o}lkopf \\
  Max Planck Institute for Intelligent Systems \\
  {\small \url{bs@tuebingen.mpg.de} }\\
  \And
  Wittawat Jitkrittum \\
  Max Planck Institute for Intelligent Systems \\
  {\small \url{wittawat@tuebingen.mpg.de} } \\
}

\begin{document}

\maketitle

\begin{abstract}
%Modern deep generative models have been shown capable of modelling realistic
%complex data. These advances have rendered many earlier methods of evaluation
%and selection obsolete. 
%Common methods have been based on comparing score functions (e.g. Fr\'{e}chet
%Inception Distance), while procedures based on statistical tests are limited to
%only two candidate models.

%We propose a non-parametric statistical test for model comparison
%between an arbitrary number of candidate models. We consider the problem as a
%test of similarity between each model and the best. The algorithm first selects
%the “best” model via Maximum Mean Discrepancy (MMD) or Kernelized Stein
%Discrepancy (KSD). Then, we test if the selected model is statistically better
%than all other candidates. We employ the framework of post selection inference
%to generate random hypotheses based on data and construct valid test statistics
%after selection. In experiments, we begin with toy experiments with two
%candidate models to demonstrate that the traditional level-$\alpha$ behaviour
%is preserved. Then we move toward a larger number of candidates and evaluate
%our method on synthetic models, a mixture of GAN generated images and real
%images, and finally a comparison of various density estimation models such as
%Gaussian mixture models and invertible generative models.

We address the problem of non-parametric multiple model comparison: given $l$
candidate models, decide whether each candidate is as good as the best one(s) or worse than it.  We propose two statistical tests,
each controlling a different notion of decision errors. The first test,
building on the post selection inference framework, provably controls the
number of best models that are wrongly declared worse (false positive
rate). The second test is based on multiple correction, and controls the
proportion of the models declared worse but are in fact as good as the best
(false discovery rate). 
We prove that under appropriate conditions the first test can yield a higher true
positive rate than the second. Experimental results on toy and real (CelebA,
Chicago Crime data) problems show that the two tests have high true positive
rates with well-controlled error rates. By contrast, the naive approach of
choosing the model with the lowest score  without correction
leads to more false positives.
\end{abstract}
\section{Introduction}
\input{intro}
\section{Background}
\label{sec:background}
Hypothesis testing of relative fit between $l=2$ candidate models, $P_1$ and
$P_2$, to the data generating distribution $R$ (unknown) can be performed by comparing the
relative magnitudes of a pre-chosen discrepancy measure which computes the distance from each of the two
models to the observed sample drawn from $R$. 
%
\begin{comment}
Let $D$ be a discrepancy measure 
such that for $i=1,\ldots, l$, $D(P_i, R) \ge 0$ measures a divergence from $P_i$ to $R$, and is 0 if
and only if $P_i = R$. A sample estimate of $D(P_i, R)$ is denoted by $\hat{D}(P_i, R) := \hat{D}_i$.
%
The relative tests of \cite{bounliphone2015test,jitkrittum2018informative}
propose a null hypothesis $H_0:D(P_1,R) \le D(P_2,R)$ (i.e., $P_1$ is better than or
as good as $P_2$) against an alternative $H_1:D(P_1,R) > D(P_2,R)$ (i.e., $P_2$ is better). 
%This approach was taken by Bounliphone et. al.  \cite{bounliphone2015test} and
%Jitkrittum et al \cite{jitkrittum2018informative}. 
However, this formulation is limited to $l=2$ candidates. 
\end{comment}
Our proposed
methods \relpsi{} and \relmul{} (described in Section \ref{sec:psi}) generalize this formulation 
based upon selective testing \cite{lee2016exact}, and multiple correction
\cite{benjamini2001control}, respectively. 
%Similarly to the previous relative tests, 
Underlying these new tests is a base
discrepancy measure $D$ for measuring the distance between each candidate model
to the observed sample. In this section, we review Maximum Mean Discrepancy
(MMD, \cite{gretton2012kernel}) and Kernel Stein Discrepancy (KSD,
        \cite{chwialkowski2016kernel,liu2016kernelized}), which will be used as
        a base discrepancy measure in our proposed tests in Section
    \ref{sec:psi}. 
    %We first start by describing reproducing kernel Hilbert spaces.

%Similarly to \cite{bounliphone2015test,
%jitkrittum2018informative}, the test statistics is also based upon the
%difference between two discrepancies measure, we consider the Maximum Mean
%Discrepancy (\textrm{MMD}) and Kernelized Stein Discrepancy (\textrm{KSD}).

%-------------
%Let $\mathcal{H}_k=\mathcal{H}$ be a Reproducing Kernel Hilbert Space (RKHS) associated
%with the reproducing kernel $k:\mathcal{X} \times \mathcal{X} \rightarrow
%\mathbb{R}$. It is known that 
\textbf{Reproducing kernel Hilbert space}
Given a positive definite kernel $k:\mathcal{X} \times \mathcal{X} \rightarrow
\mathbb{R}$,  it is known that there exists a feature map $\phi\colon
\mathcal{X} \to \mathcal{H}$ and a reproducing kernel Hilbert Space (RKHS)
$\mathcal{H}_k=\mathcal{H}$ 
%(equipped with an inner product) 
%\wjsay{A Hilbert space is a complete inner product space, by definition.
associated with the kernel $k$ \cite{BerTho2004}. The kernel $k$ is symmetric and is
a reproducing kernel on $\mathcal{H}$ in the sense that $k(x,y)=\left< \phi(x),
\phi(y) \right>_\mathcal{H}$ for all $x, y \in \mathcal{X}$ where
$\left<\cdot,\cdot \right>_\mathcal{H} = \left<\cdot,\cdot \right> $ denotes
the inner product. It follows from this reproducing property that for any $f
\in \mathcal{H}$, $\left< f, \phi(x) \right> = f(x)$ for all $x \in
\mathcal{X}$. We interchangeably write $k(x, \cdot)$ and $\phi(x)$.

\textbf{Maximum Mean Discrepancy}
Given a distribution $P$ and a positive definite kernel $k$, the mean embedding
of $P$, denoted by $\mu_P$, is defined as $\mu_P = \mathbb{E}_{x \sim
P}[k(x,\cdot)]$ \cite{smola2007hilbert} (exists if $\mathbb{E}_{x \sim P}[\sqrt{k(x,x)}] < \infty$).
Given two distributions $P$ and $R$, the Maximum Mean Discrepancy (MMD,
\cite{gretton2012kernel}) is a pseudometric defined as $\mathrm{MMD}(P,R) := || \mu_P -
\mu_R ||_\mathcal{H}$ and $\| f \|^2_\mathcal{H} = \left< f, f
\right>_\mathcal{H}$ for any $f \in \mathcal{H}$.
If the kernel $k$ is characteristic 
\cite{sriperumbudur2011universality,fukumizu2008kernel}, then $\mathrm{MMD}$ defines a metric. An important implication is that $\mathrm{MMD}^2(P,R)= 0 \iff P=R$.
%Then, the Maximum Mean Discrepancy (\textrm{MMD}), a distance between two
%distributions $P$ and $Q$, is defined as $\mathrm{MMD}^2(P,R) = || \mu_P -
%\mu_R ||^2_\mathcal{H}$.  
%It is an integral probability metric \cite{muller1997integral}
%where the set of functions is defined by an   and its associated kernel
Examples of characteristic kernels include the Gaussian and Inverse
multiquadric (IMQ) kernels \cite{steinwart2001influence,gorham2017measuring}.
It was shown in \cite{gretton2012kernel} that $\mathrm{MMD}^2$ can be written as $\mmd^2(P,R) =
\mathbb{E}_{z,z'\sim P\times R}[h(z,z')]$ where $h(z,z') =
k(x,x')+k(y,y')-k(x,y')-k(x',y)$ and $z:=(x,y), z':=(x',y')$ are independent
copies. This form admits an unbiased estimator $\bighat{\mmd}^2_u=
\frac{1}{n(n-1)}\sum_{i\neq j}h(z_i,z_j)$ where $z_i := (x_i, y_i)$, $\{ x_i
\}_{i=1}^n \stackrel{i.i.d.}{\sim} P, \{ y_i \}_{i=1}^n \stackrel{i.i.d.}{\sim}
Q$ and is a second-order U-statistic \cite{gretton2012kernel}. Gretton
et al. \cite[Section 6]{gretton2012kernel} proposed a linear-time estimator 
$\widehat{\mathrm{MMD}}^2_{l}=\frac{2}{n}\sum_{i=1}^{n/2}h(z_{2i}, z_{2i-1})$ which
can be shown to be asymptotically normally distributed both when $P=R$ and $P\neq R$ 
\cite[Corollary 16]{gretton2012kernel}.
Notice that the MMD can be estimated solely on the basis of two independent
samples from the two distributions.

% A linear time estimato
% taking the form of an incomplete U-statistic $\bighat{\mmd}_l^2 = \frac{2}{n} \sum_{i=1}^{n/2} h_p(z_i,z_{i+1})$.

%-------------
\textbf{Kernel Stein Discrepancy} 
The Kernel Stein Discrepancy (KSD,
\cite{liu2016kernelized,chwialkowski2016kernel}) is a discrepancy measure
between an unnormalized, differentiable density function $p$ and a sample,
originally proposed for goodness-of-fit testing. Let $P,R$ be two distributions
that have continuously differentiable density functions $p,r$ respectively. Let
$\boldsymbol{s}_p(x) := \nabla_x \log p(x)$ (a column vector) be the score function of $p$
defined on its support. Let $k$ be a positive definite kernel with continuous
second-order derivatives.  Following
\cite{liu2016kernelized,jitkrittum2017linear}, define 
%$\boldsymbol{\xi}_{p}(\mathbf{x},\cdot):=\frac{\partial\log
%p(\mathbf{x})}{\partial\mathbf{x}}k(\mathbf{x},\cdot)+\frac{\partial
%k(\mathbf{x},\cdot)}{\partial\mathbf{x}}$ 
${\boldsymbol{\xi}}_{p}(x,\cdot):= \boldsymbol{s}_p(x) k(x,\cdot)+ \nabla_x k(x,\cdot)$ 
which is an element in
$\mathcal{H}^d$ that has an inner product defined as
$\langle f,g \rangle_{\mathcal{H}^d}=\sum_{i=1}^d \langle f_i,g_i \rangle_{\mathcal{H}}$. 
The Kernel Stein Discrepancy is defined as 
$\mathrm{KSD}^2(P, R) := \| \mathbb{E}_{x\sim R} \boldsymbol{\xi}_p(x, \cdot)
\|^2_{\mathcal{H}^d}$. Under appropriate boundary conditions on $p$ and
conditions on the kernel $k$ \cite{chwialkowski2016kernel,liu2016kernelized},
it is known that $\mathrm{KSD}^2(P,R) = 0 \iff P=R$. Similarly to the case of
MMD, the squared KSD can be written as 
$\mathrm{KSD}^2(P,R) =
\mathbb{E}_{x,x'\sim R}[u_p(x,x')]$ where 
$u_p(x,x') 
= \left< \boldsymbol{\xi}_p(x, \cdot), \boldsymbol{\xi}_p(x', \cdot) \right>_{\mathcal{H}^d}
= \mathbf{s}_p(x)^\top \mathbf{s}_p(x')k(x,x') 
 +\mathbf{s}_p(x)^\top\nabla_{x'}k(x,x')
 + \nabla_xk(x,x')^\top \mathbf{s}_p(x')+\Tr[\nabla_{x,x'}k(x,x')]$.
The KSD has an unbiased estimator 
$\bighat{\ksd}^2_u(P,R) = \frac{1}{n(n-1)}\sum_{i\neq j}
u_p(x_i,x_j)$ where $\{ x_i \}_{i=1}^n \stackrel{i.i.d.}{\sim} R$, which is also
a second-order U-statistic. Like the MMD, a linear-time estimator of $\textrm{KSD}^2$
is  given by
$\widehat{\mathrm{KSD}}^2_{l}=\frac{2}{\floor{n}}\sum_{i=1}^{\floor{n}/2}u_p(x_{2i},
x_{2i-1})$. It is known that $\sqrt{n}\widehat{\mathrm{KSD}}^2_{l}$ is
asymptotically normally distributed \cite{liu2016kernelized}.
In contrast to the MMD estimator, the KSD
estimator requires only samples from $R$, and $P$ is represented by its
score function $\nabla_x \log p(x)$ which is independent of the normalizing
constant. 
As shown in the previous work, an explicit probability
density function is far more representative of the distribution than its
sample counterpart \cite{jitkrittum2017linear,jitkrittum2018informative}. KSD is suitable when the candidate models are given
explicitly (i.e., known density functions), whereas MMD is more suitable when
the candidate models are implicit and better represented by their samples.
 %agreeing
 %with our simulation results in Section \ref{sec:experiments}. Our general
 %recommendation is that when the model has a density, use KSD; if the model
 %has to be represented by its sample (e.g., a GAN model whose density is not
 %available in closed form), then use MMD.
 
% which has a computation time of $\mathcal{O}(n^2)$. To reduce its
%complexity to linear time an incomplete U-statistic is proposed
%\cite{liu2016kernelized} given by $\bighat{\ksd}_l^2 = \frac{2}{n}
%\sum_{i=1}^{n/2} h_p(x_i,x_{i+1})$ which typically has worse performance than
%its quadratic time variant.

 \vspace{-3mm}

\section{Proposal: non-parametric multiple model comparison}
\vspace{-2mm}
\label{section:proposal}
\input{psi}
\label{sec:psi}

\section{Performance analysis}
\vspace{-1mm}

\input{performance}
\section{Experiments}
\label{sec:experiments}

\input{experiments}

%--- We should remove the following section when we submit.
%\newpage

\subsubsection*{Acknowledgments}
M.Y. was supported by the
JST PRESTO program JPMJPR165A and partly supported by MEXT
KAKENHI 16H06299 and the RIKEN engineering network funding.
%Use unnumbered third level headings for the acknowledgments. All acknowledgments
%go at the end of the paper. Do not include acknowledgments in the anonymized
%submission, only in the final paper.

\bibliography{ref}
\bibliographystyle{plain}

% ------------ appendix ---------------
\clearpage

\onecolumn
\appendix

\begin{center}
{\LARGE{}{}{}{}{}{}\ourtitle{}} 
\par\end{center}

\begin{center}
\textcolor{black}{\Large{}{}{}{}{}{}Supplementary}{\Large{}{}{}{}{}
} 
\par\end{center}

\input{sec_appendix}

\end{document}

%% file: intro.tex
Given a sample (a set of i.i.d.\ observations), and a set of $l$ candidate models $\mathcal{M}$, we address
the problem of non-parametric comparison of the relative fit of these candidate
models. The comparison is non-parametric in the sense that the class of allowed
candidate models is broad (mild assumptions on the models). All the given
candidate models may be wrong; that is, the true data
generating distribution may not be present in the candidate list. A widely used
approach is to pre-select a divergence measure which computes a distance
between a model and the sample (e.g., Fr\'{e}chet Inception Distance (FID,
\cite{heusel2017gans}), Kernel Inception Distance \cite{BinSutArbGre2018} or
others), and choose the model which gives the lowest estimate of the
divergence. An issue with this approach is that multiple equally good models
may give roughly the same estimate of the divergence, giving a wrong conclusion
of the best model due to noise from the sample (see Table~1 in
\cite{jitkrittum2018informative} for an example of a misleading conclusion
resulted from direct comparison of two FID estimates).

It was this issue that motivates the development of a non-parametric hypothesis test of relative
fit (\relmmd) between two candidate models \cite{bounliphone2015test}. The
test uses as its test statistic the difference of two estimates of Maximum Mean
Discrepancy (MMD, \cite{gretton2012kernel}), each measuring the distance between
the generated sample from each model and the observed sample. It is known that
if the kernel function used is characteristic
\cite{sriperumbudur2011universality,fukumizu2008kernel}, the population MMD
defines a metric on a large class of distributions. As a result, the magnitude
of the relative test statistic provides a measure of relative fit, allowing one
to decide a (significantly) better model when the statistic is sufficiently large. The 
key to avoiding the
previously mentioned issue of false detection is to appropriately choose the
threshold based on the null distribution, i.e., the distribution of the
statistic when the two models are equally good. 
An extension of \relmmd{} to a linear-time relative test was considered by Jitkrittum et al.\
\cite{jitkrittum2018informative}. 

A limitation of the relative tests of RelMMD and others \cite{bounliphone2015test,
jitkrittum2018informative} is that they are limited to the comparison of only $l=2$
candidate models. Indeed, taking the difference is inherently a function of two
quantities, and it is unclear how the previous relative tests can be applied when there are
$l>2$ candidate models.
We note that relative fit testing is different from goodness-of-fit
testing, which aims to decide whether a given model is the true distribution of a set of
observations. The latter task may be achieved with the Kernel Stein Discrepancy
(KSD) test \cite{chwialkowski2016kernel,liu2016kernelized,gorham2017measuring}
where, in the continuous case, the model is specified as a probability density function and needs only
be known up to the normalizer. A discrete analogue of the KSD test is studied in
\cite{YanLiuRaoNev2018}. When the model is represented by its sample,
goodness-of-fit testing reduces to two-sample testing, and may be carried out
with the MMD test \cite{gretton2012kernel}, its incomplete U-statistic
variants \cite{zaremba2013b,yamada2018post}, the ME and SCF tests
\cite{chwialkowski2015fast,JitSzaChwGre2016}, and related kernel-based tests
\cite{EriBacHar2008,FroLerReyoth2012}, among others. 
%In the scenario where we have
%$l$ candidates, again we stress that multiple model comparison differs from
%multiple goodness-of-fit tests. 
% \wjsay{This sounds like the two cases only differ when we have l models?}
To reiterate, we stress that in general multiple model comparison differs from
multiple goodness-of-fit tests.
While the latter may be addressed with $l$
individual goodness-of-fit tests (one for each candidate), the former requires
comparing $l$ correlated estimates of the distances between each model and the
observed sample. The use of the observed sample in the $l$ estimates
is what creates the correlation which must be accounted for.

In the present work, we generalize the relative comparison
tests of RelMMD and others
\cite{bounliphone2015test, jitkrittum2018informative} to the case of $l>2$ models.
The key idea is to select the ``best'' model (reference model) that is the closest match to the
%\wjsay{We need to be consistent with the word ``sample''. I can be a set, or one example. In the first paragraph, we use it as a set.}
observed sample, and consider $l$ hypotheses. Each hypothesis
tests the relative fit of each candidate model with the reference
model, where the reference is chosen to be the model giving the lowest estimate
of the pre-chosen divergence measure (MMD or KSD). 
The total output thus consists of $l$ binary values where 1 (assign positive)
indicates that the corresponding model is significantly worse (higher
divergence to the sample) than the reference, and 0 indicates no evidence for
such claim (indecisive). We assume that the output is always 0 when the reference model is
compared to itself.
\begin{comment}
A related problem is multiple 
goodness-of-fit testing considered in \cite[Section 3.4]{yamada2018post}. 
Multiple goodness-of-fit testing seeks to identify the
presence of the true data distribution in the candidate list (zero divergence
wrt. the data distribution), whereas our problem of multiple model comparison
seeks to identify the model that has the best fit (the lowest divergence,
potentially non-zero).
\end{comment}
The need for a reference model greatly complicates the formulation of the null hypothesis 
%since this model is typically unknown to us and requires estimation from data 
%\wjsay{This sounds like we need to fit a model using the data. It is the model index that is unknown?}
(i.e., the null hypothesis is random due to
the noisy selection of the reference), an issue that is not present in the
multiple goodness-of-fit testing. 
%\jlsay{I used reference to describe the underlying data distribution, might need to distinguish between reference model and reference distribution} 
%\wjsay{I normally call it observed sample. Okay. I will fix the intro to match yours.}

We propose two non-parametric multiple model comparison tests (Section \ref{sec:psi}) following the
previously described scheme. Each test controls a different notion of decision
errors. The first test \relpsi{} builds on the post selection inference
framework and provably (Lemma \ref{lemma:fpr}) controls the number of best models that are wrongly
declared worse (FPR, false positive rate). The second test \relmul{} is based on multiple
correction, and controls the proportion of the models declared worse but are in
fact as good as the best (FDR, false discovery rate). 
In both tests, the underlying divergence measure can be
chosen to be either the Maximum Mean Discrepancy (MMD) allowing each model to
be represented by its sample, or the Kernel Stein Discrepancy (KSD) allowing
the comparison of any models taking the form of unnormalized, differentiable density functions. 

As theoretical contribution, the asymptotic null distribution of \relmulksd{}
(\relmul{} when the divergence measure is KSD)
is provided (Theorem \ref{theorem:ksd}), giving rise to a relative KSD test in
the case of $l=2$ models, as a special case. To our knowledge, this is the first time that a
KSD-based relative test for two models  is studied.  Further, we show (in
Theorem \ref{theorem:tpr}) that the \relpsi{} test can
yield a higher true positive rate (TPR) than the \relmul{} test, under appropriate conditions. 
%The statement holds true regardless of the underlying divergence (MMD or KSD). 
Experiments (Section \ref{sec:experiments}) on toy
and real  (CelebA, Chicago Crime data) problems show that the two proposed
tests have high true positive rates with well-controlled respective error rates
-- FPR for \relpsi{} and FDR for \relmul{}. By contrast, the naive approach of
choosing the model with the lowest divergence without correction leads to more false positives.

%% file: psi.tex
In this section, we propose two new tests: \relmul{} (Section \ref{sec:relmul})
and \relpsi{} (Section \ref{sec:relpsi}), each controlling a different notion
of decision errors.  
%We first start by formally defining the multiple model
%comparison problem.
% %
% \textit{Post selection inference} (PSI) aimst to addresses the issues of performing inference
% after selection \cite{taylor2015statistical}. Having mined a set of data, how do we properly assess the
% strength of our inferences? To be specific, in the model comparison, after
% selecting our model, how do we account for the selection bias in our data?
% There are two methods around this issue \cite{berk2013valid}, we follow in the 
% example of \cite{lee2016exact} by conditioning on the selection event. An example
% of a commonplace PSI algorithm is splitting the data set \cite{cox1975note}.
\begin{problem*}[Multiple Model Comparison]
    Suppose we have $l$ models denoted as $\mathcal{M} = \{P_i\}_{i=1}^l$, 
    %each
    %can be represented by either a sample (a collection of $n$ realizations$),
    %or an unnormalized log density $\log p(x)$. 
which we
can either: draw a  sample (a collection of $n$ i.i.d. realizations) from or
have access to their unnormalized log density $\log p(x)$. 
%
%The goal is to find the model $P \in \mathcal{M}$ which
%is the best relative fit to $R$ (represented by $n$ samples) i.e. the
%distribution closest to $R$.
The goal is to decide whether each candidate $P_i$ is worse than the best
one(s) in the candidate list (assign positive), or indecisive (assign zero).
The best model is defined to be $P_J$  such that  $J \in \arg \min_{j
\in \{1,\ldots, l\}} D(P_j, R)$ where $D$ is a base discrepancy measure (see
Section \ref{sec:background}), and $R$ is the data generating distribution
(unknown). 
%find the model $P \in \mathcal{M}$ which
%is the best relative fit to $R$ (represented by $n$ samples) i.e. the
%distribution closest to $R$.
\end{problem*}
%  We denote the random variable $J$ as the index of the selected model. Then, we test the hypothesis that for every $i \neq J$, $H^J_{0,i}: D(P_i,R) < D(P_J,R)$ at a significance level $\alpha$.
% Rejecting the null will amount to significant evidence that $H^J_{1,j}: D(P_i,R) < D(P_J,R)$ that model $P_J$ is indeed the better the model $P_i$.
Throughout this work, we assume that all candidate models $P_1, \ldots, P_l$
and the unknown data generating distribution $R$ have a common support
$\mathcal{X} \subseteq \mathbb{R}^d$, and are all distinct.
The task can be seen as a multiple binary decision making task, where a model $P \in \mathcal{M}$
is considered negative if it is as good as the best one, i.e., $D(P, R) = D(P_J, R)$ where $J \in \arg \min_j
D(P_j, R)$. The index set of all models which are as good as the best one is
denoted by $\mathcal{I}_{-} := \{ i \mid D(P_i, R) = \min_{j=1,\ldots, l}
D(P_j, R)\}$. 
%Note that it is possible that $|\mathcal{I}_-| > 1$ even when all
%models are distinct. 
When $|\mathcal{I}_-| > 1$, $J$ is an arbitrary index
in $\mathcal{I}_-$.  Likewise, a model is considered positive if it is worse
than the best model. Formally, the index set of all positive models is denoted
by
$\mathcal{I}_+ := \{ i \mid D(P_i, R) > D(P_J, R) \}$. It follows that 
$\mathcal{I}_- \cap \mathcal{I}_+ = \emptyset$ and 
$\mathcal{I}_- \cup \mathcal{I}_+ = \mathcal{I} := \{ 1,\ldots, l\}$. The problem
can be equivalently stated as the task of deciding whether the index for each
model belongs to $\mathcal{I}_+$ (assign positive). The total output thus
consists of $l$ binary values where 1 (assign positive) indicates that the
corresponding model is significantly worse (higher divergence to the sample)
than the best, and 0 indicates no evidence for such claim (indecisive). 
In practice, there are two difficulties: firstly, $R$ can only be observed
through a sample $X_n := \{ x_i \}_{i=1}^n \stackrel{i.i.d.}{\sim} R$ so that
$D(P_i, R)$ has to be estimated by $\hat{D}(P_i, R)$ computed on the
sample; secondly, the index $J$ of the reference model (the best model)  is
unknown. 
In our work, we consider the complete, and linear-time U-statistic
estimators of \textrm{MMD} or \textrm{KSD} as the discrepancy $\hat{D}$ (see Section \ref{sec:background}). 

%\wjsay{moved here. Still need to work on the text.}
%\jlsay{Maybe this should go here?}
We note that the main assumption on the discrepancy $\hat{D}$ is that
$\sqrt{n}(\hat{D}(P_i, R)  - \hat{D}(P_j, R)) \xrightarrow{d}
\mathcal{N}(\mu,\sigma^2)$ for any $P_i, P_j \in \mathcal{M}$ and $i\neq j$. If this holds,
%and there is a method for estimating the variance $\sigma^2$ (or the covariance matrix)
our proposal can be easily amended to accommodate a new discrepancy measure $D$ beyond MMD or KSD. 
%and, as a result it can utilize a wide range of discrepancies and its estimators than what is considered here. 
Examples
include (but not limited to) the Unnormalized Mean Embedding
\cite{chwialkowski2015fast, jitkrittum2018informative}, Finite Set Stein
Discrepancy \cite{jitkrittum2017linear, jitkrittum2018informative}, or other
estimators such as the block \cite{zaremba2013b} and incomplete estimator
\cite{yamada2018post}. 
%------------------------------------
\subsection{Selecting a reference candidate model}
\label{sec:selection}
In both proposed tests, the algorithms start by first choosing a model
$P_{\hat{J}} \in \mathcal{M}$ as the reference model where $\hat{J} \in
\arg \min_{j \in \mathcal{I}} \hat{D}(P_j, R)$ is a random variable.
The algorithms then proceed to test the relative fit of each model $P_i$ for $i\neq \hat{J}$ and
determine if it is statistically worse than the selected reference $P_{\hat{J}}$. 
The null and the alternative hypotheses for the $i^{th}$ candidate model can be written as 
% Then the algorithm tests for all $i\neq
%\hat{J}$, if the model $P_i$ is statistically worse than $P_{\hat{J}}$.
%
%We assume that $P_{\hat{J}}$ is one of the best models and perform $l-1$ tests for
%the remaining candidate models. If all $l-1$ models are rejected, it amounts to
%$P_{\hat{J}}$ being the best by a significant margin \cite[Theorem
%1]{berger1982multiparameter}. 
%In reality, this will rarely occur as it is
%unlikely that a model is better than every other candidate in all aspects.
%This setup can written as,
%
%\begin{align*}
%     H^{\hat{J}}_{0,i}\colon & D(P_{\hat{J}},R) \ge D(P_i,R)\ |\ P_{\hat{J}}
%        \text{ is selected as the reference,} \\
%    H^{\hat{J}}_{1,i}\colon & D(P_{\hat{J}},R) < D(P_i,R)\ |\ P_{\hat{J}}
%        \text{ is selected as the reference.}
%\end{align*}
\begin{align*}
     H^{\hat{J}}_{0,i}\colon &D(P_i,R) -  D(P_{\hat{J}},R) \le  0
      \ |\ P_{\hat{J}}  \text{ is selected as the reference,} \\
    H^{\hat{J}}_{1,i}\colon & D(P_i,R) - D(P_{\hat{J}},R) > 0 
      \ |\ P_{\hat{J}}  \text{ is selected as the reference.}
\end{align*}
% \wjsay{Say that testing this hypothesis is equivalent to (rewrite it in terms of $\eta, A, b$).
% This will be a kind of prequel to the polyhedral lemma already.}
These hypotheses are conditional on the selection event (i.e., selecting
$\hat{J}$ as the reference index).
For each of the $l$ null hypotheses, the test uses as its statistics
$\bm{\eta}^\top\bm{z} :=
\sqrt{n}[\hat{D}(P_i,R)-\hat{D}(P_{\hat{J}},R)]$ where 
$\bm\eta=
[0,\cdots,\underbrace{-1}_{\hat{J}},\cdots,\underbrace{1}_i,\cdots]^\top$ and
$\bm{z}=\sqrt{n} [\hat{D}(P_1,R),  \cdots,\hat{D}(P_l,R)]^\top$.
%Recall that $\hat{D}$ can be either an estimator of $\mathrm{MMD}^2$ or $\mathrm{KSD}^2$.
%
The distribution of the test statistic
$\bm{\eta}^\top\bm{z}$ depends on the choice of estimator for the
discrepancy measure $\hat{D}$ which can be $\bighat{\mmd}^2_u$ or
$\bighat{\ksd}^2_u$. Define $\bm{\mu} := [D(P_1, R), \ldots, D(P_l, R)]^\top$,
then the hypotheses above can be equivalently expressed as 
$H^{\hat{J}}_{0,i}: \bm{\eta}^\top \bm{\mu} \le 0  \mid \bm{A}\bm{z} \le \bm{0}$ vs.
    $H^{\hat{J}}_{1,i}: \bm{\eta}^\top \bm{\mu} > 0   \mid \bm{A}\bm{z} \le \bm{0}$, 
%\begin{align*}
%    H^{\hat{J}}_{0,i}: \bm{\eta}^\top \bm{\mu} \le 0 & \mid \bm{A}\bm{z} \le \bm{0}, \\
%    H^{\hat{J}}_{1,i}: \bm{\eta}^\top \bm{\mu} > 0  & \mid \bm{A}\bm{z} \le \bm{0}, 
%\end{align*}
%
where we note that $\bm{\eta}$ depends on $i$, $\bm{A} \in \{-1, 0, 1\}^{(l-1)
\times l}$, $\bm{A}_{s, :} = [0, \ldots, \underbrace{1}_{\hat{J}},
\cdots, \underbrace{-1}_{s} ,\cdots,0]$ for all $s \in \{1,\ldots, l\} \backslash \{ \hat{J} \}$ and
$\bm{A}_{s,:}$ denote the $s^{th}$ row of $\bm{A}$. This equivalence was
exploited in the multiple goodness-of-fit testing by Yamada et al.\ \cite{yamada2018post}.
The condition $\bm{A}\bm{z} \le \bm{0}$ represents the fact that $P_{\hat{J}}$
is selected as the reference model, and expresses $\hat{D}(P_{\hat{J}},
R) \le \hat{D}(P_{s}, R)$ for all $s \in \{ 1, \ldots, l \} \backslash \{ \hat{J} \} $.

%where we note that $\bm{\mu}$ depends on $i$, $\bm{A} \in \{-1, 0, 1\}^{(l-1)
%\times l}$, $\bm{A}_{s, \hat{J}} = [0, \ldots, \overbrace{1}^{\hat{J}}, \cdots,
%\overbrace{-1}^{s} ,\cdots,0]$ for all $s \in {1,\ldots, l-1}$.
%------------------------------------
\subsection{\relmul: for controlling false discovery rate (FDR)}
\label{sec:relmul}
%
%\textbf{Conditional null hypothesis}
Unlike traditional hypothesis testing, the null hypotheses here are conditional on the
selection event, making the null distribution non-trivial to derive
\cite{leeb2005model, leeb2006can}. Specifically, the sample used to form the
selection event (i.e., establishing the reference model) is the same sample
used for testing the hypothesis, creating a dependency.
%we seek to characterize the distribution of the statistic
%under the \textit{conditional} null distribution which is no trivial task .
%
Our first approach of \relmul{} is to divide the sample into two independent sets,
%\cite{cox1975note},
where the first is used to choose $P_{\hat{J}}$ and the latter for performing the test(s).
This approach 
%not only produces 
%valid inferences after selection \cite[Section 2.5]{fithian2014optimal}, it 
simplifies the null distribution since the sample used to form the selection
event and the test sample are now independent. That is,  
    $H^{\hat{J}}_{0,i}: \bm{\eta}^\top \bm{\mu} \le 0  \mid \bm{A}\bm{z} \le \bm{0}$ simplifies to 
    $H^{\hat{J}}_{0,i}: \bm{\eta}^\top \bm{\mu} \le 0 $ due to independence.
In this case, the distribution of
the test statistic (for $\widehat{\mmd}^2_u$ and $\widehat{\ksd}^2_u$) after
selection is the same as its unconditional null distribution. Under our
assumption that all distributions are distinct, the test statistic is
asymptotically normally distributed
\cite{gretton2012kernel,liu2016kernelized,chwialkowski2016kernel}. 

For the complete U-statistic estimator of Maximum Mean
Discrepancy ($\widehat{\mathrm{MMD}}^2_u$), Bounliphone et al.
\cite{bounliphone2015test} showed that, under mild assumptions, $\bm{z}$
is jointly asymptotically normal, where the covariance matrix is known in
closed form. However, for $\bighat{\ksd}^2_u$, only the marginal variance is
known \cite{chwialkowski2016kernel,liu2016kernelized} and not its cross
covariances, which are required for characterizing the null distributions of
our test (see Algorithm \ref{algorithm:relmul} in the appendix for the full algorithm of \relmul). 
We present the asymptotic multivariate characterization of $\widehat{\ksd}^2_u$
in Theorem \ref{theorem:ksd}.
Given a desired significance level $\alpha \in (0,1)$, the rejection threshold is chosen to be the $(1-\alpha)$-quantile
of the distribution $\mathcal{N}(0,\hat{\sigma}^2)$ where $\hat{\sigma}^2$ is the plug-in
estimator of the asymptotic variance $\sigma^2$ of our test statistic (see \cite[Section
3]{bounliphone2015test} for 
MMD and Section \ref{sec:RelKSD} for KSD). 
%And our inferences after selection is valid i.e. the $p$-values follow a uniform distribution under the null 
%hypothesis even though our null is dependent on the data.
With this choice, the false rejection rate for each of the $l-1$ hypotheses is
upper bounded by $\alpha$ (asymptotically).
However, to control the false discovery rate for the $l-1$ tests it is necessary to
further correct with multiple testing adjustments. We use the Benjamini–Yekutieli procedure
\cite{benjamini2001control} to adjust $\alpha$. We note that when testing
$H_{0,\hat{J}}^{\hat{J}}$, the result is always 0 (fail to reject) by default.
When $l>2$, following the result of \cite{benjamini2001control} the asymptotic false
discovery rate (FDR) of \relmul{} is provably no larger than $\alpha$. The FDR in our case is the
fraction of the models declared worse that are in fact as good as the (true) reference model.
 For $l=2$, no correction is required as only one test is performed.
%-----------------------------------
 \subsection{\relpsi: for controlling false positive rate (FPR)}
\label{sec:relpsi}

A caveat of the data splitting used in \relmul{} is the loss of true positive rate since a
portion of sample for testing is used for forming the selection. 
%A more 
%subtle issue is that since the sample used for forming the selection event is
%smaller than the full sample, there is a higher chance of selecting a wrong
%reference model.  
When the selection is wrong, i.e., $\hat{J} \in \mathcal{I}_+$,
the test will yield a lower true positive rate. 
%To see this consider the
%worst case when $P_{\hat{J}}$ is chosen to be the worst model, then the null
%case is true for all $l-1$ tests.  
It is possible to alleviate this issue by using the full sample for selection
and testing, which is the approach taken by our second proposed test \relpsi.
This approach requires us to know the null distribution of the conditional null
hypotheses (see Section \ref{sec:selection}), which can be derived based on Theorem \ref{theorem:poly}.

\begin{theorem}[Polyhedral Lemma \cite{lee2016exact}]
    \label{theorem:poly}
    Suppose that $\bm{z} \sim \mathcal{N}(\bm\mu,\bm\Sigma)$ and the selection event is affine, i.e., $\bm{Az} \le \bm{b}$ for some matrix $\bm{A}$ and
    $\bm{b}$, then for any $\bm\eta$, we have
\begin{equation*}
    \bm\eta^\top\bm{z}\ |\ \bm{Az} \le \bm{b}\ \sim\ \mathcal{TN}(\bm\eta^\top\bm\mu,\ \bm\eta^\top\bm\Sigma\bm\eta,\  \mathcal{V}^-(\bm{z}),\  \mathcal{V}^+(\bm{z})),
\end{equation*}
where $\mathcal{TN}(\mu, \sigma^2, a, b)$ is a truncated normal distribution with
mean $\mu$ and variance $\sigma^2$ truncated at $[a,b]$. Let $\bm\alpha =\frac{\bm{A\Sigma\eta }}{\bm{\eta^\top\Sigma\eta}}$. The truncated points are given by:
$\mathcal{V}^-(\bm{z}) = \max_{j:\bm\alpha_j<0} \frac{\bm{b}_j -
\bm{Az}_j}{\bm\alpha_j} + \bm\eta^\top\bm{z}$, and $\mathcal{V}^+(\bm{z})
= \min_{j:\bm\alpha_j>0} \frac{\bm{b}_j - \bm{Az}_j}{\bm\alpha_j} +
\bm\eta^\top\bm{z}$.
%\begin{align*}
%    \mathcal{V}^-(\bm{z}) = \max_{j:\bm\alpha_j<0} \frac{\bm{b}_j - \bm{Az}_j}{\bm\alpha_j} + \bm\eta^\top\bm{z},\qquad
%    \mathcal{V}^+(\bm{z}) = \min_{j:\bm\alpha_j>0} \frac{\bm{b}_j - \bm{Az}_j}{\bm\alpha_j} + \bm\eta^\top\bm{z}.
%\end{align*}
\end{theorem}
This lemma assumes two parameters are known: $\bm\mu$ and $\bm\Sigma$.
Fortunately, we do not need to estimate $\bm\mu$ and can set
$\bm\eta^\top\bm\mu = 0$. To see this note that threshold is given by
($1-\alpha$)-quantile of a truncated normal which is
$t_\alpha := \bm{\eta}^\top\bm{\mu}+\sigma\Phi^{-1}\big((1-\alpha)\Phi\big(\frac{{\mathcal{V^+}-\bm{\eta}^\top\bm{\mu}}}{\sigma}\big)+\alpha\Phi\big(\frac{{\mathcal{V^-}-\bm{\eta}^\top\bm{\mu}}}{\sigma}\big)\big)$
where $\sigma^2=\bm\eta^\top\bm\Sigma\bm\eta$.
%\jlsay{added this part to define selective type-I error}
If our test statistic $\bm\eta^\top\bm{z}$ exceeds the threshold, we reject the null hypothesis $H^{\hat{J}}_{0,i}$. This choice of the rejection threshold  will control the \textit{selective type-I error} $\Prob(\bm\eta^\top\bm{z} > t_\alpha\ \mid H^{\hat{J}}_{0,i} \text{ is true}, P_{\hat{J}} \text{ is selected})$ to be no larger than $\alpha$.
%\wjsay{check again. Add $\sqrt{\cdot}$?}.
However $\bm\mu$ is not known, the
threshold can be adjusted by setting $\bm{\eta}^\top\bm{\mu}=0$ and can be seen
as a more conservative threshold.
A similar adjustment procedure is used in Bounliphone et al.\ \cite{bounliphone2015test} and
Jitkrittum et al.\ \cite{jitkrittum2018informative} for Gaussian distributed
test statistics.
And since $\bm\Sigma$ is also unknown, we replace $\bm\Sigma$ with a consistent
plug-in estimator $\hat{\bm\Sigma}$ given by Bounliphone et al.\ \cite[Theorem 2]{bounliphone2015test} for
$\widehat{\mmd}^2_u$ and Theorem \ref{theorem:ksd} for $\widehat{\ksd}^2_u$.
Specifically, we have as the threshold $\hat{t}_\alpha :=
\hat{\sigma}\Phi^{-1}\big((1-\alpha)\Phi\big(\frac{\mathcal{V^+}}{\hat{\sigma}}\big)+\alpha\Phi\big(\frac{\mathcal{V^-}}{\hat{\sigma}}\big)\big)$
where $\hat{\sigma}^2 = \bm\eta^\top\hat{\bm\Sigma}\bm\eta$
(see Algorithm \ref{algorithm:relpsi} in the appendix for the full algorithm of RelPSI).

% In Theorem \ref{theorem:consistency_mmd}, we prove the test consistency of \relpsi{} for two
% candidate models when $\bm\eta$ is fixed (and does not depend on our selection procedure) and
% \textrm{MMD} is the discrepancy function (which we call \relpsimmd).
Our choice of $\bm\eta$ depends on the realization of $\hat{J}$, but
$\bm\eta$ can be fixed such that the test we perform is independent of our
observation of $\hat{J}$ (see Experiment 1). For a fixed $\bm\eta$, the concept
of power, i.e., $\Prob(\bm{\eta}^\top\bm{z}> \hat{t}_\alpha)$ when
$\bm{\eta}^\top\bm\mu>0$, is meaningful; and we show in Theorem \ref{theorem:consistency_mmd} that our test is
consistent using MMD. However, when
$\bm\eta$ is random (i.e., dependent on $\hat{J}$) the notion of test power is less appropriate, and we use true
positive rate and false positive rate to measure the performance (see Section \ref{sec:performance}).

% In Theorem \ref{theorem:consistency_mmd}, we prove the test consistency of \relpsi{} for two
% candidate models when $\bm\eta$ is fixed (and does not depend on our selection procedure) and \textrm{MMD} is the discrepancy function (which we call \relpsimmd).

\begin{comment}
We can construct a valid conditional p-value for testing the null $H_{0,i}:
D(P_{i}, R) \ge D(P_j,R)$ by taking the probability integral transform of the
statistic which is pivotal i.e it is a uniform distribution  \cite[Lemma
1,2]{tibshirani2016exact} (For an empirical validation, see Appendix
\ref{sec:callibration}). At a given significance level $\alpha$, we can either
reject the null hypothesis if the p-value is less than $\alpha$ or if our test
statistic $\bighat{\bm{\eta}^\top\bm{z}} > \hat{T}_\alpha$ where
$\hat{T}_\alpha$ is the $(1-\alpha)$-quantile of the truncated normal. Both are
equivalent. 
\end{comment}

\begin{restatable}[Consistency of \textrm{RelPSI-MMD}]{theorem}{mmdcons}
        Given two models $P_1$, $P_2$ and a data distribution $R$ (which are
        all distinct). Let $\hat{\bm\Sigma}$ be a consistent estimate of the covariance matrix
        defined in Theorem \ref{theorem:mmd}.
and $\bm\eta$ be defined such that $\bm{\eta}^\top\bm{z} = \sqrt{n}[\bighat{\mmd}^2_u(P_2,R) - \bighat{\mmd}^2_u(P_1,R)]$.
    Suppose that the threshold $\hat{t}_\alpha$ is the $(1-\alpha)$-quantile of 
$\mathcal{TN}(\bm{0}, \bm{\eta}^\top\hat{\bm\Sigma}\bm\eta, \mathcal{V}^-, \mathcal{V}^+)$ where
    $\mathcal{V}^+$ and $\mathcal{V}^-$ are defined in Theorem \ref{theorem:poly}. Under
    $H_{0}: \bm{\eta}^\top\bm{\mu} \le 0 \,|\,P_{\hat{J}} \text{ is selected}$, the asymptotic type-I error is bounded above by $\alpha$.
    Under $H_{1}:\bm{\eta}^\top\bm{\mu} > 0\,|\,P_{\hat{J}} \text{ is selected}$, we have
    $\Prob(\bm{\eta}^\top\bm{z}> \hat{t}_\alpha) \rightarrow 1$ as $n \rightarrow \infty$.
    \label{theorem:consistency_mmd}
  \end{restatable}
A proof for Theorem \ref{theorem:consistency_mmd} can be found in Section 
\ref{proof:consistency} in the appendix.
A similar result holds for \textrm{RelPSI-KSD} (see Appendix
\ref{theorem:consksd}) whose proof follows closely the proof of Theorem
\ref{theorem:consistency_mmd} and is omitted.

\begin{comment}
\wjsay{Move to the main section}
\jlsay{Should we drop the section?}
\textbf{Post Selection Inference (PSI)} aims to provide valid inference after
selection. In recent years, exact PSI has be achieved from the conditional
approach \cite{berk2013valid} using the polyhedral lemma \cite{lee2016exact,
tibshirani2016exact}. This is a general result which states that if $\bm{z}
\sim \mathcal{N}(\bm\mu, \bm\Sigma)$ then conditioning on affine conditions,
i.e., $\bm{Az} \le \bm{b}$, the resultant conditional distribution, $\bm{z}\,
|\,\bm{Az} \le \bm{b}$, is a truncated normal with parameters determined by the
Polyhedral Lemma. The combination of a kernel-based test and PSI was first
studied by Yamada et al. \cite{yamada2016post,yamada2018post} who proposed
methods for feature and dataset selection with $\mmd$ and $\mathrm{HSIC}$. To
satisfy the normality assumption of the Polyhedral Lemma, block estimators
\cite{zaremba2013b, yamada2016post} and incomplete estimator
\cite{yamada2018post} were proposed.
\end{comment}
%

%% file: performance.tex
\label{sec:performance}
Post selection inference (PSI) incurs its loss of power from conditioning on
the selection event \cite[Section 2.5]{fithian2014optimal}.
%which disregards information used in selection
%to partake in the testing phase 
Therefore, in the fixed hypothesis (not conditional) setting of $l=2$ models, it is unsurprising
that the empirical power of \relmmd{} and \textrm{RelKSD} is higher than its PSI
counterparts (see Experiment 1). However, when $l=2$, and conditional hypotheses are
considered, it is unclear which approach is desirable.
Since both PSI (as in \relpsi) and
data-splitting (as in \relmul) approaches for model comparison have tractable null
distributions, we study the performance of our proposals for the case when the hypothesis 
is dependent on the data.

% The choice of selection event is an important problem in the PSI framework. It
% determines the behaviour of the test. If we condition on too much information,
% there will be insufficient information to carry out a meaningful test but too
% little and the bias will not be corrected \cite{fithian2014optimal}. In some
% circumstances, it might be desirable to condition on finer selection events as
% it can be computationally convenient \cite{lee2016exact}, and furthermore, it
% can assert stronger inferential guarantees \cite[SectionD
% 8]{fithian2014optimal}. However, these benefits often results in a loss of test
% power.

We measure the performance of \relpsi{} and \relmul{} by
\textit{true positive rate} (\textrm{TPR}) and \textit{false
positive rate} (\textrm{FPR}) in the setting of $l=2$ candidate models. These
are popular metrics when reporting the performance of
selective inference approaches \cite{suzumura2017selective, yamada2018post,
fithian2014optimal}. \textrm{TPR} is the expected proportion of models worse
than the best that are correctly reported as such. \textrm{FPR} is the expected proportion
of models as good as the best that are wrongly reported as worse.
It is desirable for \textrm{TPR} to be
high and \textrm{FPR} to be low. We defer the formal definitions to 
Section \ref{sec:prdefinition} (appendix); when we estimate TPR and FPR, we denote it as $\widehat{\mathrm{TPR}}$ and $\widehat{\mathrm{FPR}}$ respectively. In the following
theorem, we show that the \textrm{TPR} of \relpsi{} is higher than
the TPR of \relmul{}.
\begin{restatable}[\textrm{TPR} of \relpsi{} and \relmul{}]{theorem}{tpr}
    Let $P_1, P_2$ be two candidate models, and $R$ be a data generating distribution.
    Assume that $P_1, P_2$ and $R$ are distinct.
    Given $\alpha \in [0, \frac{1}{2}]$ 
    and split proportion $\rho\in(0,1)$ for \relmul{} so that $(1-\rho)n$ samples are used for
    selecting $P_{\hat{J}}$ and $\rho n$ samples for testing,  
    for all $n \gg N =
    \big(\frac{\sigma\Phi^{-1}(1-\frac{\alpha}{2})}{\mu(1-\sqrt{\rho})}\big)^2$,
    we have $\mathrm{TPR}_{\relpsi} \gtrapprox
    \mathrm{TPR}_{\relmul}$. 
    %\wjsay{Empirical quantities?} \jlsay{TPR should be a population quantity} 
    %\wjsay{All these are ``approximately''? I will think of a way to make it more precise with $O(1/n)$ if possible.}
\label{theorem:tpr}
\end{restatable}
The proof is provided in the Section \ref{proof:tpr}. This result holds for
both \textrm{MMD} and \textrm{KSD}. Additionally, in the following result we
show that both approaches bound \textrm{FPR} by $\alpha$. Thus,
\textrm{RelPSI} controls FPR
regardless of the choice of discrepancy measure and number of candidate models. 
%This means that if we wish to control \textrm{FPR} then \textrm{RelPSI} does
%not require any explicit correction,

\begin{restatable}[FPR Control]{lemma}{fpr}
    % \wjsay{Write where the proof is. Explain the connection to polyhedral lemma, and the truncated normal null distribution that we use.}
\label{lemma:fpr}
Define the selective type-I error for the $i^{th}$ model to be $s(i, \hat{J}) := \mathbb{P}(\text{reject } H^{\hat{J}}_{0,i} \mid H^{\hat{J}}_{0,i} \text{ is true}, P_{\hat{J}} \text{ is selected})$.
If $s(i, \hat{J}) \le \alpha$ for any $i,\hat{J} \in \{1,\ldots, l\}$, then $\mathrm{FPR} \le \alpha$.

%Given a set of models $\mathcal{M}$, a test that controls the selective type-I error (i.e., if $H^{\hat{J}}_{0,i}$ is true, then we have $\label{define:selective_type_1}
% \mathbb{P}(\sqrt{n}[\hat{D}(P_i,R)-\hat{D}(P_{\hat{J}},R)] > \hat{t}_\alpha\,|\, P_{\hat{J}} \text{ is selected}) \le \alpha$)
%then the false positive rate is bounded than alpha, i.e.,
%$\mathrm{FPR} \le \alpha.$ 
\end{restatable}

The proof can be found in Section \ref{proof:fpr}. For both \relpsi{} and \relmul{}, the test threshold is chosen to control the selective type-I error.  
%$\mathbb{P}(\sqrt{n}[\hat{D}(P_i,R)-\hat{D}(P_{\hat{J}},R)] > \hat{t}_\alpha \,|\, P_{\hat{J}} \text{ is selected}) \le \alpha$ 
Therefore, both control \textrm{FPR} to be no larger than $\alpha$. In \textrm{RelPSI}, we explicitly control this quantity by characterizing the distribution of statistic under the conditional null.
%For \relmul{}, controlling type-I error such that it is bounded by $\alpha$ is equivalent to bounding FPR by $\alpha$.
%
\begin{remark}
The selection of the best model is a noisy process, and we can pick a model that is worse
than the actual best, i.e., $\hat{J} \notin \arg \min_j D(P_j, R)$. An incorrect selection results
in a higher portion of true conditional null hypotheses. So,
the true positive rate of the test will be lowered. However, the false rejection is still controlled
at level $\alpha$.
%and so, the theoretical guarantees of controlling the number of false positives
%still hold regardless of the choice of $P_{\hat{J}}$.
\end{remark}

%% file: experiments.tex
\begin{figure}[]
    \centering
        \includegraphics[width=.75\textwidth]{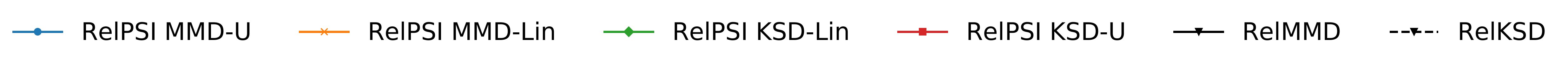}
        \includegraphics[width=.24\textwidth]{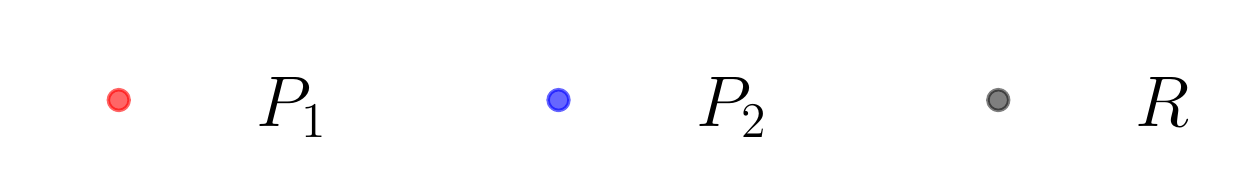}
  \subfloat[Mean shift $d=10$\label{fig:ex1_ms}]{
        \includegraphics[width=.25\textwidth]{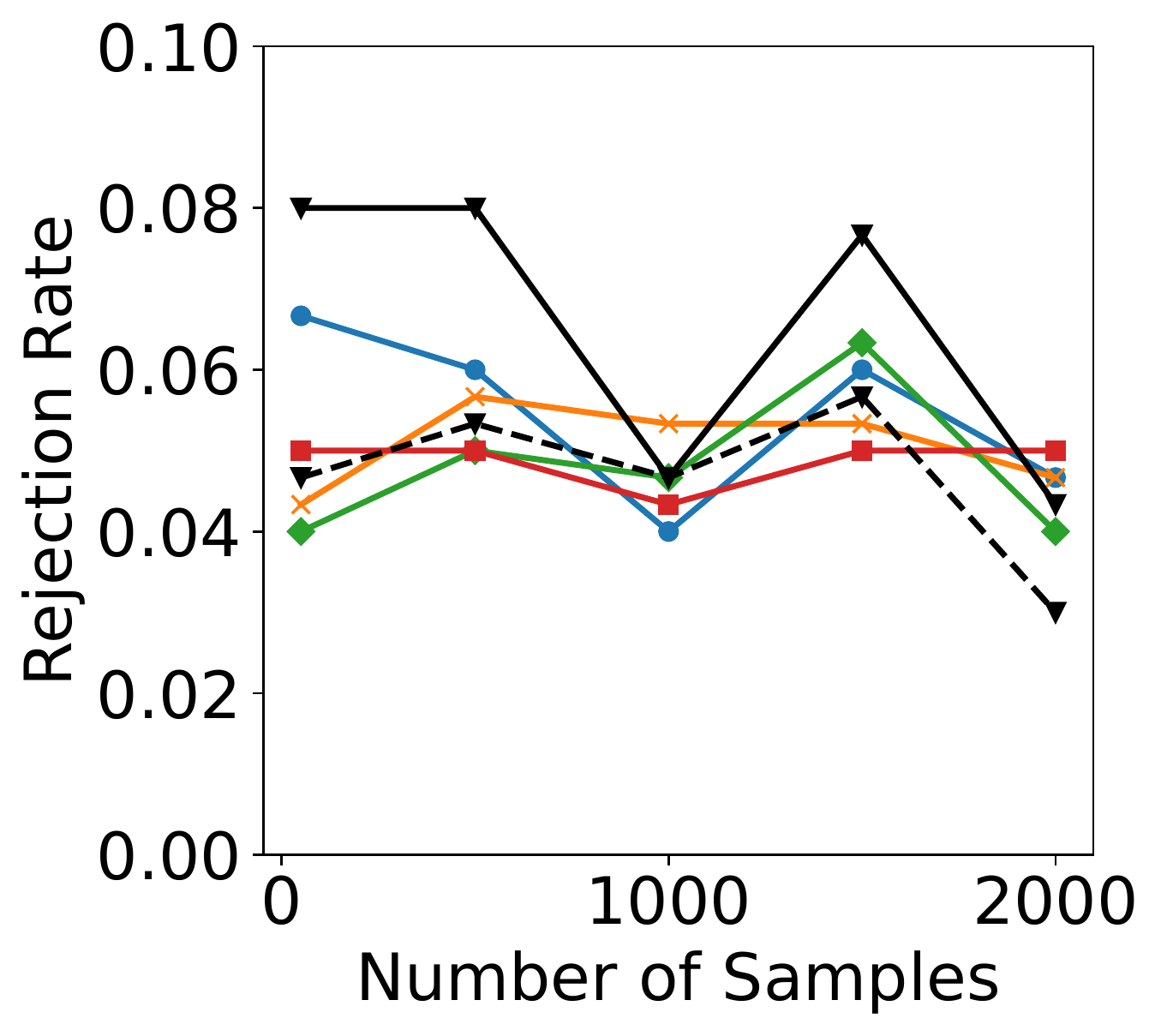}}
  \subfloat[Blobs $d=2$\label{fig:ex1_blobs}]{
      \centering
       \includegraphics[width=.25\textwidth]{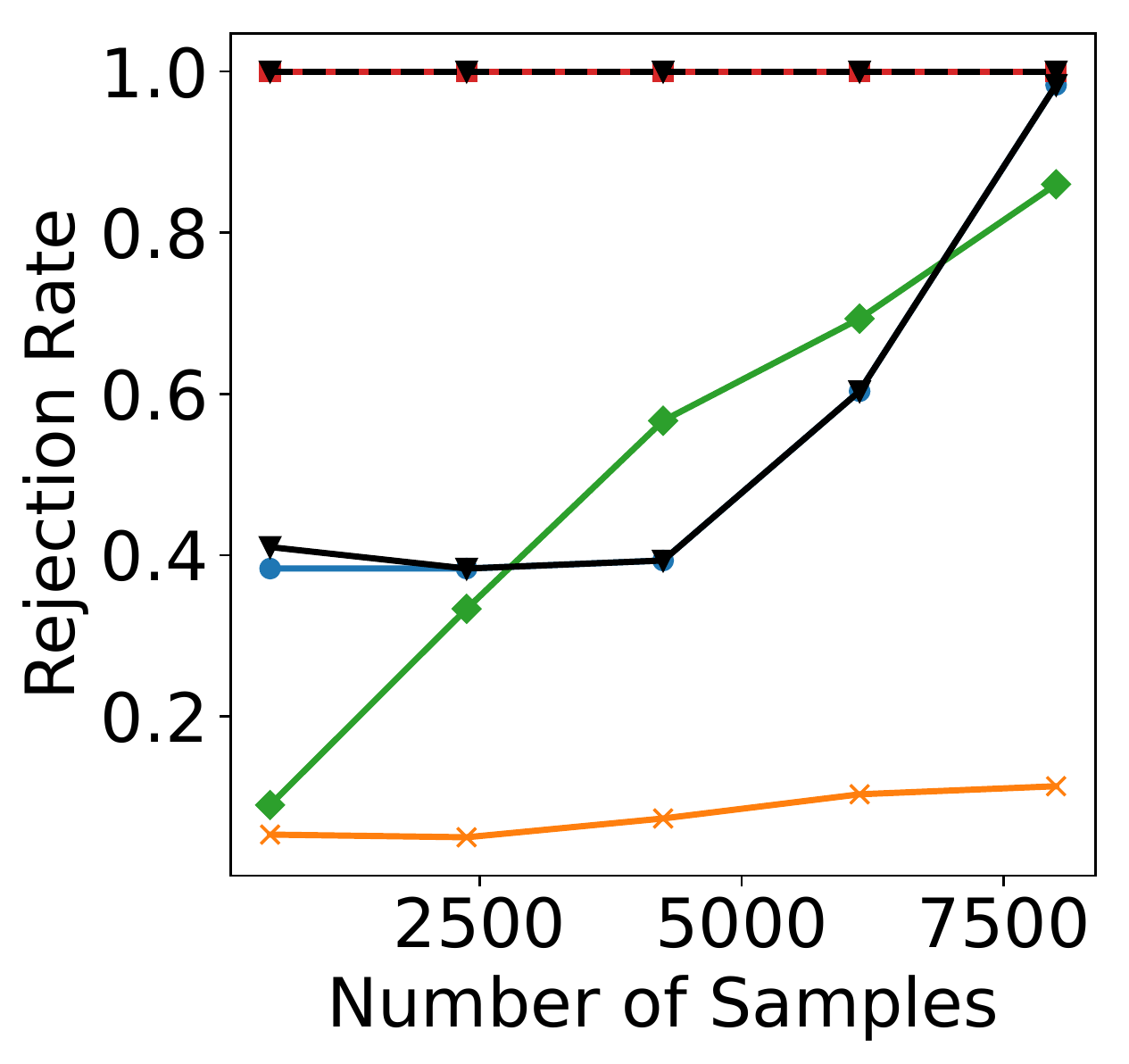}}
  \subfloat[RBM $d=20$\label{fig:ex1_rbm}]{
      \centering
       \includegraphics[width=.25\textwidth]{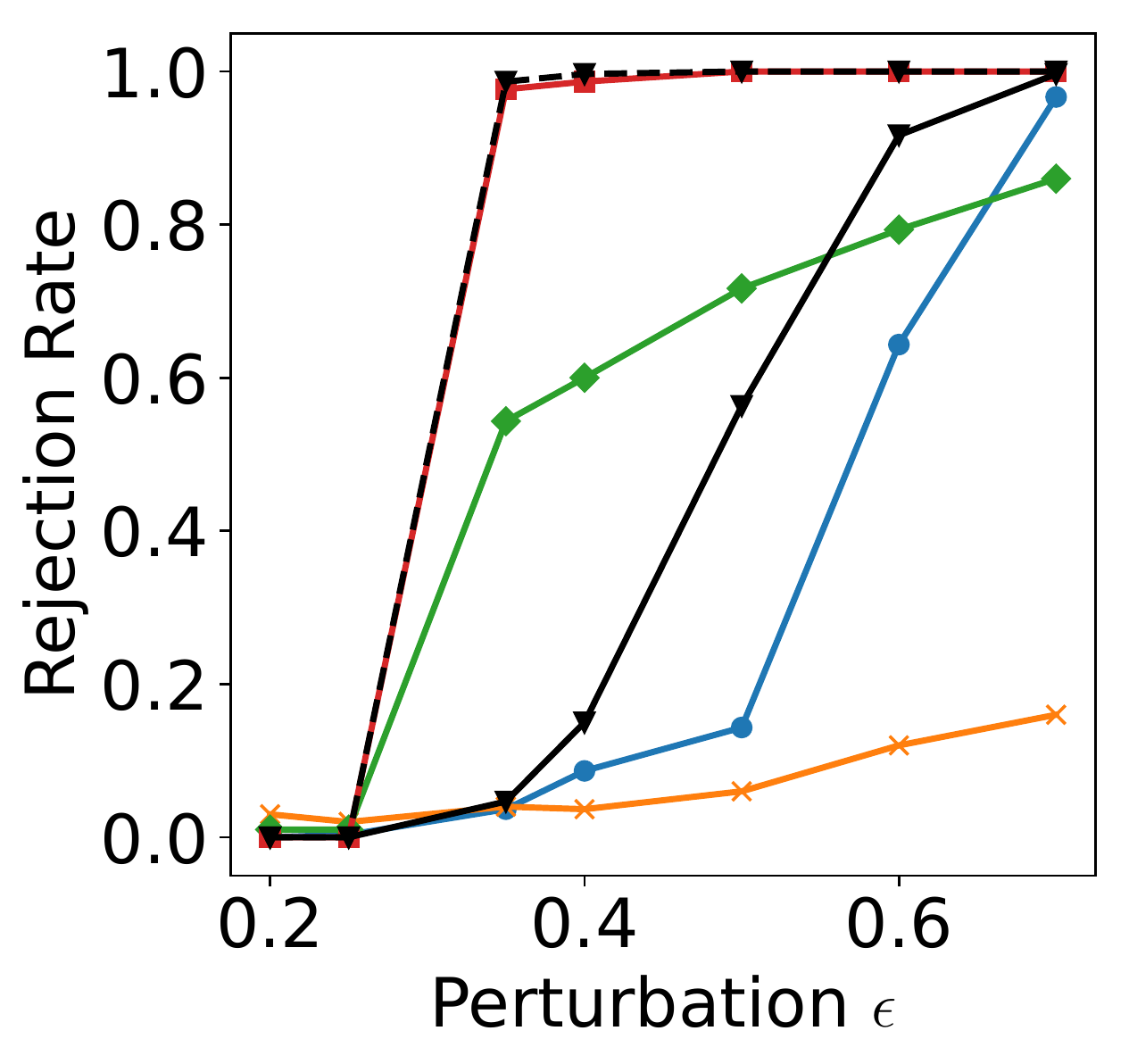}}
   \subfloat[Blobs problem.\label{fig:ex1_blobs_vis}]{
       \centering
       \includegraphics[width=.25\textwidth]{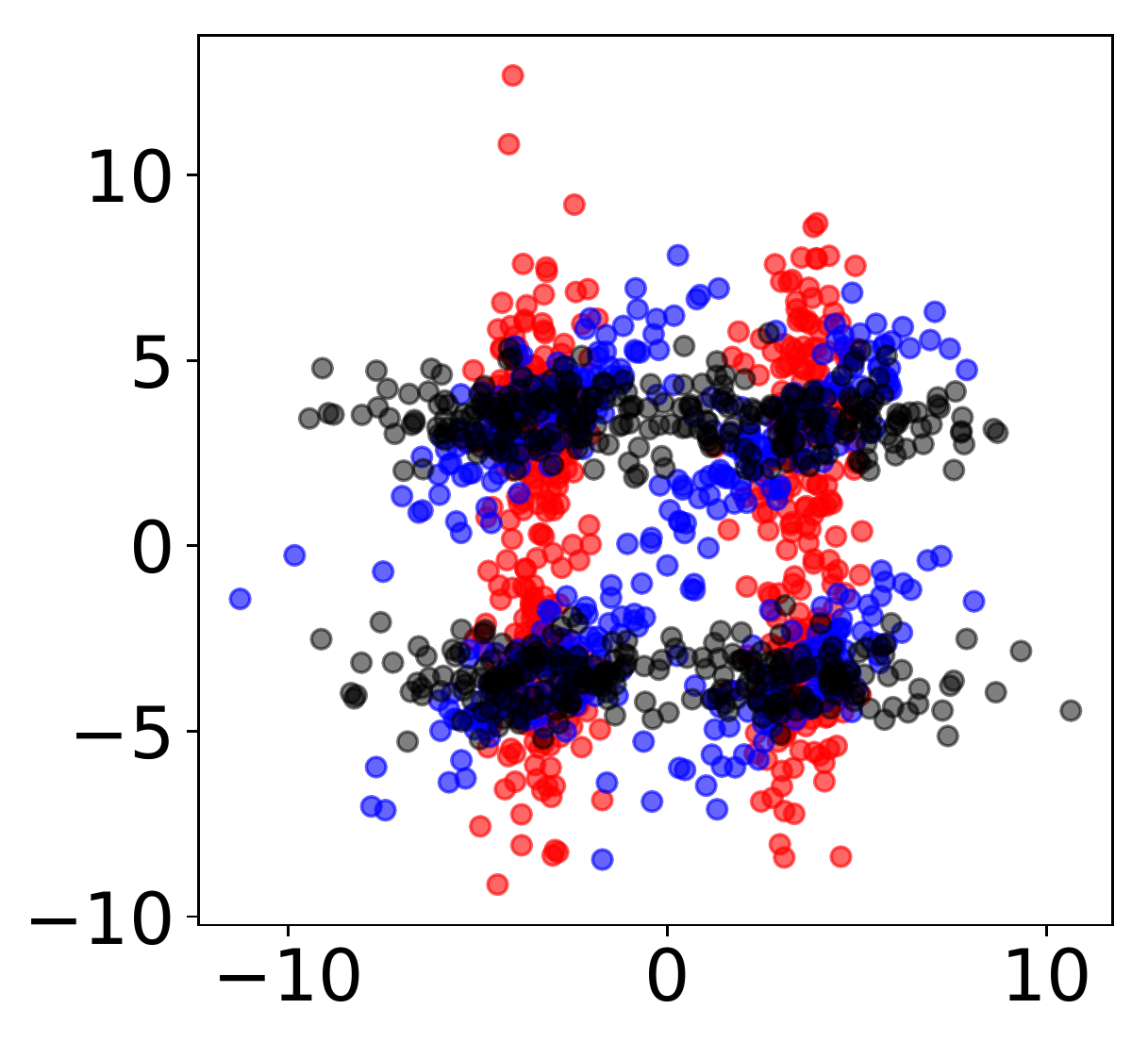}}
    \caption{Rejection rates (estimated from 300 trials) for the six tests with $\alpha = 0.05$ is shown. ``MMD-U'' refers
        to the usage of the complete U-statistic for \textrm{MMD} which is $\widehat{\mathrm{MMD}}^2_u$, ``MMD-Lin'' refers to the linear
        time estimator $\widehat{\mathrm{MMD}}^2_l$ and similarly for KSD Complete and KSD Linear (defined in Section \ref{sec:background}).}
    \label{fig:ex1}
    \vspace{-3mm}
\end{figure}
In this section, we demonstrate our proposed method for both toy problems and
real world datasets. Our first experiment is a baseline comparison of our proposed method \relpsi\ to \relmmd\ \cite{bounliphone2015test} and \relksd\ 
(see Appendix \ref{sec:RelKSD}). In this experiment, we consider a fixed hypothesis of model comparison for two candidate models (\relmul\ is not applicable here). This is the original setting that \relmmd\ was proposed for. In the second experiment, we consider a
set of mixture models for smiling and non-smiling images of CelebA \cite{liu2018large} 
where each model has its own unique generating proportions from the real data set or images
from trained GANs. For our final experiment, we examine density estimation models trained 
on the Chicago crime dataset considered by Jitkrittum et al.\
\cite{jitkrittum2017linear}. In this experiment, each model has a score
function which allows us to apply both \relpsi\ and \relmul\ with \textrm{KSD}.
In the last two experiments on real data, there is no ground truth for which candidate model is the best; so
estimating \textrm{TPR}, \text{FDR} and \textrm{FPR} is infeasible. Instead, the experiments
are designed to have a strong indication of the ground truth with support from another metric. More synthetic experiments are shown in Appendix \ref{sec:more_experiments} to verify our theoretical results. In order to account for sample variability, our experiments are averaged over at least $100$ trials with new samples (from a different seed) redrawn for each trial. Code for reproducing the results is online.\footnote{\url{https://github.com/jenninglim/model-comparison-test}}

\textbf{1. A comparison of} \relmmd \textbf{,
}\textrm{\relksd}\textbf{, }\textrm{\relpsiksd} \textbf{and}
\textrm{\relpsimmd} ($l=2$): The aim of this experiment is to investigate the behaviour
of the proposed tests with \textrm{RelMMD} and \textrm{RelKSD} as baseline comparisons and
empirically demonstrate that \textrm{RelPSI-MMD} and \textrm{RelPSI-KSD} possess 
desirable properties such as level-$\alpha$ and comparable test power. Since
\textrm{RelMMD} and \textrm{RelKSD} have no concept of selection, in order for
the results
to be comparable we fixed null hypothesis to be $H_0:D(P_1,R) \le D(P_2,R)$ which is possible for \textrm{RelPSI} by fixing $\bm\eta^\top =[-1,1]$. In this experiment, we consider
the following problems:
\begin{enumerate}
    \item \textit{Mean shift}: The two candidate models are isotropic Gaussians on
    $\mathbb{R}^{10}$ with varying
    mean: $P_1 = \mathcal{N}([0.5,0,\cdots,0],\bm{I})$ and $P_2 = \mathcal{N}([-0.5,0,\cdots,0],\bm{I})$.
        Our reference distribution is $R = \mathcal{N}(\bm{0},\bm{I})$. In this case, $H_0$ is true.
    \item \textit{Blobs}: This problem was studied by various authors 
        \cite{chwialkowski2015fast,gretton2012optimal,jitkrittum2018informative}.
        Each distribution is a mixture of 
        Gaussians with a similar structure on a global scale but different locally by rotation. Samples from this distribution is shown in Figure \ref{fig:ex1_blobs_vis}. In this case, the $H_1$ is true.
    \item \textit{Restricted Boltzmann Machine (RBM)}: This problem was studied by 
        \cite{liu2016kernelized, jitkrittum2017linear, jitkrittum2018informative}.
        Each distribution is given by a Gaussian Restricted Boltzmann Machine
        (RBM) with a density $p(\bm{y})=\sum_{\bm{x}} p'(\bm{y}, \bm{x})$ and $p'(\bm{y},\bm{x}) =\frac{1}{Z}\exp (\bm{y^\top B x}+\bm{b^\top y}+\bm{c^\top x}-\frac{1}{2}||\bm{y}||^2)$ where $\bm{x}$ are the latent variables
        and model parameters are $\bm{B, b, c}$. The model will share the same parameters
        $\bm{b}$ and $\bm{c}$ (which are drawn from a standard normal) with the reference distribution but the matrix 
        $\bm{B}$ (sampled uniformly from $\{-1, 1\}$) will be perturbed with $\bm{B}^{p_2}= \bm{B} + 0.3 \delta$ and 
        $\bm{B}^{p_1}= \bm{B} + \epsilon \delta$ where $\epsilon$ varies between $0$ and $1$. It measures the sensitivity of the test \cite{jitkrittum2017linear} since perturbing only one entry can create a
        difference that is hard to detect.
        Furthermore, We fix $n=1000$, $d_x = 5$,
        $d_y = 20$.
\end{enumerate}
Our proposal and baselines are all non-parametric kernel based test. For a fair comparison, all the tests use the same Gaussian kernel with its bandwidth chosen by the
median heuristic.
In Figure \ref{fig:ex1}, it shows the rejection rates for all tests. As expected, the tests based on \textrm{KSD} have higher power than \textrm{MMD} due to having access to the density function. Additionally, linear time estimators perform worse than their complete counterpart.

In Figure \ref{fig:ex1_ms}, when $H_0$ is true, then the false rejection rate (type-I error) is controlled  around level $\alpha$ for all tests. In Figure \ref{fig:ex1_blobs}, the poor performance of \textrm{MMD}-based tests in blobs experiments is caused by 
an unsuitable choice of bandwidth. The median heuristic cannot capture 
the small-scale differences \cite{gretton2012optimal, jitkrittum2018informative}. Even though \textrm{KSD}-based tests utilize the same heuristic, equipped with the density function a
mismatch in the distribution shape can be detected.
Interestingly, in all our experiments, the \textrm{RelPSI} variants perform comparatively to their cousins, \textrm{Rel-MMD} and \textrm{Rel-KSD} but as expected, the 
power is lowered due to the loss of information from our conditioning \cite{fithian2014optimal}. These two problems show the behaviour of the tests when the number of samples $n$ increases.

In Figure \ref{fig:ex1_rbm}, this shows the behaviour of the tests when the difference between the candidate models increases (one model gets closer to the reference distribution). When $\epsilon < 0.3$, the null case is true and the tests exhibit a low rejection rate. However, when $\epsilon > 0.3$ then the alternative is true. Tests utilizing \textrm{KSD} can detect this change quickly which indicated by the sharp increase in the rejection rate when $\epsilon=0.3$. However, MMD-based tests are unable to detect the differences at that point. As the amount of perturbation increases,  this changes and \textrm{MMD} tests begin to identify with significance that the alternative is true. Here we see that \textrm{RelPSI-MMD} has visibly lowered rejection rate indicating the cost of power for conditioning, whilst for \textrm{RelPSI-KSD} and \textrm{RelKSD} both have similar power.
\begin{table}[]
\centering
\caption{A comparison of our proposed method with FID. The
underlying distribution are samples forming a mixture of smiling (S) or
non-smiling (N) faces which can be either generated (G) or real (R). ``Rej.'' denotes the rate of rejection of the model indicating that it is significantly worse than the best model.
``Sel.'' is the rate at which the model is selected (the one with the minimum discrepancy score).
Average \textrm{FID} scores are also reported. These results are averaged over 100 trials. 
%The ground truth model is 0.5 (R) for S and 0.5 (R) for N.
} 
\label{tab:ex3_table}
\begin{tabular}{@{}ccccccccc@{}}
\toprule
      & \multicolumn{2}{c}{Mix} & \multicolumn{2}{c}{\textrm{RelPSI-MMD}} & \multicolumn{2}{c}{\relmul{-MMD}} & \multicolumn{2}{c}{FID} \\ \midrule
Model & S          & N          & Rej.         & Sel.         & Rej.         & Sel.         & Aver.       & Sel.      \\ \midrule
1     & 0.50 (G)   & 0.50 (G)   & 0.99            & 0.0            & 1.0          & 0.0          & $27.86 \pm 0.49$       & 0         \\
2     & 0.60 (R)   & 0.40 (R)   & 0.39         & 0.02            & 0.18         & 0.08          & $16.01 \pm 0.19$         & 0.39         \\
3     & 0.40 (R)   & 0.60 (R)   & 0.28         & 0.03            & 0.19          & 0.10          & $16.29 \pm 0.20$         & 0.03         \\
4     & 0.51 (R)   & 0.49 (R)   & 0.02          & 0.52         & 0.03          & 0.37         & $16.03 \pm 0.18$         & 0.27       \\
5     & 0.52 (R)   & 0.48 (R)   & 0.06          & 0.43         & 0.0          & 0.45          & $16.01 \pm 0.17$         & 0.31       \\ \midrule
Truth & 0.5 (R)    & 0.5 (R)    & -            & -            & -            & -            & -           & -         \\ \bottomrule
\end{tabular}
\end{table}

\textbf{2. Image Comparison} $l=5$: In this experiment, we apply our proposed test \textrm{RelPSI-MMD} and \relmulmmd{} for comparing 
between five image generating candidate models. We consider the CelebA dataset 
\cite{liu2018large}
which for each sample is an image of a celebrity labelled with 40 annotated features.
As our reference distribution and candidate models, we use a mixture of
smiling and non-smiling faces of varying proportions (Shown in Table \ref{tab:ex3_table}) where
the model can generate images from a GAN or from the real dataset. For generated images, we use the GANs of \cite[Appendix B]{jitkrittum2018informative}.
In each trial, $n=2000$ samples are used. We partition the dataset such that the reference distribution draws distinct independent samples, and each model samples independently of the remainder of the pool. All algorithms receive the same model samples.
The kernel used is the Inverse Multiquadric (IMQ) on 2048 features extracted by the Inception-v3 network at
the pool layer \cite{szegedy2016rethinking}. Additionally, we use 50:50 split for  \relmulmmd{}.
Our baseline is the procedure of choosing the lowest Fr\'{e}chet Inception Distance (FID) \cite{heusel2017gans}. We note the
authors did not propose a statistical test with FID. Table 
\ref{tab:ex3_table} summaries the results from the experiment.

In Table \ref{tab:ex3_table}, we report the model-wise rejection rate (a high rejection indicts a poor candidate relatively speaking) and the model selection rate (which indicates the rate that the model has the smallest discrepancy from the given samples). The table illustrates several interesting points. First, even though
Model 1 shares the same portions as the true reference models, the quality of the generated images is a poor match to the reference images and thus is frequently rejected. A considerably higher FID score (than the rest) also supports this claim. Secondly, in this experiment, \textrm{MMD} is a good estimator of the best model for both \textrm{RelPSI} and \relmul{} (with splitting exhibiting higher variance) but the minimum FID score selects the incorrect model $73\%$ of the time. The additional testing indicate that Model 4 or Model 5 could be the best as they were rarely deemed worse than the best which is unsurprising given that their mixing proportions are closest to the true distribution. The low rejection for Model 4 is expected given that they differ by only $40$ samples. Model 2 and 3 have respectable model-wise rejections to indicate their position as worse than the best. Overall, both \textrm{RelPSI} and \relmul{} perform well and shows that the additional testing phase yields more information than the approach of picking the minimum of a score function (especially for FID).

%\newpage
\textbf{3. Density Comparison} $l=5$:
In our final experiment, we demonstrate \textrm{RelPSI-KSD} and \relmulksd{} on the Chicago data-set considered in Jitkrittum et al.\ \cite{jitkrittum2017linear} which consists of $11957$ data points. We split the data-set into disjoint sets such that $7000$ samples are used for training and the remainder for testing. For our candidate models, we trained a
Mixture of Gaussians (MoG) with expectation maximization with $C$ components where $C\in\{1,2,5\}$, Masked Auto-encoder for Density Estimation (MADE) \cite{germain2015made} and a Masked Auto-regressive Flow (MAF) \cite{papamakarios2017masked}. MAF with 1 autoregressive layer with a standard normal as the base distribution (or equivalently MADE) and MAF model has 5 autoregressive layers with a base distribution of a MoG (5). Each autoregressive layer is a feed-forward network with 512 hidden units. Both invertible models are trained with maximum likelihood with a small amount of $\ell_2$ penalty on the weights. In each trial, we sample $n=2000$ points independently of the test set. The resultant density shown in Figure \ref{fig:flow_dens} and the reference distribution in Figure \ref{fig:e5_ref}. We compare our result with the negative log-likelihood (NLL). Here we use the IMQ kernel.

The results are shown in Table \ref{tab:ex4_table}.
If performance is measured by a higher model-wise rejection rates,
for this experiment RelPSI-KSD performs better than \relmul{-KSD}. RelPSI-KSD suggests that MoG (1), MoG (2) and MADE are worse than the best but is unsure about MoG (5) and MAF. Whilst the only significant rejection of \relmul{-KSD} is MoG (1). These findings with RelPSI-KSD can be further endorsed by inspecting the density (see Figure \ref{fig:flow_dens}). It is clear that MoG (1), MoG (2) and MADE are too simple. But between MADE and MAF (5), it is unclear which is a better fit. Negative Log Likelihood (NLL) consistently suggest that MAF is the best which corroborates with our findings that MAF is one of the top models. The preference of MAF for NLL is due to log likelihood not penalizing the complexity of the model (MAF is the most complex with the highest number of parameters).
\begin{figure}[]
\centering
      \subfloat[Truth\label{fig:e5_ref}]{
        \includegraphics[width=.16\textwidth]{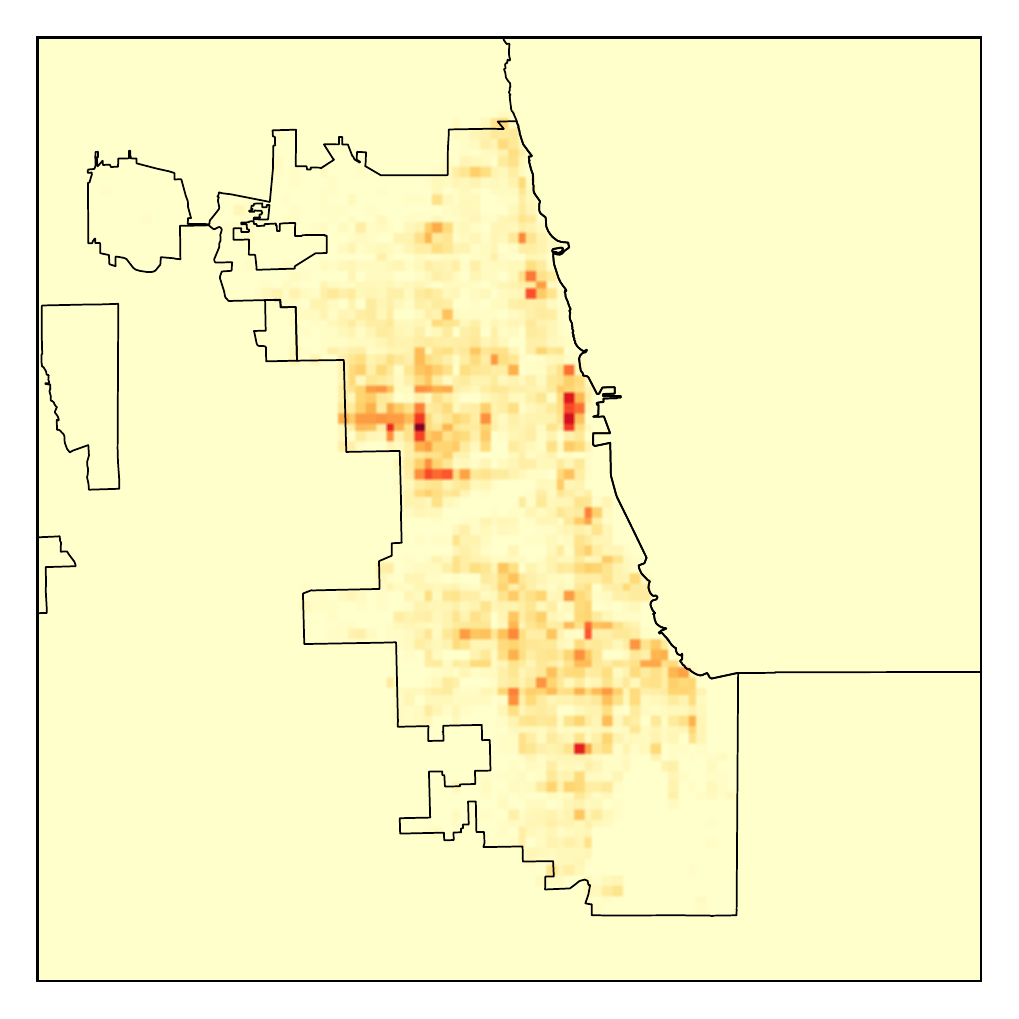}}
      \subfloat[MoG (1)\label{fig:e5_1}]{
        \includegraphics[width=.16\textwidth]{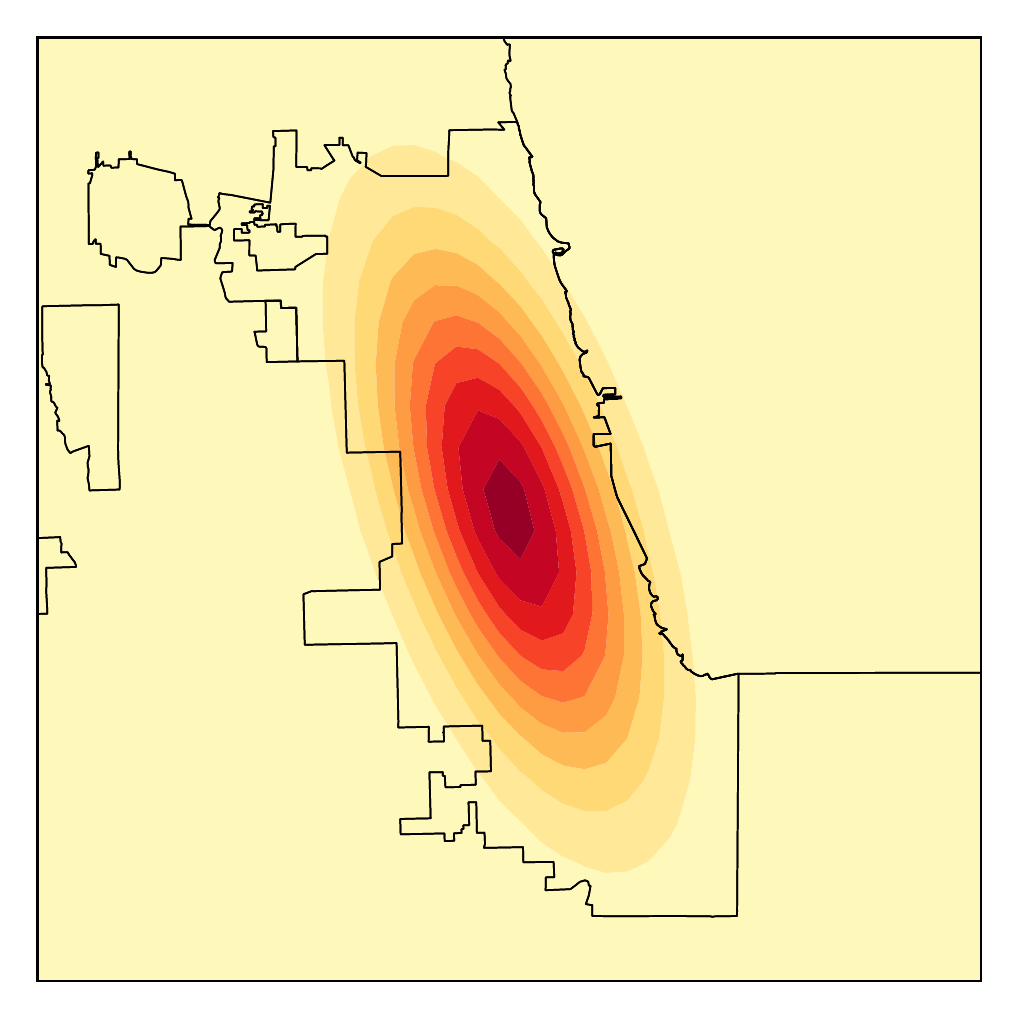}}
              \subfloat[MoG (2)\label{fig:e5_2}]{
        \includegraphics[width=.16\textwidth]{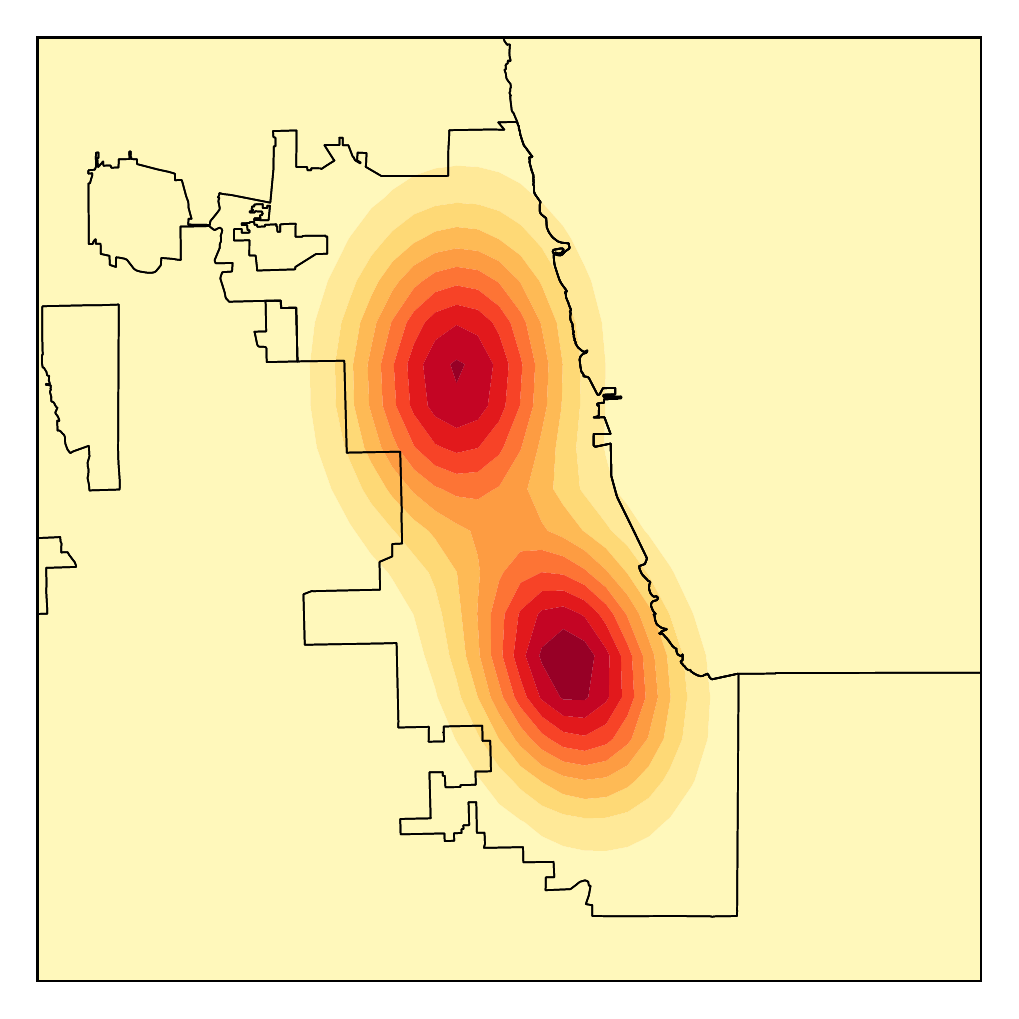}}
              \subfloat[MoG (5)\label{fig:e5_3}]{
        \includegraphics[width=.16\textwidth]{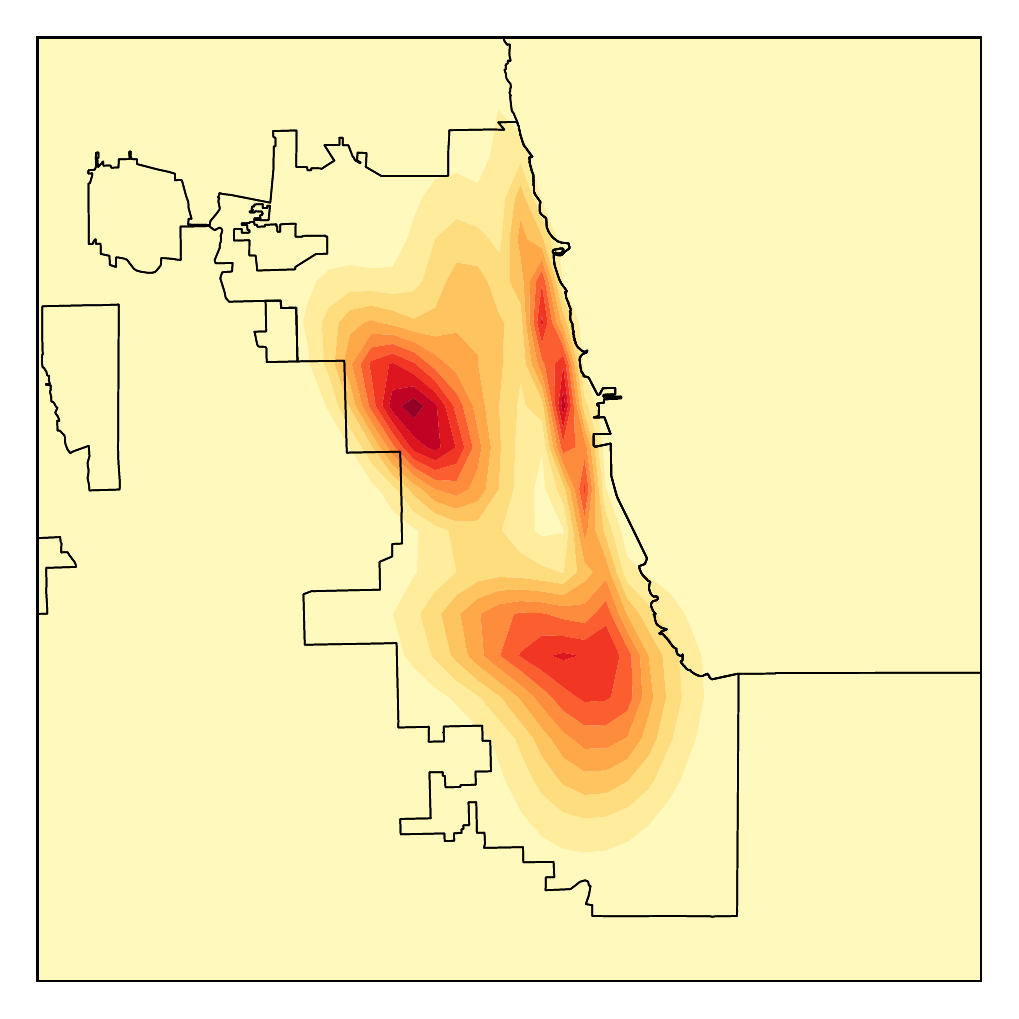}}
              \subfloat[MADE\label{fig:e5_4}]{
        \includegraphics[width=.16\textwidth]{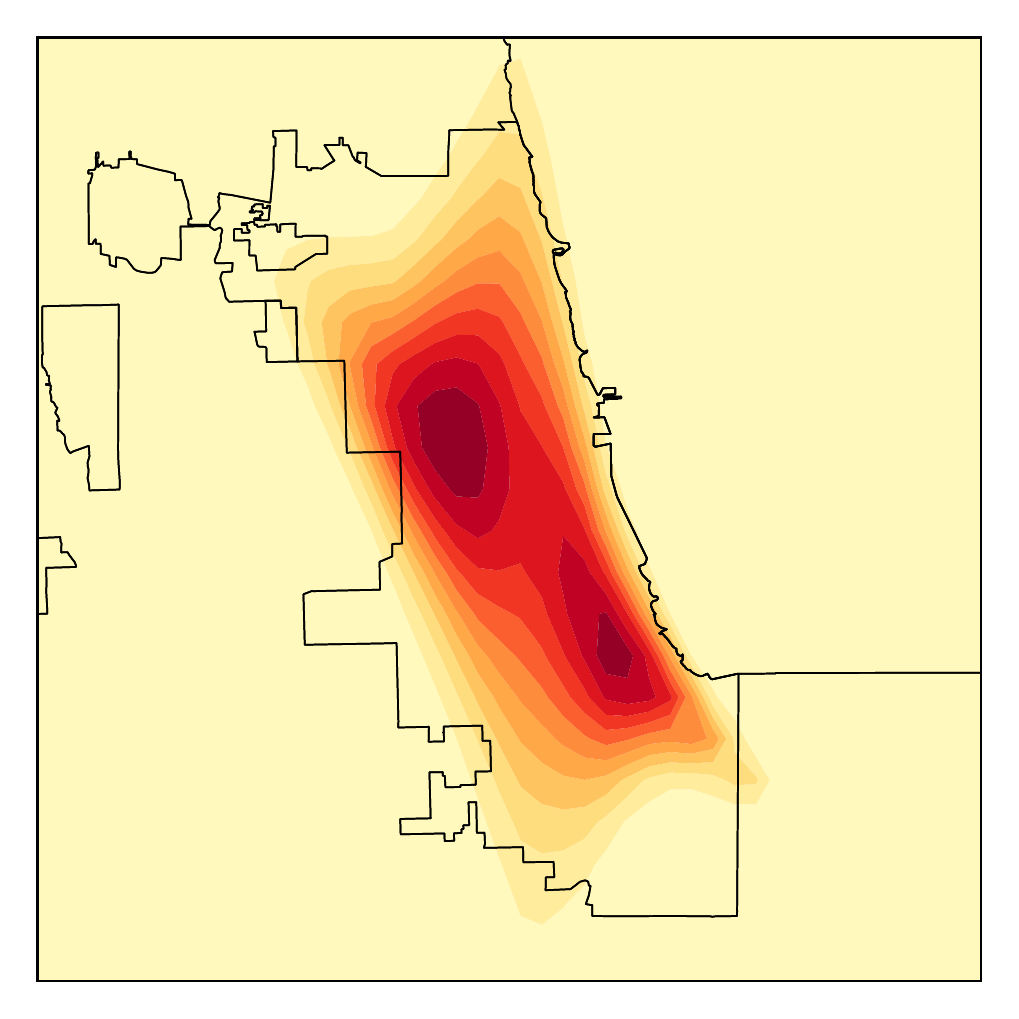}}
              \subfloat[MAF\label{fig:e5_5}]{
        \includegraphics[width=.16\textwidth]{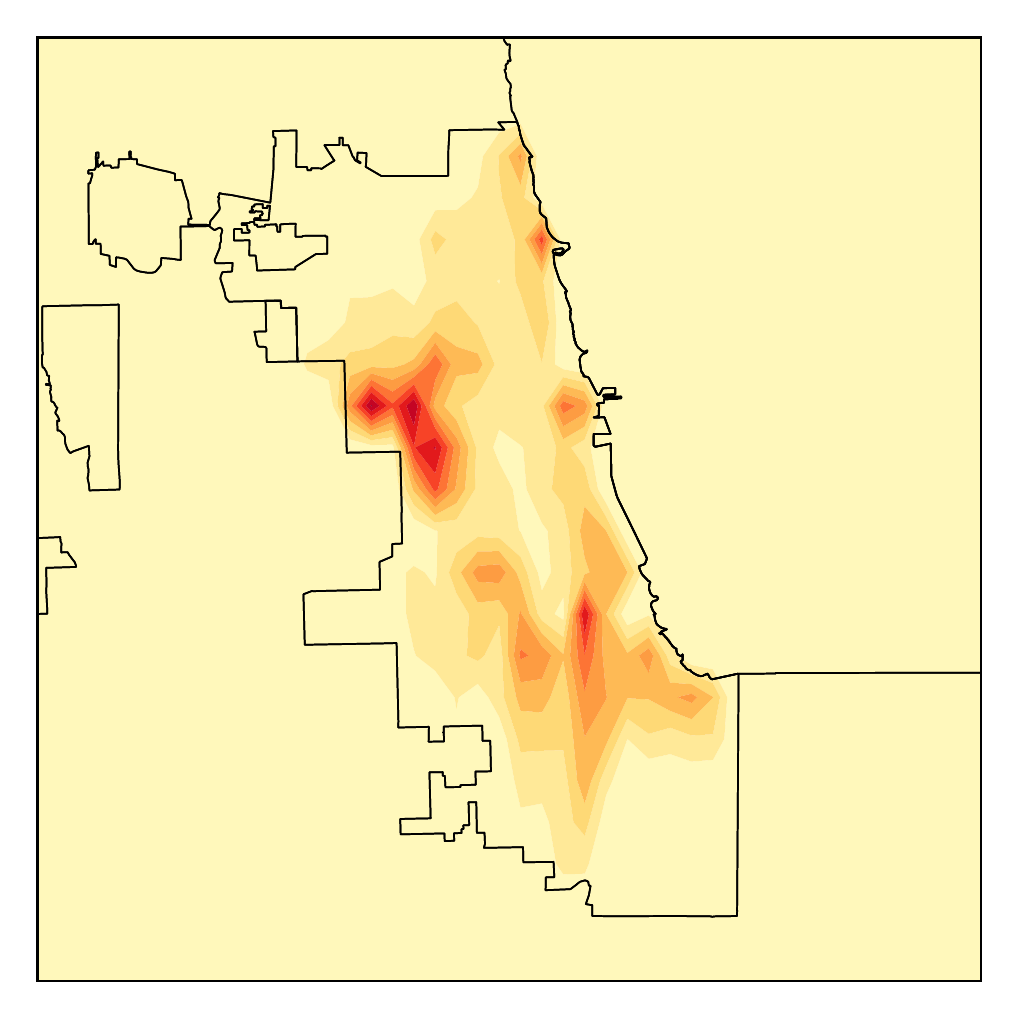}}
    \caption{The density plots of the trained models on the Chicago Crime dataset.
    }
    \label{fig:flow_dens}
\end{figure}
% \usepackage{booktabs}
%
%\begin{wraptable}{r}{0.53\textwidth} 
\begin{table}[ht]
\centering
\begin{tabular}{@{}lcccccc@{}}
\toprule
        & \multicolumn{2}{c}{\small RelPSI-KSD} & \multicolumn{2}{c}{\small RelMul-KSD} & \multicolumn{2}{c}{NLL} \\ \midrule
Model   & Rej.         & Sel.         & Rej.          & Sel.          & Aver.            & Sel. \\ \midrule
MoG (1) & 0.42         & 0.           & 0.22          & 0             & 2.64   & 0    \\
MoG (2) & 0.28         & 0.01         & 0.07          & 0.08          & 2.55   & 0    \\
MoG (5) & 0.02         & 0.62         & 0             & 0.38          & 2.38  & 0    \\ \midrule
MADE    & 0.26         & 0.01         & 0.04          & 0.03          & 2.53  & 0    \\
MAF (5) & 0          & 0.36         & 0             & 0.51          & 2.25  & 1.   \\ \bottomrule
\end{tabular}
\caption{Relative testing on unconditional density estimation models. The
    model-wise rejection rates, selection rates and Negative Log Likelihood (NLL) scores are reported.
    These results are averaged over 100 trials. 
    %PSI-KSD refers to \relpsiksd{}.
}
\label{tab:ex4_table}
\vspace{-4mm}
\end{table}
%\end{wraptable} 
% \usepackage{booktabs}
%\begin{table}[]

%\begin{tabular}{@{}lcccccc@{}}
%\toprule
%        & \multicolumn{2}{c}{RelPSI-KSD} & \multicolumn{2}{c}{RelMult-KSD} & \multicolumn{2}{c}{NLL} \\ \midrule
%Model   & Rej.         & Sel.         & Rej.          & Sel.          & Aver.            & Sel. \\ \midrule
%MoG (1) & 0.47         & 0.           & 0.13          & 0             & $2.64 \pm 0.01$  & 0    \\
%MoG (2) & 0.34         & 0.01         & 0.04          & 0.04          & 2.55 $\pm 0.01$  & 0    \\
%MoG (5) & 0.02         & 0.65         & 0             & 0.35          & 2.38$\pm 0.01$   & 0    \\ \midrule
%MADE    & 0.31         & 0.01         & 0.03          & 0.05          & 2.53$\pm 0.02$   & 0    \\
%MAF (5) & 0.0          & 0.33         & 0             & 0.56          & 2.25$\pm 0.02$   & 1.   \\ \bottomrule
%\end{tabular}
%\caption{Relative testing on unconditional density estimation models. The model-wise rejection rates, selection rates and FID scores are reported. These results are averaged over 100 trials.}
%\label{tab:ex4_table}
%\end{table}

%% file: sec_appendix.tex
\section{Definitions and FPR proof}
In this section, we define TPR and FPR, and prove that our proposals control FPR.

Recall the definitions of $\mathcal{I}_-$ and $\mathcal{I}_+$ (see Section \ref{sec:psi}).
$\mathcal{I}_-$ is the set of models that are \textit{not} worse than $P_J$ (the true best model).
$\mathcal{I}_+$ is the set of models that are worse than $P_J$.
We say that an algorithm decides that a model $P_i$ is positive if it decides that $P_i$ is worse than $P_J$.
We define true positive rate (TPR) and false positive rate (FPR) to be

\label{sec:prdefinition}
\begin{equation*}
\mathrm{FPR} =\frac{1}{|\mathcal{I}_{-}|}\mathbb{E}[|\{i\in\mathcal{I}_-:\text{
the algorithm decides that $P_i$ is positive}\}|],
\end{equation*}
\begin{equation*}
\mathrm{TPR} =\frac{1}{|\mathcal{I}_{+}|}\mathbb{E}[|\{i\in\mathcal{I}_+:\text{
the algorithm decides that $P_i$ is positive}\}|].
\end{equation*}

Both TPR and FPR can be estimated by averaging the TPR and FPR with multiple independent trials (as was done in Experiment \ref{sec:more_experiments}). We call this quantity the empirical TPR and FPR, denoted as $\widehat{\mathrm{TPR}}$ and $\widehat{\mathrm{FPR}}$ respectively.

The following lemma shows that our proposals controls FPR.

\fpr*
\begin{proof}
\label{proof:fpr}
From law of total expectation, we have
\begin{align*}
\mathrm{FPR} &=\frac{1}{\vert \mathcal{I}_{-}|}\mathbb{E}[|\{i\in\mathcal{I}_-:\text{ the algorithm decides that $i$ is positive}\}|] \\
&=\frac{1}{|\mathcal{I}_{-}|}\mathbb{E}[\mathbb{E}[|\{i\in\mathcal{I}_-:\text{ the algorithm decides that $i$ is positive}\}|\,|\, P_{\hat{J}} \text{ is selected}]]  \\
&=\frac{1}{|\mathcal{I}_{-}|}\mathbb{E}[\mathbb{E}[\sum_{i\in\mathcal{I}_-}\mathbb{I}(\text{The algorithm decides that $i$ is positive})\,|\,P_{\hat{J}} \text{ is selected}]]  \\
&= \frac{1}{|\mathcal{I}_{-}|}\mathbb{E}[\sum_{i\in\mathcal{I}_-}\mathbb{P}(\text{The algorithm decides that $i$ is positive}\,|\, P_{\hat{J}} \text{ is selected})] \\
&= \frac{1}{|\mathcal{I}_{-}|}\mathbb{E}[\sum_{i\in\mathcal{I}_-}\mathbb{P}(\sqrt{n}[\hat{D}(P_i,R)-\hat{D}(P_{\hat{J}},R)] > \hat{t}_\alpha\,|\,\hat{J} \text{ is selected})] \\
&\le \frac{1}{|\mathcal{I}_{-}|}\mathbb{E}[\sum_{i\in\mathcal{I}_-}\alpha] \\
&= \frac{1}{|\mathcal{I}_{-}|}\mathbb{E}\big [|I_{-}| \alpha \big ] \\
&= \alpha,
\end{align*}
where $\mathbb{I}$ is the indicator function.
\end{proof}

%-----------------------------
\section{Algorithms}
Algorithms for \relpsi\ (see Algorithm \ref{algorithm:relpsi}) and \relmul\
(see Algorithm \ref{algorithm:relmul}) proposed in Section
\ref{section:proposal} are provided in this section.
\begin{algorithm}[ht]
    \caption{\relpsi\ $H_{0,i}: D(P_{\hat{J}},R)\ge D(P_i,R)\ | \
    P_{\hat{J}}$ is selected.
    %\wjsay{Unfortunately I don't think we have space for a full algorithm in the main text. I would put in the appendix.}
    }
    \begin{algorithmic}[1]
        \Procedure{\relpsi{}}{$\mathcal{M}, R, \alpha$}
        \State{Estimate $\hat{\bm\Sigma}$ given in Theorem \ref{theorem:ksd} and Theorem \ref{theorem:mmd} for KSD and MMD respectively.}
        \State{$\bm{r} \gets (0, \ldots, 0) \in \{0,1\}^l$ }
        \State{$\bm{z} \gets [\sqrt{n}\hat{D}(P_1,R), \sqrt{n}\hat{D}(P_2,R) ,\ldots, \sqrt{n}\hat{D}(P_l,R)]^\top$}
        \State{$\hat{J} \gets \arg\min_{j \in \mathcal{I}}\hat{D}(P_j,R)$.}
        \State{Compute $\bm{A}$ and $\bm{b}$ (as defined in Section \ref{sec:selection}).}
        \For{$i\in \mathcal{I} : i \neq \hat{J}$}
            \State{$\bm\eta \gets [0,\ldots,\overbrace{-1}^{\hat{J}},\ldots,\overbrace{1}^{i},\ldots,0]^\top$}
            \State{$\hat{\sigma} \gets \sqrt{\bm\eta^\top\hat{\bm\Sigma}\bm\eta}$}
            \State{Compute $\mathcal{V}^+$ and $\mathcal{V}^-$ (described in Lemma \ref{theorem:poly}).}
            \State{$\hat{t}_\alpha \gets \hat{\sigma}\Phi^{-1}\bigg((1-\alpha)\Phi\left(\frac{{\mathcal{V^+}}}{\hat{\sigma}}\right)+\alpha\Phi\left(\frac{{\mathcal{V^-}}}{\hat{\sigma}}\right)\bigg)$}
            \State{ $\bm{r}_i\gets\bm\eta^\top\bm{z} > \hat{t}_\alpha$ }
        \EndFor\\
        \Return{$\bm{r}$}
        \EndProcedure
    \end{algorithmic}
    \label{algorithm:relpsi}
\end{algorithm}

\begin{algorithm}[ht]
    \caption{\relmul\ $H_{0,i}: D(P_{\hat{J}},R)\ge D(P_{i},R)\ | \ P_{\hat{J}}$ is selected.}
    \begin{algorithmic}[1]
        \Procedure{\relmul}{$\mathcal{M}, R, \alpha, \rho$}
        \State{Estimate $\hat{\bm\Sigma}$ as given in Theorem \ref{theorem:ksd} and Theorem \ref{theorem:mmd} for KSD and MMD respectively.}
        \State{$\mathcal{D}_0,\,\mathcal{D}_1 \gets$  \texttt{SplitData}($\mathcal{M}, R, \rho$)}
        \State{$n_1 \gets \rho n$}
        \State{(With $\mathcal{D}_0$) $\hat{J} \gets \arg\min_{j\in \mathcal{I}}\hat{D}(P_j,R)$.}
        \State{Compute $\bm{A}$ and $\bm{b}$. }
        \For{$i\in\mathcal{I}: i \neq \hat{J}$}\ (with $\mathcal{D}_1$)
                \State{Compute $\bm{z}_2 = [\sqrt{n_1}\hat{D}(P_1,R), \sqrt{n_1}\hat{D}(P_2,R) ,\ldots, \sqrt{n_1}\hat{D}(P_l,R)]^\top$}
                \State{$\bm\eta^\top = [0,\ldots,\overbrace{-1}^{\hat{J}}\ldots,\overbrace{1}^{i},\ldots,0]$}
                \State{$\hat{\sigma} \gets \sqrt{\bm\eta^\top\hat{\bm\Sigma}\bm\eta}$}
                \State{$\bm{p}_i \gets 1-\Phi(\frac{\bm\eta^\top\bm{z}_2}{\hat{\sigma}})$}
        \EndFor\\
        \Return{\texttt{FDRCorrection}($\bm{p}, \alpha$)}
        \EndProcedure
    \end{algorithmic}
    \label{algorithm:relmul}
\end{algorithm}
In algorithm \ref{algorithm:relmul}, \texttt{FDRCorrection}($\bm{p}, \alpha$) takes a list of $p$-values $\bm{p}$ and returns a list of rejections for each element of $\bm{p}$ such that the false discovery rate is controlled at $\alpha$. In our experiments, we use the Benjamini–Yekutieli procedure \cite{benjamini2001control}.
\texttt{SplitData}($\mathcal{M}, R, \rho$) is a function that splits the samples generated by $R$ and $\mathcal{M}$ (if it is represented by samples). It returns two datasets $\mathcal{D}_0$ and $\mathcal{D}_1$ such that $|\mathcal{D}_0| = (1-\rho)n$ and $|\mathcal{D}_1| = \rho n$.
%\wjsay{Need to define all symbols used. FDRCorrection is undefined. $m_1$}

%-----------------------------
\section{Asymptotic distributions}
In this section, we prove the asymptotic distribution of $\widehat{\ksd}^2_u$ and also provide the asymptotic distribution of $\widehat{\mmd}^2_u$ for completeness.
%\ksdasmpy*

\begin{restatable}[Asymptotic Distribution of $\bighat{\ksd^2_u}$]{theorem}{ksdasmpy}
    \label{theorem:ksd}
    Let $P_i, P_j$ be distributions with density functions $p_i, p_j$
    respectively, and let $R$ be the data generating distribution. Assume that
    $P_i, P_j, R$ are distinct. We denote a sample by $Z_n = Z \sim R$. 
    %Furthermore, we assume that $\mathbb{E}[k(r,r)] < \infty$,$\mathbb{E}[k(x,r)] < \infty$,
    %$\mathbb{E}[k(y,r)] < \infty$, 
    %\wjsay{These conditions look arbitrary. We will get questions.}
$
        \sqrt{n}
        \Bigg (
        \begin{pmatrix}
            \bighat{\ksd}^2_u(P_i,Z) \\
            \bighat{\ksd}^2_u(P_j,Z)
        \end{pmatrix}
        -
        \begin{pmatrix}
            \ksd^2(P_i,R) \\
            \ksd^2(P_j,R)
        \end{pmatrix}
        \Bigg )
        \xrightarrow{d}
        \mathcal{N}
        (
        \bm{0}
        ,
        \bm\Sigma
    ),
$
    where $\bm\Sigma =         \begin{pmatrix}
            \sigma^2_{P_iR}\quad \sigma_{P_iRP_jR} \\
            \sigma_{P_iRP_jR}\quad  \sigma^2_{P_jR} \\
        \end{pmatrix},\ \sigma_{P_iRP_jR}=\mathrm{Cov}_{x\sim R} [\E_{x'\sim R}
        [u_{p_i}(x,x')], \E_{x'\sim R} [u_{p_j}(x,x')]]$ and $ \sigma^2_{P_iR}=
        \mathrm{Var}_{x'\sim R} [\E_{x'\sim R} [u_{p_i}(x,x')]$.
            %\wjsay{Let's change $m$ in the proof to $n$. Above, $Z$ is $R$?}
\end{restatable}
\begin{proof}
    \label{proof:asymp_ksd_u}
    Let $\mathcal{\mathcal{\mathcal{\mathcal{{X}}}}}=\{x_{i}\}_{i=1}^{n}$
be $n$ i.i.d. random variables drawn from $R$ and we have a model
with its corresponding gradient of its log density $s_{P_i}(x)=\nabla_{x}\log p_i(x)$.
The complete U-statistic estimate of KSD between $P_i$ and $R$ is
\[
\bighat{\ksd}_{u}^{2}(P_i,R)=\mathbb{E}_{x,x'\sim R}[u_{p_i}(x,x')]\approx\frac{{1}}{n_{2}}\sum_{i\neq j}^{n}u_{p_i}(x_{i},x_{j})
\]
where $u_{p_i}(x,y)=s_{p_i}(x)^{\top}s_{p_i}(y)k(x,y)+s_{p_i}(y)^{\top}\nabla_{x}k(x,y)+s_{p_i}(x)^{\top}\nabla_{y}k(x,y)+\Tr[\nabla_{x,y}k(x,y)]$
and $n_{2}=n(n-1)$.

Similarly, for another model $P_j$ and its gradient of its log density $s_{P_j}(x)=\nabla_{x}\log p_j(x)$.
Its estimator is
\[
\bighat{\ksd}_{u}^{2}(P_j,R)=\mathbb{{E}}_{x,x'\sim R}[u_{p_j}x,x')]\approx\frac{{1}}{n_{2}}\sum_{i\neq j}^{n}u_{p_j}(x_{i},x_{j})
\]
where $u_{p_j}(x,y)=s_{p_j}(x)^{\top}s_{p_j}(y)k(x,y)+s_{p_j}(y)^{\top}\nabla_{x}k(x,y)+s_{p_j}(x)^{\top}\nabla_{y}k(x,y)+\Tr[\nabla_{x,y}k(x,y)]$.

The covariance matrix of a U-statistic with a kernel of order
2 is

\[
\bm\Sigma = \frac{4(n-2)}{n(n-1)}\bm\zeta+\mathcal{O}_p(n^{-2})
\]

where, for the variance term and covariance term, we have $\zeta_{ii}=\mathrm{Var}_{x\sim R}(\mathbb{{E}}_{y\sim R}[u_{p_i}(x,y)])$
and $\zeta_{ij}=\Cov_{x\sim R}(\mathbb{{E}}_{y\sim R}[u_{p_i}(x,y)],\mathbb{{E}}_{y\sim R}[u_{p_j}(x,y)])$
respectively.
\end{proof}
The asymptotic distribution is provided below and is shown to be the case by Bounliphone et al.\ \cite{bounliphone2015test}.
\begin{theorem}[Asymptotic Distribution of $\bighat{\mmd}^2_u$ \cite{bounliphone2015test}]
    \label{theorem:mmd}
    Assume that $P_i$, $P_j$ and $R$ are distinct. We denote samples $X \sim P_i$, $Y \sim P_j$, $Z \sim R$.
    \begin{equation*}
        \sqrt{n}
        \bigg (
        \begin{pmatrix}
            \bighat{\mmd}^2_u(P_i,Z) \\
            \bighat{\mmd}^2_u(P_j,Z)
        \end{pmatrix}
        -
        \begin{pmatrix}
            \mmd^2(P_i,R) \\
            \mmd^2(P_j,R)
        \end{pmatrix}
        \bigg )
        \xrightarrow{d}
        \mathcal{N}
        (
        \bm{0}
        ,
        \bm\Sigma )
    \end{equation*}

    where $\bm\Sigma =        \begin{pmatrix}
            \sigma^2_{P_iR}\quad \sigma_{P_iRP_jR} \\
            \sigma_{P_iRP_jR}\quad  \sigma^2_{P_jR} \\
        \end{pmatrix}$ and $\sigma_{P_iRP_jR}=\Cov[\E_{x'\sim P_i\times R} [h(X,x')], \E_{x'\sim P_j\times R} [g(X,x')]]$ and $ \sigma^2_{P_j R}= \Var [\E_{x'\sim P_j\times R} [h(X,x')]$.
\end{theorem}

%-----------------------------
\section{Relative Kernelized Stein Discrepancy (\textrm{RelKSD})}
\label{sec:RelKSD}
In this section, we describe the testing procedure for relative tests with KSD (a simple extension of \textrm{RelMMD} \cite{bounliphone2015test}).
Currently, there is no test of relative fit with Kernelized Stein Discrepancy, and so we propose such a test using the complete estimator $\bighat{\ksd}^2_u$ which we call \textrm{RelKSD}. The test mirrors the proposal of Bounliphone et al.\ \cite{bounliphone2015test} and, given the asymptotic distribution of $\bighat{\ksd}^2_u$, it is a simple extension since its
cross-covariance is known (see Theorem \ref{theorem:ksd}).

Given two candidate models $P_1$ and $P_2$ with a reference distribution $R$ with its samples denoted as $Z \sim R$, we define our test statistic as  $\sqrt{n}[\bighat{\ksd}^2_u(P_1,Z) -\bighat{\ksd}^2_u(P_2,Z)]$. For the test of relative similarity,
we assume that the candidate models ($P_1$ and $P_2$) and unknown generating distribution $R$ are all distinct. Then, under the null hypothesis $H_0: \ksd^2(P_1,R) - \ksd^2(P_2,R) \le 0$, we can derive the asymptotic null distribution as follows. By the continuous mapping theorem
and Theorem \ref{theorem:ksd}, we have
%\wjsay{do not write $\sqrt{n}$ on the RHS}
%\jlsay{Is this better?}
%
\begin{equation*}
    \sqrt{n}[\bighat{\ksd}^2_u(P_1,Z) -\bighat{\ksd}^2_u(P_2,Z)]
    - \sqrt{n}[\ksd^2(P_1,R)-\ksd^2(P_2,R)] \xrightarrow{d} \mathcal{N}\bigg(
    0, \sigma^2
       \bigg )         
\end{equation*}
where $\sigma^2 =\begin{pmatrix}
        1 \\
        -1
    \end{pmatrix}^\top
         \bm{\Sigma}
        \begin{pmatrix}
        1 \\
        -1
    \end{pmatrix} = \sigma^2_{P_1 R} - 2\sigma_{P_1 R P_2 R} + \sigma^2_{P_2
    R},$ 
    %\wjsay{Jenning, please check this expression}
    %\jlsay{I think this is right, it looks slightly different from your result in \cite{jitkrittum2018informative} because I absorbed the $4$ into $\bm\Sigma$} 
    and $\bm\Sigma$ is
    defined in Theorem \ref{theorem:ksd} (which we assume is
    positive definite). We will also use the most conservative
threshold by letting the rejection threshold $t_\alpha$ be the
$(1-\alpha)$-quantile of the asymptotic distribution of
$\sqrt{n}[\bighat{\ksd}^2_u(P_1,Z)-\bighat{\ksd}^2_u(P_2,Z)]$ with mean zero
(see Bounliphone et al.\cite{bounliphone2015test}). If our statistic is above
the $t_\alpha$, we reject the null.

%-----------------------------
\section{Calibration of the test}
\label{sec:callibration}
In this section, we will show that the p-values obtained are well calibrated, when two distributions are equal, measured by either \textrm{MMD} or \textrm{KSD}. The distribution of p-values should be uniform.
Figure \ref{fig:appen_ecdf} shows the empirical CDF of p-values and should lie on the line if it is calibrated. Additionally, we show the empirical distribution of p values for a three of different mean shift problems where observed distribution is $R=\mathcal{N}(0,1)$ and our candidate models are $P_1=\mathcal{N}(\mu_1,1)$ and $P_2=\mathcal{N}(\mu_2,1)$

\begin{figure}[h]
    \centering
    \includegraphics[width=\linewidth]{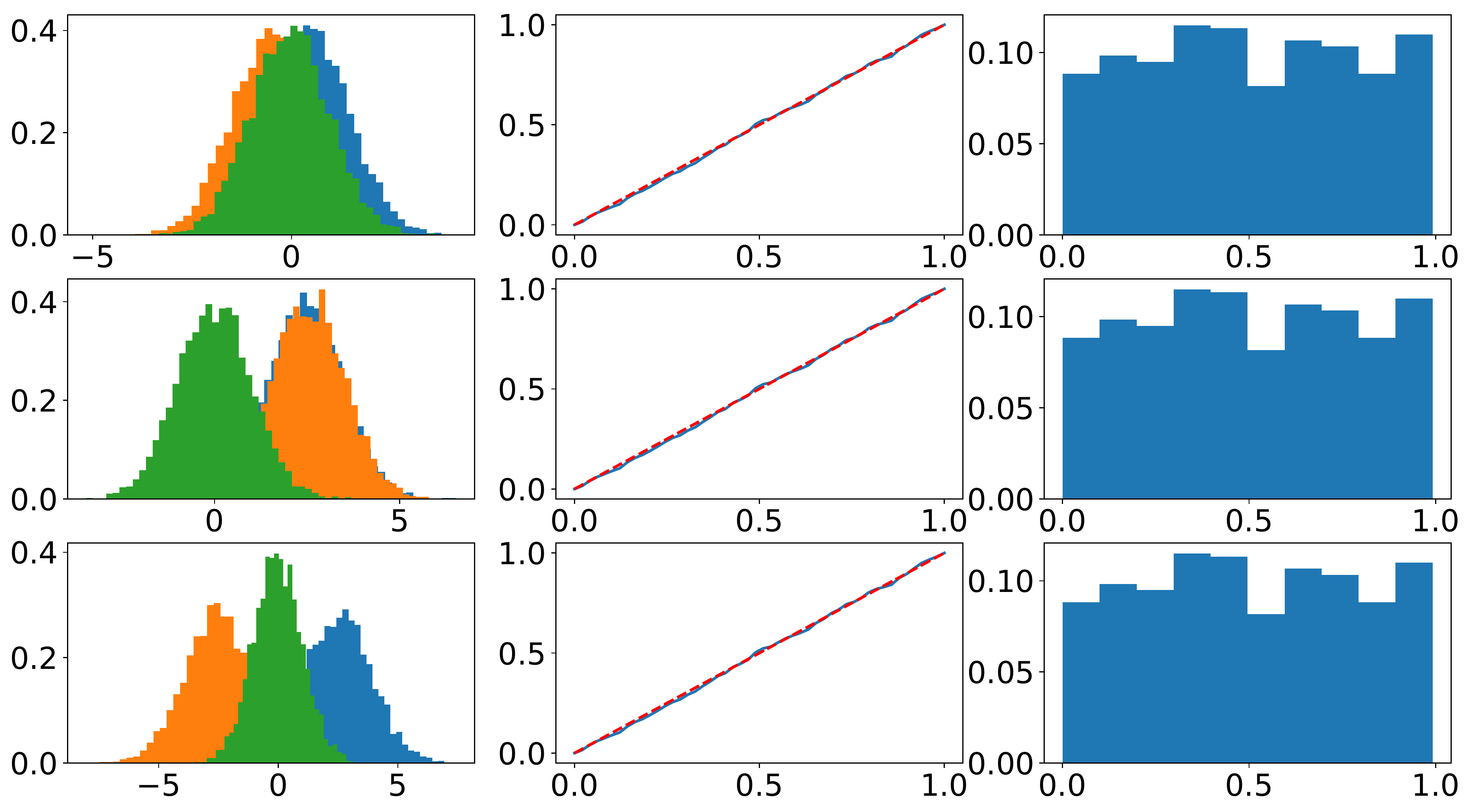}
    \caption{Mean shift experiment described in experiments with
    (a) $\mu_1 = 0.5$ and $\mu_2 = -0.5$,
    (b) $\mu_1 = 2.5$  and $\mu_2 = 2.5$,
    (c)  $\mu_1 = 2.5$ and $\mu_2 = -2.5$.}
    \label{fig:appen_ecdf}
\end{figure}

%-----------------------------
%Throughout this section, we assume that there are $l=2$ candidate models and
%that $D(P_1, R) < D(P_2, R)$ i.e., $P_1$ is the best one.

%Unlike in Section
%\ref{sec:selection}, in this section, we consider the following hypotheses
%\begin{align*}
%     H^{\hat{J}}_{0}\colon &D(P_2,R) -  D(P_1,R) \le  0
%      \ |\ P_{\hat{J}}  \text{ is selected as the reference,} \\
%    H^{\hat{J}}_{1}\colon & D(P_2,R) - D(P_1,R) > 0 
%    \ |\ P_{\hat{J}}  \text{ is selected as the reference,}
%\end{align*}
%where $\hat{J} \in \{1,2\}$.

\input{appen_power.tex}

\section{Additional experiments}
\label{sec:more_experiments}
In this section, we show results of two experiments. The first
investigates the behaviour of \textrm{RelPSI} and \relmul{} for multiple candidate models; and the second focuses on empirically verifying the implication of Theorem \ref{theorem:tpr}.
\subsection{Multiple candidate models experiment}
\label{sec:appen_mm}
In the following experiments, we demonstrate our proposal for synthetic problems when there are more than two 
candidate models and report the empirical true positive rate $\widehat{\mathrm{TPR}}$, empirical false discovery rate $\widehat{\mathrm{FDR}}$, and empirical false positive rate $\widehat{\mathrm{FPR}}$.
We consider the following problems:
\begin{enumerate}
    \item{\textit{Mean shift $l = 10$:} There are many candidate models that are equally good.
    We set nine models to be just as good, compared to the reference $R = \mathcal{N}(\bm{0},\bm{I})$, with one model that is worse than all of them. To be specific, the set of equally good candidates are defined as $\mathcal{I}_-=\{\mathcal{N}(\bm\mu_i, \bm{I}):
    \bm\mu_i \in \{[0.5,0,\ldots,0],[-0.5,0,\ldots,0],[0, 0.5,\ldots,0],[0,-0.5,\ldots,0], \ldots\} \}$
    and for the worst model, we have $Q = \mathcal{N}([1,0,\ldots,0], \bm{I})$. Our candidate model list is defined as 
    $\mathcal{M} = I_- \cup \{Q\}$. Each model is defined on $\mathbb{R}^{10}$.}
    %where $\mu^c_i \in {\{[j,0,\ldots,0]: j \in [2,\ldots,5]}$}
    \item{\textit{Restricted Boltzmann Machine (RBM) $l=7$:} This experiment is similar to Experiment 1.
    Each candidate model is a Gaussian Restricted Boltzmann Machine with different perturbations of the unknown RBM parameters (which generates our unknown distribution $R$). We show how the behaviour of our proposed test vary with the degree
        of perturbation $\epsilon$ of a single model while the rest of the candidate models remain the same. The perturbation changes the model from the best to worse than the best. Specifically, we have $\epsilon \in \{ 0.18, 0.19, 0.20, 0.22 \}$ and
        the rest of the six models have fixed perturbation of
        $\{0.2, 0.3, 0.35, 0.4, 0.45, 0.5\}$.
        This problem demonstrates the sensitivity of each test.}
\end{enumerate}
The results from the mean shift experiment are shown in Figure \ref{fig:appen_ex_ms} 
and results from RBM experiment are shown in Figure \ref{fig:appen_ex_rbm}.
Both experiments show that $\widehat{\mathrm{FPR}}$ and $\widehat{\mathrm{FDR}}$ is controlled for \relpsi\ and \relmul\ respectively.
As before, KSD-based tests exhibits the highest $\widehat{\mathrm{TPR}}$ in the RBM experiment. In the mean shift example, \relpsi{} has lower $\widehat{\mathrm{TPR}}$ compared with \relmul{} and is an example where condition
on the selection event results in a lower $\widehat{\mathrm{TPR}}$ (lower than data splitting). In both experiments, for \relmul{} $50\%$ of the data is used for selection and $50\%$ for testing.
\begin{figure}[ht]
    \centering
    \includegraphics[width=\textwidth]{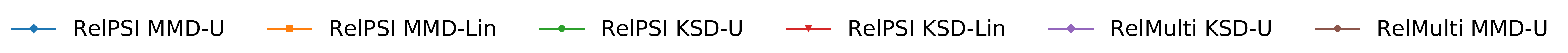}
    \subfloat[Empirical FDR]{
    \includegraphics[width=.32\textwidth]{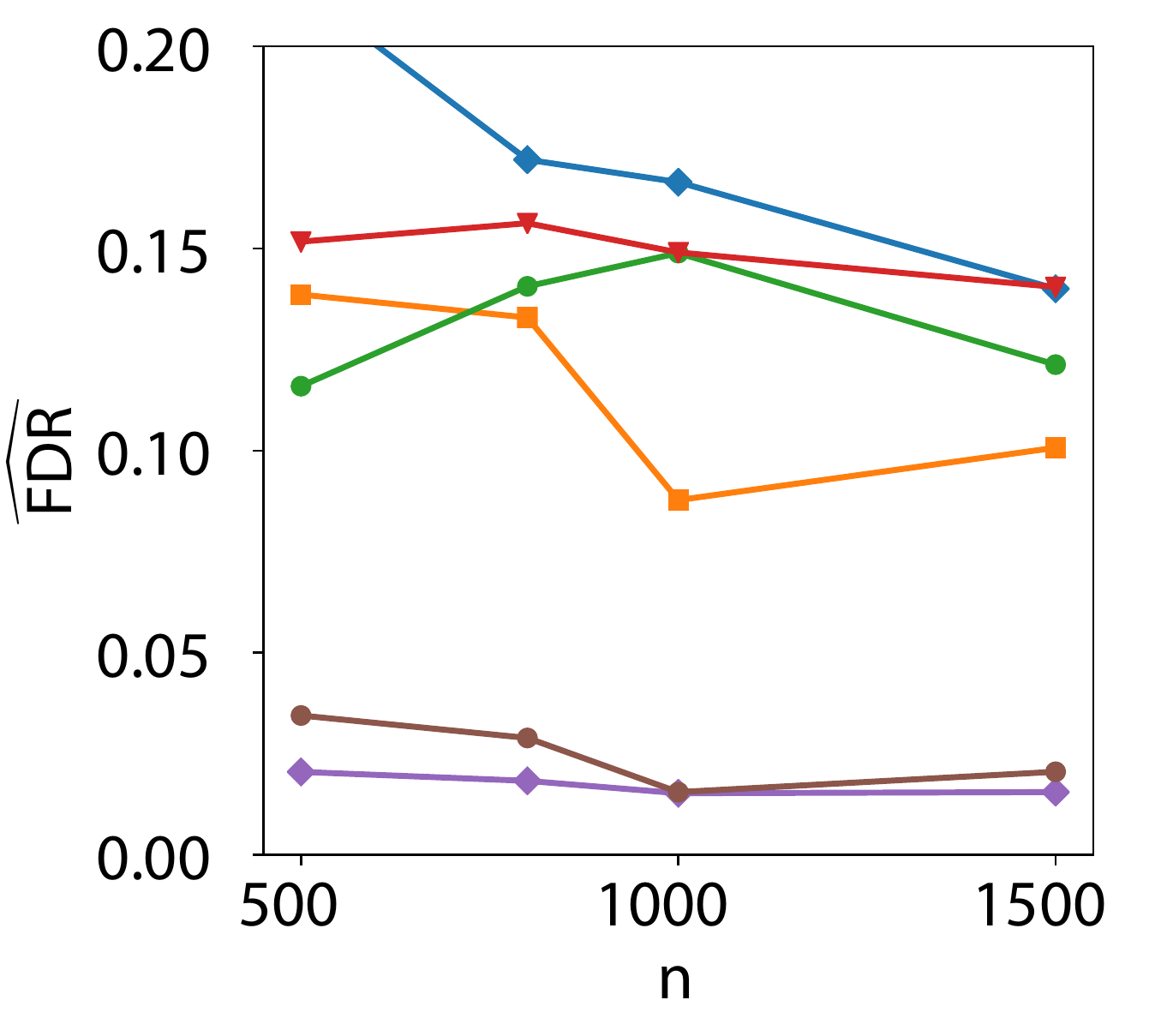}}
        \subfloat[Empirical TPR]{
        \includegraphics[width=.33\textwidth]{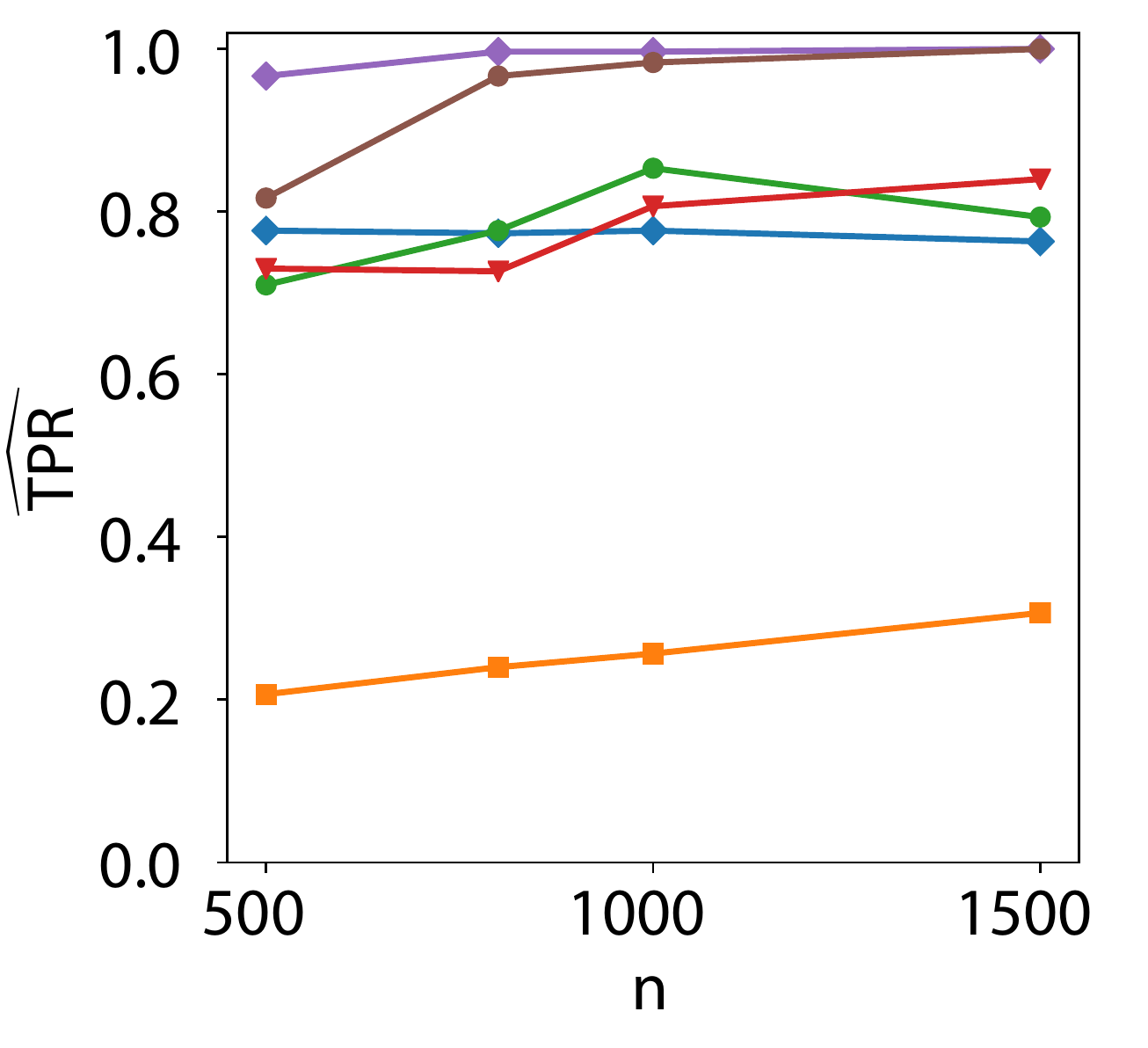}}
        \subfloat[Empirical FPR]{
        \includegraphics[width=.33\textwidth]{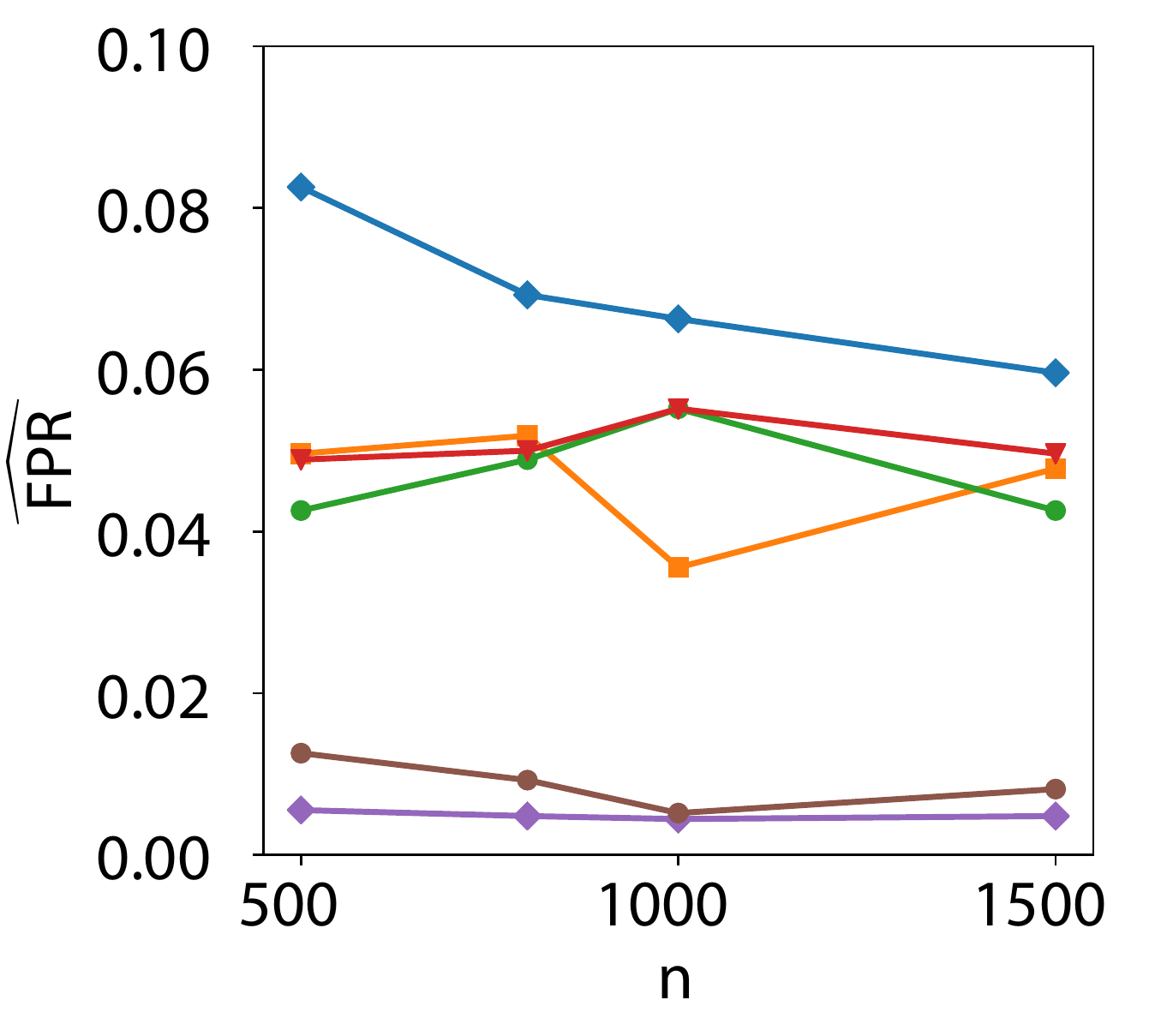}}
    \caption{Mean Shift Experiment: Rejection rates (estimated from 300 trials) for the six tests with $\alpha = 0.05$ is shown.}
    \label{fig:appen_ex_ms}
\end{figure}
\begin{figure}[ht]
    \centering
    \includegraphics[width=\textwidth]{graphs/ex3/blobs_hlegend.pdf}
    \subfloat[Empirical FDR]{
        \includegraphics[width=.32\textwidth]{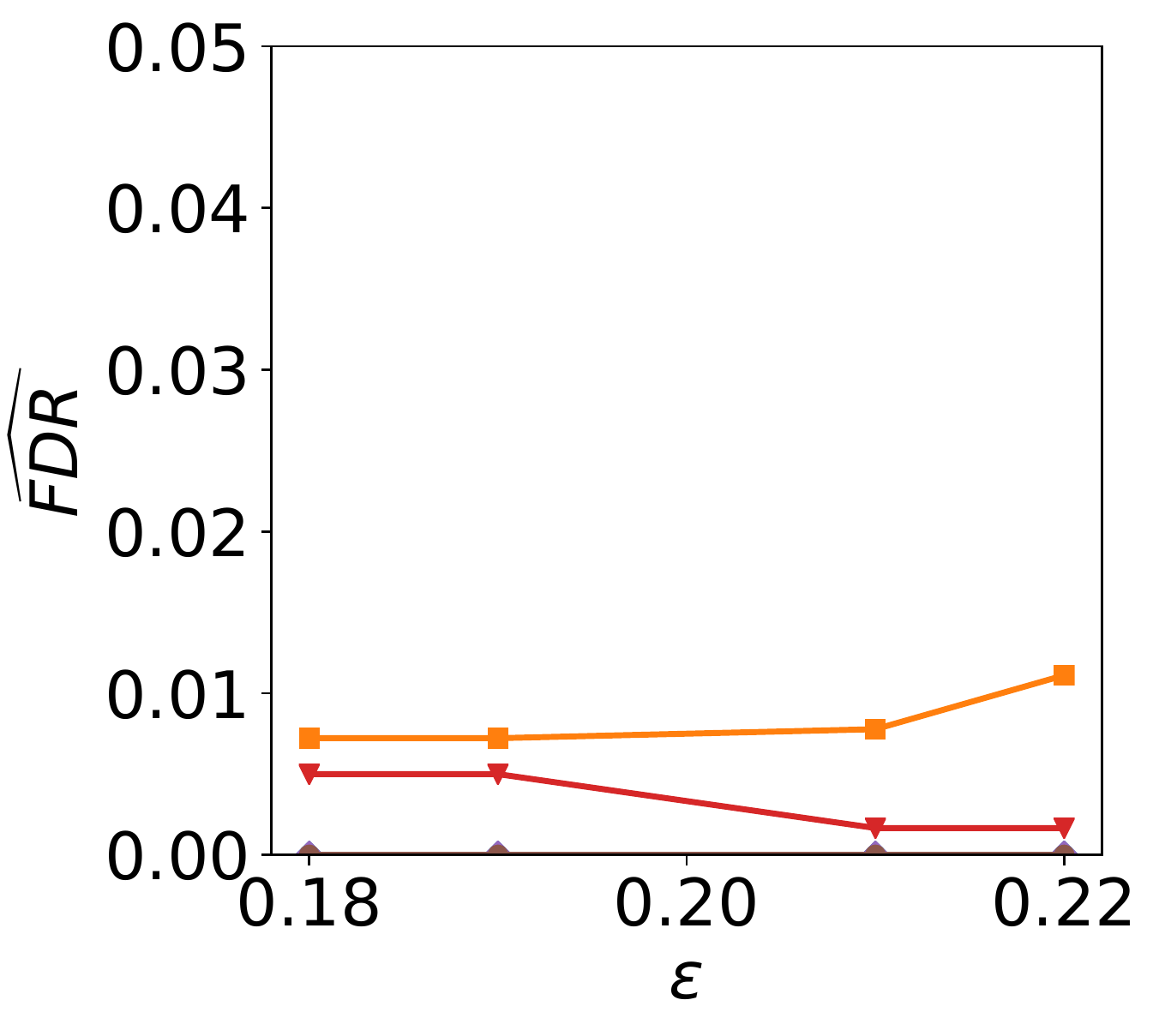}}
        \subfloat[Empirical TPR]{
        \includegraphics[width=.33\textwidth]{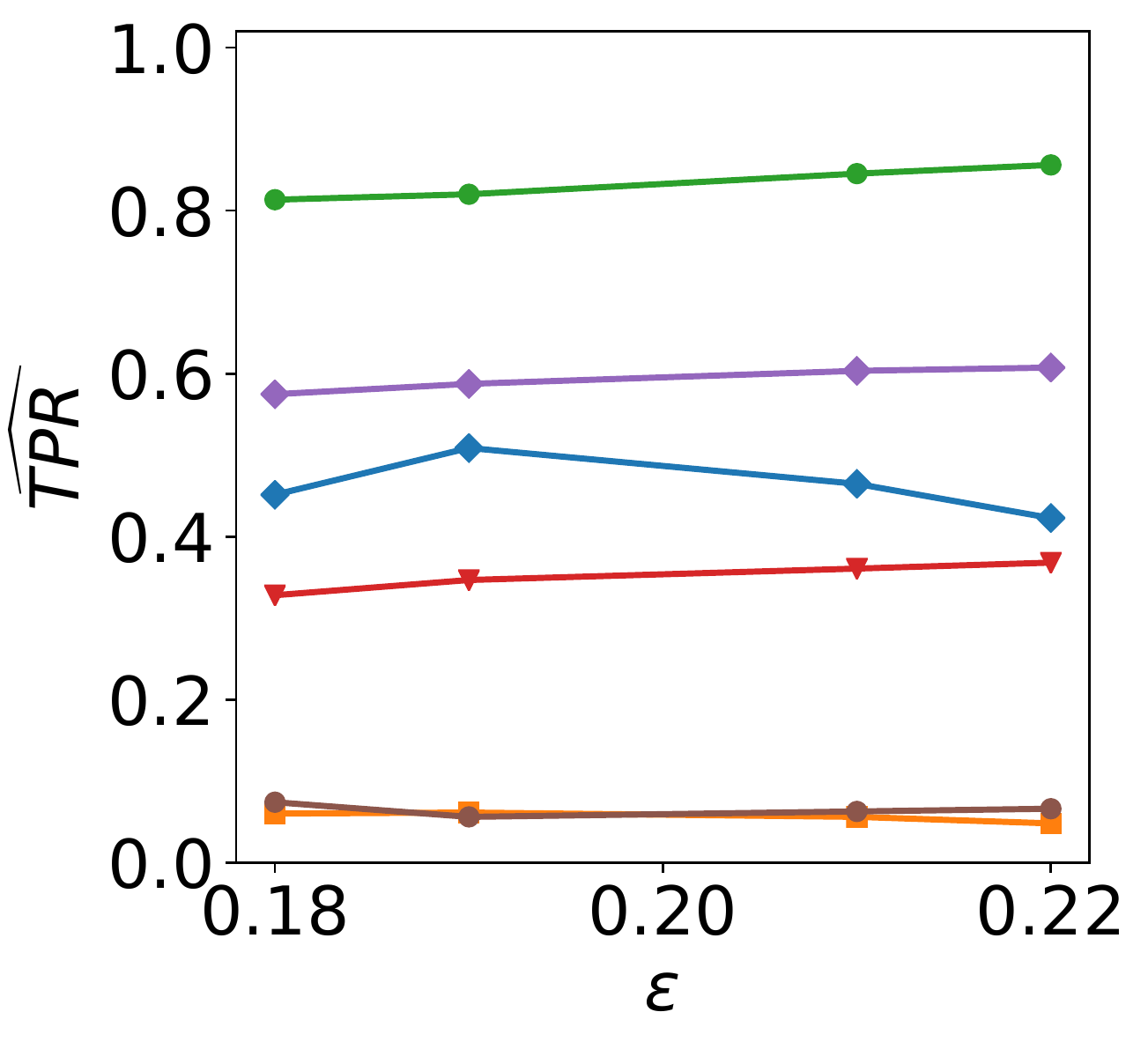}}
        \subfloat[Empirical FPR]{
        \includegraphics[width=.33\textwidth]{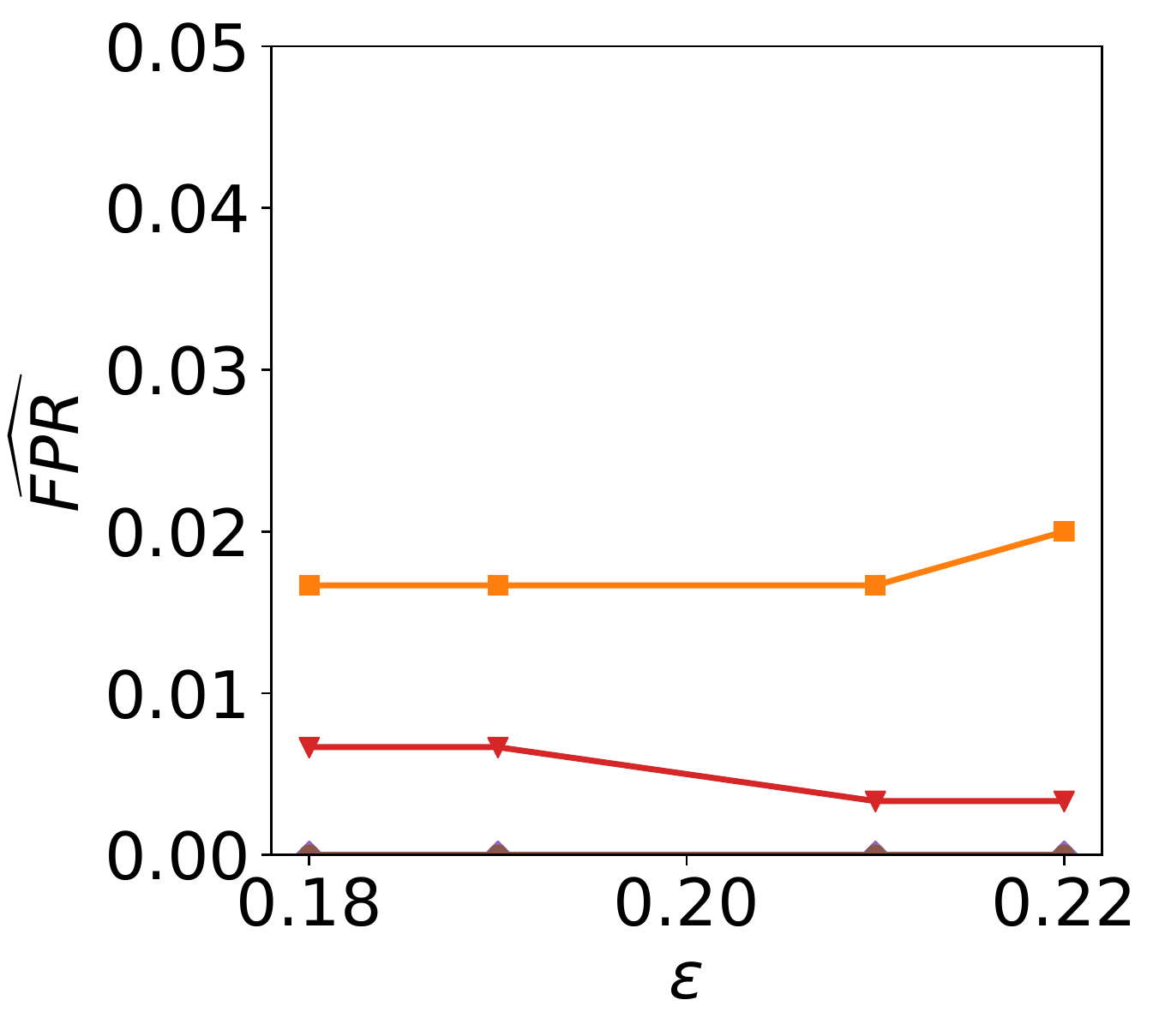}}
    \caption{RBM Experiment. Rejection rates (estimated from 300 trials) for the six tests with $\alpha = 0.05$ is shown.}
    \label{fig:appen_ex_rbm}
\end{figure}
\begin{figure}[ht]
   \centering
   \includegraphics[width=0.75\textwidth]{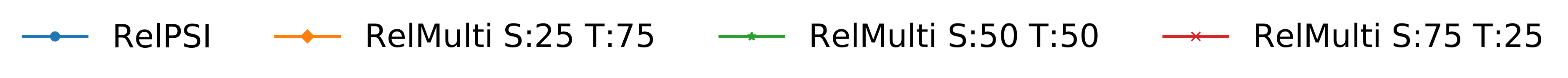}
  \subfloat[$\widehat{\mathrm{TPR}}$ of $\widehat{\mathrm{MMD}}^2_u$\label{fig:ex2_mmd}]{
        \includegraphics[width=.25\textwidth]{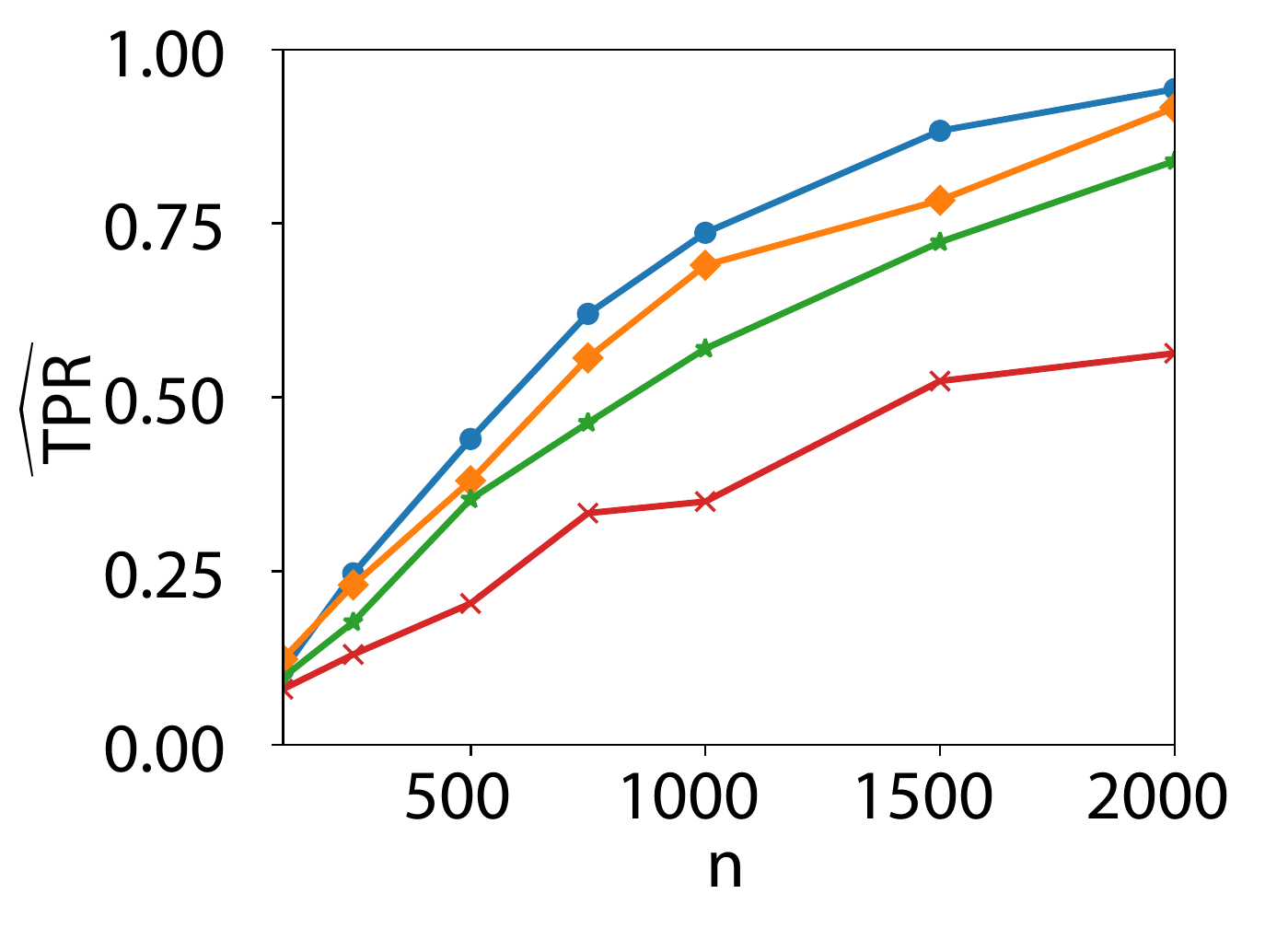}}
  \subfloat[Mixture Problem\label{fig:ex2_mix_vs}]{
       \includegraphics[width=.25\textwidth]{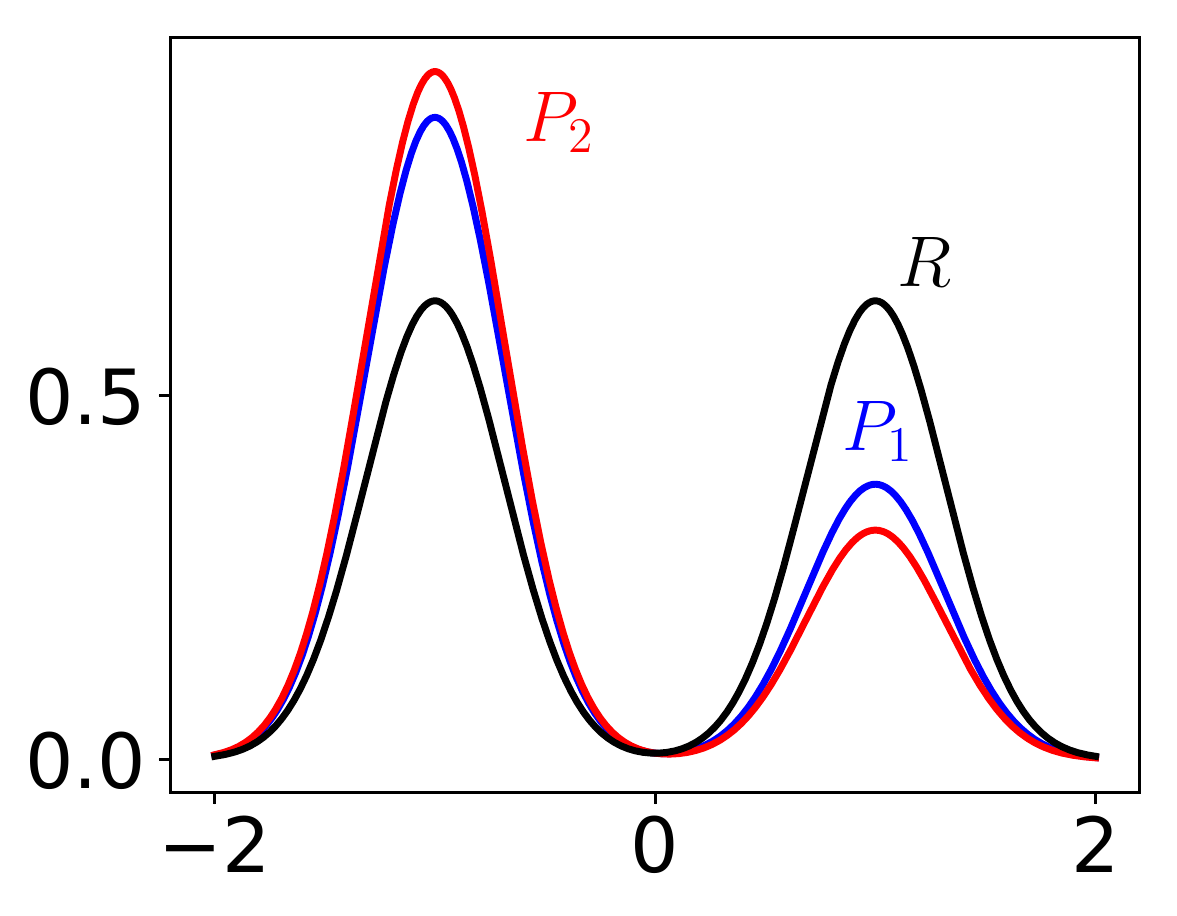}}
  \subfloat[$\widehat{\mathrm{TPR}}$ of $\widehat{\mathrm{KSD}}^2_u$\label{fig:ex2_ksd}]{
       \includegraphics[width=.25\textwidth]{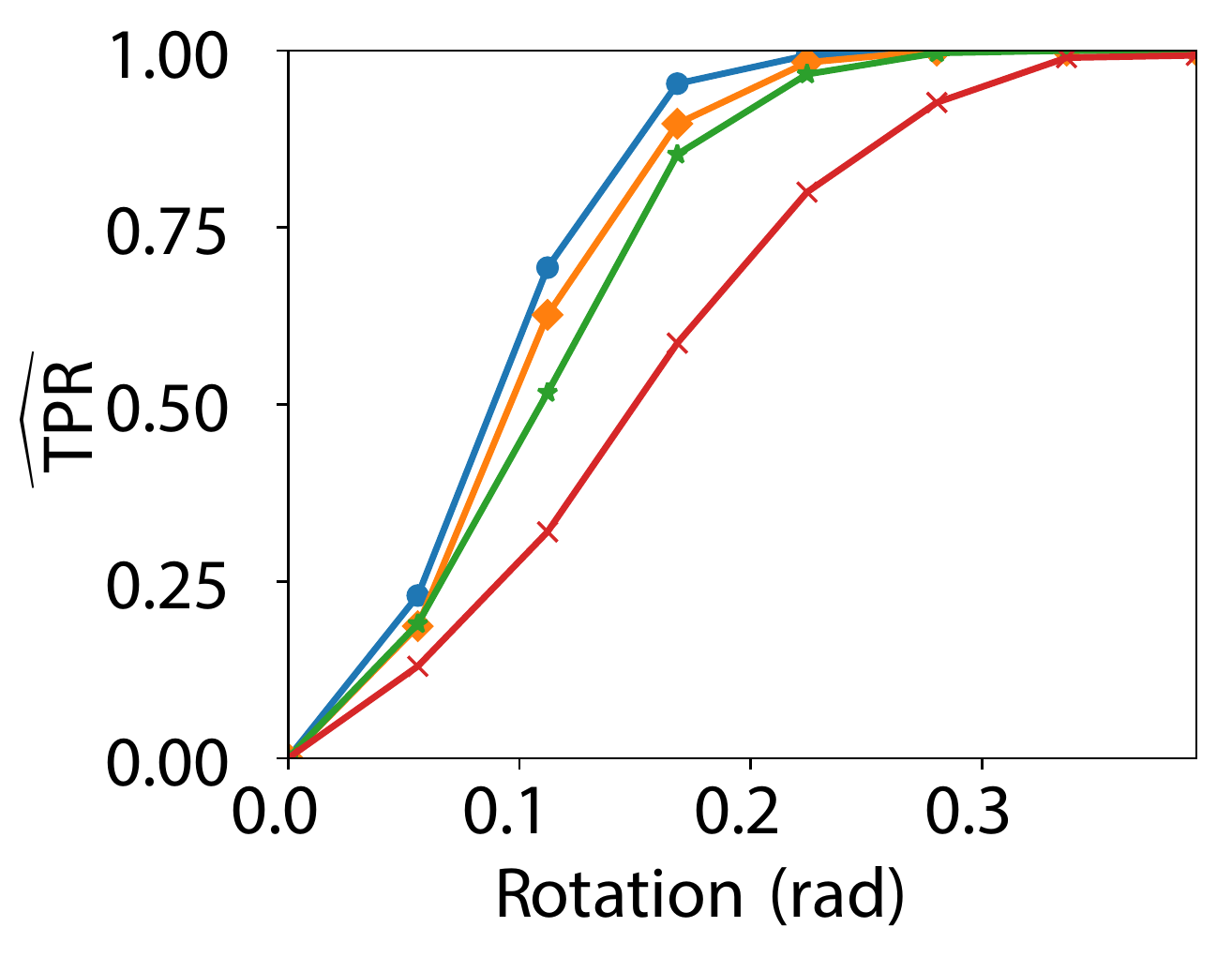}}
   \subfloat[Rotation Problem\label{fig:ex2_rot_vs}]{
       \includegraphics[width=.25\textwidth]{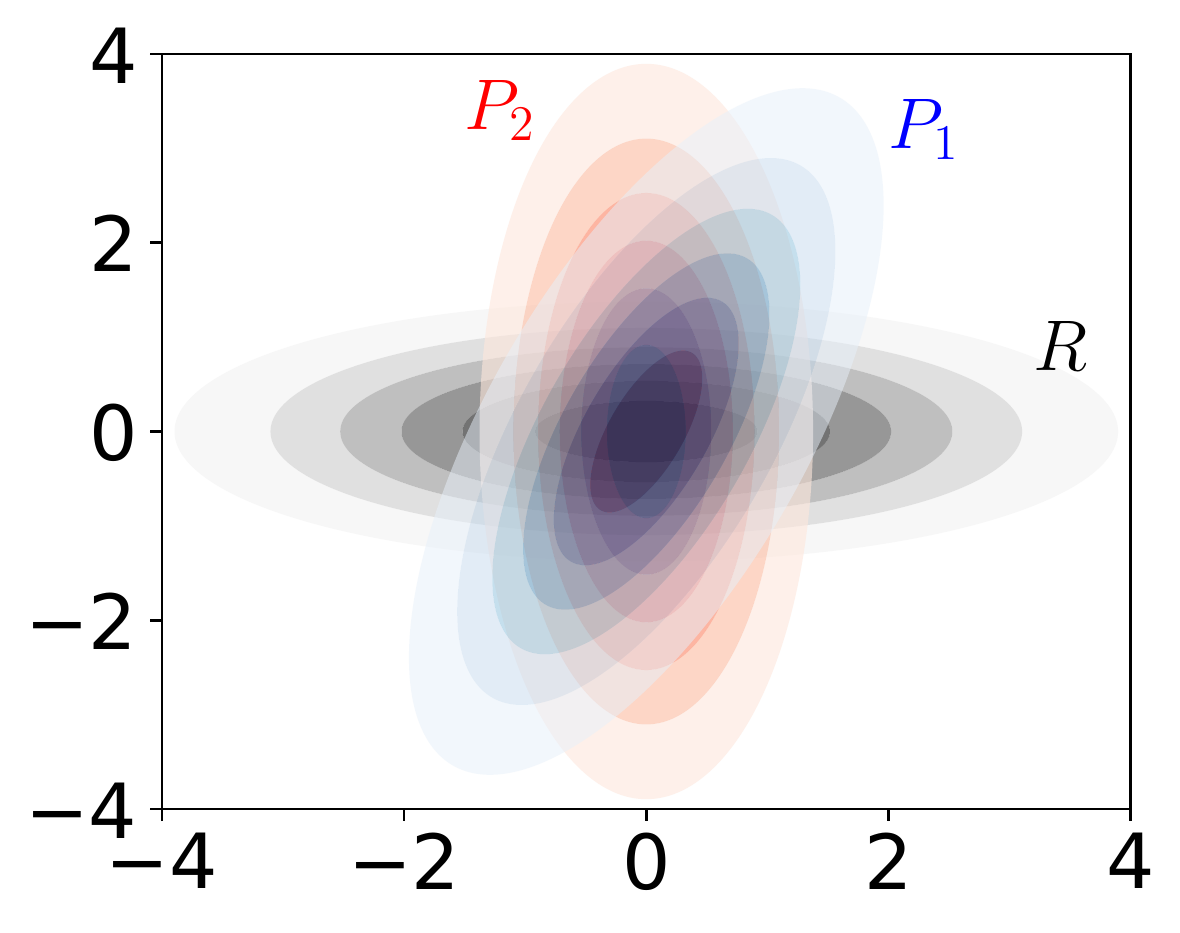}}
    \caption{$l=2$: The empirical true positive rate $\widehat{\mathrm{TPR}}$ of $\widehat{\mathrm{MMD}}^2_u$ (a) for the mixture problem of (b). We show empirical $\widehat{\mathrm{TPR}}$ of $\widehat{\mathrm{KSD}}^2_u$ (c) for rotation problem of (d). S:$a\%$ T:$b\%$ indicates that $a\%$ of the original dataset is used for selection and $b\%$ of the dataset used for testing.}
    \label{fig:ex2_tpr}
\end{figure}
% \begin{subfig}
%     \centering
%     \includegraphics[width=1.\textwidth]{graphs/power.pdf}
%     \caption{The true positive rates of tests using the PSI and data-splitting
%     approaches S denotes the percentage of the data used for selection and T
% the percentage for testing. In (a) we increase $n$ for mixture problem of (b).
% And (c) fixed $n$ while varying $\mu$ for the rotating problem of (d).
% \wjsay{Figures too tall? But up to you. Make sure to save each subfigure separately. }
% }
%     \label{fig:power}
% \end{subfig}

\subsection{TPR experiment}
\label{sec:appen_tpr}
For this experiment, our goal is to empirically
evaluate and validate Theorem \ref{theorem:tpr} where $l=2$. For some sufficiently large
$n$, it states that the \textrm{TPR} of \textrm{RelPSI} will be an upper bound for the \textrm{TPR}
of \relmul{} (for both \textrm{MMD} and \textrm{KSD}). We consider the following two
synthetic problems:
\begin{enumerate}
    \item{\textit{Mixture of Gaussian}:
The candidate models and unknown distribution are $1$-d mixture of Gaussians where $M(\rho) = \rho\mathcal{N}(1,1) +
(1-\rho)\mathcal{N}(-1,1)$ with mixing portion $\rho \in (0,1)$. We set the reference to be
$R=M(0.5)$, and two candidate models to $P_1=M(0.7)$ and $P_2=M(0.75)$
(see Figure \ref{fig:ex2_mix_vs}). In this case, $P_1$ is closer to the reference distribution but only by a small amount. In this problem, we apply MMD and report the behaviour of the test as $n$ increases.}

    \item{\textit{Rotating Gaussian:} The two candidate models and our reference distributions are $2$-d Gaussian distributions that differ by rotation (see Figure \ref{fig:ex2_rot_vs}). We fix the sample size to $n=500$. Instead, we
rotate the Gaussian distribution $P_1$ away from $P_2$ such that $P_1$ continues to get closer to 
the reference $R$ with each rotation. They are initially the same distribution but $P_1$ becomes a closer relative fit (with each rotation). In this problem, we apply KSD and report the empirical TPR as the Gaussian rotates and becomes an easier problem.}
\end{enumerate}

For each problem we consider three possible splits of the
data: $25\%$, $50\%$, $25\%$ of the original samples for selection (and the rest for testing). Both problems use a Gaussian kernel with bandwidth set to $1$. The overall results are shown in Figure \ref{fig:ex2_tpr}.

 In Figure \ref{fig:ex2_mmd}, we plot the $\widehat{\mathrm{TPR}}$ for \textrm{RelPSI-MMD} and \relmulmmd{} for the Mixture of Gaussian problem. \textrm{RelPSI} performs the best with the highest empirical $\widehat{\mathrm{TPR}}$ confirming with Theorem \ref{theorem:tpr}. The next highest is \relmul{} that performs a S:$25\%$ T:$75\%$ selection test split. The worst performer is the \relmul{} with S:$75\%$ T:$25\%$ selection test split which can be explained by noting that most of the data has been used in selection, there is an insufficient amount of remaining data points to reject the hypothesis. The same behaviour can be observed in Figure \ref{fig:ex2_ksd} for $\widehat{\ksd}^2_u$. Overall, this experiment corroborates with our theoretical results that TPR of RelPSI will be higher in population.

%% file: appen_power.tex
\section{Performance analysis for two models}

In this section, we analyse the performance of our two proposed methods:
\relpsi{} and \relmul{} for $l=2$ candidate models. 
We begin by computing the
probability that we select the best model correctly (and selecting
incorrectly). Then provide a closed form formula for computing the rejection
threshold, and from this we were able to characterize the probability of
rejection and proof our theoretical result.

\subsection{Cumulative distribution function of a truncated normal}

The cumulative distribution function (CDF) of a truncated normal is given by
\[
\Psi(x\,|\:\mu,\sigma,\mathcal{V}^-,\mathcal{V}^+)=\frac{{\Phi(\frac{{x-\mu}}{\sigma})-\Phi(\frac{{\mathcal{V}^--\mu}}{\sigma})}}{\Phi(\frac{{\mathcal{V}^+-\mu}}{\sigma})-\Phi(\frac{{\mathcal{V}^--\mu}}{\sigma})},
\]
where $\Phi$ is the CDF of the standard normal distribution \cite[Section 3.3]{burkardt2014truncated}.

%----------------
\subsection{Characterizing the selection event}

Under the null and alternative hypotheses, for both $\widehat{\mathrm{MMD}}_{u}^{2}(P,R)$
and $\widehat{\mathrm{{KSD}}}_{u}^{2}(P,R)$, the test statistic is asymptotically
normal i.e., for a sufficiently large $n$, we have
\begin{equation*}
\sqrt{{n}}\big[\hat{D}(P_{2},R)-\hat{D}(P_{1},R) -\mu \big]\sim\mathcal{{N}}(0,\sigma^2),
\end{equation*}
where $\mu:=D(P_{2},R)-D(P_{1},R)$ is the population difference and $D(\cdot,\cdot)$
can be either $\mathrm{MMD}^{2}$ or $\mathrm{KSD}^{2}$.
The probability of selecting the model $P_1$, i.e., $P_{\hat{J}}=P_1$, is equivalent to the probability of observing $\hat{D}(P_{1},R) < \hat{D}(P_{2},R)$. The following lemma derives this quantity.
\begin{lemma}
\label{lemma:selection_event}
    Given two models $P_1$ and $P_2$, and the test statistic
    $\sqrt{{n}}[\hat{D}(P_{2},R)-\hat{D}(P_{1},R)]$ such that 
$\sqrt{{n}}\big[\hat{D}(P_{2},R)-\hat{D}(P_{1},R) -\mu
\big] \stackrel{d}{\to} \mathcal{{N}}(0,\sigma^2)$, where $\mu:=D(P_{2},R)-D(P_{1},R)$,
 then the probability that we select $P_1$ as the reference is 
    \begin{equation*}
        \mathbb{P}(P_{\hat{J}}=P_{1}) =
        \mathbb{P}(\hat{D}(P_{1},R)<\hat{D}(P_{2},R))
        \approx \Phi\left(\frac{\sqrt{n}\mu}{\sigma}\right).
    \end{equation*}
    It follows that     $\mathbb{P}(P_{\hat{J}} = P_{2})
    \approx \Phi(-\frac{\sqrt{{n}}{\mu}}{\sigma})$.
    %\wjsay{Unclear. Do you mean the probability of selection $P_2$ as the reference when $P_1$ is the best?}
\end{lemma}
\begin{proof} 
    For some sufficiently large $n$, we have
    \begin{align*}
\mathbb{P}(P_{\hat{J}}=P_{1}) & =\mathbb{P}(\sqrt{{n}}\hat{D}(P_{1},R)<\sqrt{{n}}\hat{D}(P_{2},R))\\
 %& =\mathbb{P}(\sqrt{{n}}\hat{D}(P_{1},R)-\sqrt{{n}}\hat{D}(P_{2},R)<0)\\
 & =\mathbb{P}(\sqrt{{n}}[\hat{D}(P_{2},R)-\hat{D}(P_{1},R)]>0)\\
 & \approx1-\Phi(-\frac{\sqrt{n}{\mu}}{\sigma})=\Phi\left(\frac{\sqrt{n}\mu}{\sigma}\right),
\end{align*}
and
%\begin{align*}
$\mathbb{P}(P_{\hat{J}}=P_{2})  = 1- \mathbb{P}(P_{\hat{J}}=P_{1}) 
\approx \Phi(-\frac{\sqrt{{n}}{\mu}}{\sigma}).$
%\end{align*}
\end{proof}
It can be seen that as $n$ gets larger the selection procedure is more likely to select the correct model.

\subsection{Truncation points of \relpsi{}}
To study the performance of \relpsi{}, it is necessary to
characterize the truncation points in the polyhedral lemma (Theorem \ref{theorem:poly}). In the case
of two candidate models, the truncation points are simple as shown in Lemma \ref{lem:truncation}.
\begin{lemma}
\label{lem:truncation}
Consider two candidate models $P_1$ and $P_2$ and the selection algorithm described in Section
\ref{sec:selection}, with the test statistic
$\sqrt{n}[\hat{D}(P_2,R)-\hat{D}(P_1,R)]$. The
upper truncation point $\mathcal{V}^+$ and lower truncation point
$\mathcal{V}^-$ (see Theorem \ref{theorem:poly}) when the selection procedure
observes $\hat{D}(P_1,R) < \hat{D}(P_2,R)$, i.e., $P_{\hat{J}} = P_1$, are
$$
\mathcal{V}^- =0,\ \mathcal{V}^+ =\infty.
$$
When the selection procedure observes $\hat{D}(P_2,R) < \hat{D}(P_1,R)$, i.e., $P_{\hat{J}} = P_2$, then the truncation points are 
$$
\mathcal{V}^- =-\infty, \ \mathcal{V}^+ =0.
$$ 
\end{lemma}
\begin{proof}
    If the selection procedure observes that $\hat{D}(P_1,R) < \hat{D}(P_2,R)$
    then $\hat{J}=1$ and our test statistic is
    $\sqrt{n}[\hat{D}(P_2,R)-\hat{D}(P_{\hat{J}},R)]= \sqrt{n}[\hat{D}(P_2,R)-\hat{D}(P_1,R)]= \bm\eta^\top\bm{z}$ 
    where $\bm\eta=\begin{pmatrix}
            -1 & 1
            \end{pmatrix}^\top$ and $\bm{z} = \sqrt{n}\begin{pmatrix}
            \hat{D}(P_1,R) \\
            \hat{D}(P_2,R)
            \end{pmatrix}$.
    Then the affine selection event can be written as $\bm{Az} \le \bm{b}$ where $\bm{A} =
        \begin{pmatrix}
            1 & -1
        \end{pmatrix}$
    and $\bm{b} = 0$.
    It follows from the definition of $\mathcal{V}^+$ and $\mathcal{V}^-$ (see Theorem \ref{theorem:poly}) that we have
    $\mathcal{V}^-=0$ and $\mathcal{V}^+=\infty$.
    
    A similar result holds for the case where the selection event observes
    $\hat{D}(P_2,R) < \hat{D}(P_1,R)$ (i.e., $\hat{J}=2$). The test statistic is 
    $\sqrt{n}[\hat{D}(P_1,R)-\hat{D}(P_2,R)]$. 
    The selection event can be
    described with 
    $\bm{A} =
        \begin{pmatrix}
            -1 & 1
        \end{pmatrix}$
    and $\bm{b} = 0$.
    Following from their definitions, we have $\mathcal{V}^-=-\infty$ and $\mathcal{V}^+=0$.
\end{proof}
%
%--------------------------------
\subsection{Test threshold}
Given a significance level $\alpha\in(0,1)$, the test threshold is defined to the $(1-\alpha)$-quantile
of the truncated normal for RelPSI, and normal for RelMulti.
The test threshold of the \relpsi{} is 
\begin{align*}
t^{\text{RelPSI}}(\alpha) & =\Psi^{-1}(1-\alpha\,|\,\mu=0,\sigma,\mathcal{V}^-, \mathcal{V}^+)\\
 & =\mu+\sigma\Phi^{-1}\bigg((1-\alpha)\Phi\left(\frac{{\mathcal{V}^+-\mu}}{\sigma}\right)+\alpha\Phi\left(\frac{{\mathcal{V}^--\mu}}{\sigma}\right)\bigg)\\
 & =\sigma\Phi^{-1}\bigg((1-\alpha)\Phi\left(\frac{{\mathcal{V}^+}}{\sigma}\right)+\alpha\Phi\left(\frac{{\mathcal{V}^-}}{\sigma}\right)\bigg),
\end{align*}
where $\Psi^{-1}(\cdot\,|\,\mu,\sigma,\mathcal{V}^-,\mathcal{V}^+)$ is the inverse of the
CDF of the truncated normal with mean $\mu$, standard deviation $\sigma$,
and lower and upper truncation points denoted $\mathcal{V}^-, \mathcal{V}^+$, and $\Phi^{-1}$ is the inverse of the CDF of the standard normal distribution. Note that under the
null hypothesis, $\mu \le 0$ (recall $\mu:= D(P_2,R) - D(P_1,R)$), we set $\mu = 0$ which results in a more conservative 
test for rejecting the null hypothesis.
Furthermore, we generally do not know $\sigma$; instead we use its plug-in estimator $\hat{{\sigma}}$.

Given two candidate models $P_{1}$ and $P_{2}$, the truncation
points $(\mathcal{V}^-, \mathcal{V}^+)$ are either (see Lemma \ref{lem:truncation}):
\begin{itemize}
    \item Case 1: $\mathcal{V}^-=0,\ \mathcal{V}^+=\infty$, or
    \item Case 2: $\mathcal{V}^-=-\infty,\ \mathcal{V}^+=0$.
\end{itemize}
The two cases result in different level-$\alpha$ rejection thresholds since the value of the rejection threshold is dependent on the truncation points.

For Case 1, the threshold is
\[
t^{\text{RelPSI}}_{1}(\alpha)=\hat{{\sigma}}\Phi^{-1}\left(1-\frac{{\alpha}}{2}\right).
\]
For Case 2, the threshold is
\[
t^{\text{RelPSI}}_{2}(\alpha)=\hat{{\sigma}}\Phi^{-1}\left(\frac{{1}}{2}-\frac{{\alpha}}{2}\right).
\]
Note that since $\Phi^{-1}(\cdot)$ is monotonically increasing, we have $t^{\text{RelPSI}}_{2}(\alpha)<0<t^{\text{RelPSI}}_{1}(\alpha)$.

For \relmul{}, the threshold is given by the $(1-\alpha)$-quantile
of the asymptotic null distribution which is a normal distribution (with the mean $\mu$ adjusted to 0):
\begin{align*}
t^{\text{RelMulti}}(\alpha) & = \hat{{\sigma}}\Phi^{-1}(1-\alpha).
\end{align*}
\subsection{Rejection probability}
%For the fixed hypothesis 
%$$
%H_{0}: D(P_1,R) \ge D(P_2,R) \,|\, P_{\hat{J}}\text{ is selected as the reference model},
%$$
%and the corresponding 
Consider the test statistic $\sqrt{n}\hat{\mu} :=
\sqrt{{n}}\big[\hat{D}(P_{2},R)-\hat{D}(P_{1},R)\big]$. 

\paragraph{\relpsi{}}
Depending on whether $\hat{J}=1$ or $\hat{J}=2$, 
the rejection probability for \relpsi{} is
given by
\begin{align*}
    & \mathbb{P}(\sqrt{{n}} \hat{\mu} >t^{\text{RelPSI}}_1(\alpha)\,|\,P_{\hat{J}}
= P_1)  \text{ or} \\
& \mathbb{P}(\sqrt{{n}} \hat{\mu}>t^{\text{RelPSI}}_2(\alpha)\,|\,P_{\hat{J}}
= P_2).
\end{align*}

Assume $n$ is sufficiently large. The rejection probability of \relpsi{} when
$P_{\hat{J}} = P_1$ is 
\begin{align*}
\mathbb{P}(\sqrt{{n}}\hat{\mu}>t^{\text{RelPSI}}_1(\alpha)\,|\,P_{\hat{J}}=P_{1}) & \approx1-\frac{\Phi(\frac{\hat{{\sigma}}\Phi^{-1}(1-\frac{{\alpha}}{2})-{\sqrt{{n}}\mu}}{\sigma})-\Phi(\frac{{-\sqrt{{n}}\mu}}{\sigma})}{1-\Phi(\frac{{-\sqrt{{n}}\mu}}{\sigma})}\\
 & \stackrel{(*)}{\approx} 1-\frac{\Phi(\Phi^{-1}(1-\frac{{\alpha}}{2})-\frac{{\sqrt{{n}}\mu}}{\sigma})-\Phi(\frac{{-\sqrt{{n}}\mu}}{\sigma})}{1-\Phi(\frac{{-\sqrt{{n}}\mu}}{\sigma})}. \stepcounter{equation}\tag{\theequation}\label{eq:relpsi_rej_correct}
\end{align*}
When $P_{\hat{J}} = P_2$, it is
\begin{align*}
\mathbb{P}(\sqrt{{n}}\hat{\mu}>t^{\text{RelPSI}}_2(\alpha)\,|\,P_{\hat{J}}=P_{2}) & \approx1-\frac{\Phi(\frac{\hat{{\sigma}}\Phi^{-1}(\frac{{1}}{2}-\frac{{\alpha}}{2})-\sqrt{{n}}{\mu}}{\sigma})}{\Phi(\frac{-{\sqrt{{n}}\mu}}{\sigma})}\\
 & \stackrel{(*)}{\approx}1-\frac{\Phi(\Phi^{-1}(\frac{{1}}{2}-\frac{{\alpha}}{2})-\frac{\sqrt{{n}}{\mu}}{\sigma})}{\Phi(\frac{-{\sqrt{{n}}\mu}}{\sigma})}. \stepcounter{equation}\tag{\theequation}\label{eq:relpsi_rej_incorrect}
\end{align*}
\paragraph{\relmul}
 For RelMulti, it is
 $\mathbb{P}(\sqrt{{n}} \hat{\mu} >t^{\text{RelMulti}}(\alpha))$.
The rejection probability of RelMulti is
\begin{align*}
\mathbb{P}(\sqrt{{n}}\hat{\mu}>t^{\text{RelMulti}}(\alpha)) & \approx1-\Phi(\frac{{\hat{{\sigma}}\Phi^{-1}(1-\alpha)-\sqrt{{n}}\mu}}{\sigma})\\
& \stackrel{(*)}{\approx}1 - \Phi\left(\Phi^{-1}(1-\alpha)-\frac{{\sqrt{{n}}\mu}}{\sigma} \right),
 \stepcounter{equation}\tag{\theequation}
 \label{eq:relmulti_rej}
\end{align*}
where we use the fact that $\sqrt{n}(\hat{\mu} - \mu)  \stackrel{d}{\to}
\mathcal{N}(0,\sigma^2)$. 
We note that at $(*)$ we use the fact that as $n \to \infty$, $\hat{\sigma}$
converges to $\sigma$ in probability.
%
%---------------------------------
\subsection{True positive rates of \relpsi{} and \relmul{}}
For the remainder of the section, we assume without loss of generality that
$D(P_{1},R)<D(P_{2},R)$, i.e., $P_1$ is the better model, so we have $\mu = D(P_2, R) - D(P_1, R)>0$.
%Then for sufficiently large $n$, we want to reject $D(P_{1},R)\ge D(P_{2},R)$. 

\paragraph{\relpsi}
The TPR (for \relpsi{}) is given by
\begin{align*}
\text{{TPR}}_{\text{RelPSI}} & =\mathbb{E}\Bigg[\frac{\text{Number of True Positives assigned Positive}}{\underbrace{{\text{Number of True Positives}}}_{=1}}\Bigg]\\
& \stackrel{(a)}{=} \mathbb{P}(\text{decide that } P_2 \text{ is worse}) \\
& = \mathbb{P}(\text{decide that } P_2 \text{ is worse} \mid P_1 \text{ is
selected} )  \mathbb{P}(P_1 \text{ is selected}) \\ 
& \phantom{=} +\mathbb{P}(\text{decide that } P_2
\text{ is worse} \mid P_2 \text{ is selected} )  \mathbb{P}(P_2 \text{ is
selected})\\
& \stackrel{(b)}{=} \mathbb{P}(\text{decide that } P_2 \text{ is worse} \mid P_1 \text{ is
selected} )  \mathbb{P}(P_1 \text{ is selected}) \\ 
 %& \mathbb{=E}[\mathbb{I}(\text{Reject } H_{0}:D(P_{1},R)\ge D(P_{2},R))]\\
 %& =\mathbb{P}[\text{Reject } H_{0}:D(P_{1},R)\ge D(P_{2},R)]\\
& =\mathbb{P}(\text{Reject }H_{0}:D(P_{1},R)\ge D(P_{2},R)\,|\,P_{1}\ \text{is selected})\mathbb{P}(P_{1}\ \text{is selected}) \\
& =\mathbb{P}(\sqrt{n}\hat{\mu}>t^{\text{RelPSI}}_1({\alpha})\,|\,P_{\hat{J}}=P_{1})P(P_{\hat{J}}=P_{1}),
 %& +\mathbb{P}[\text{Reject } H_{0}:D(P_{1},R)\ge D(P_{2},R)\,|\,P_{2}\ \text{is selected}]\mathbb{P}(P_{2}\ \text{is selected}),
 \end{align*}
where we note that at $(a)$, deciding that $P_2$ is worse than $P_1$ is the
same as assigning positive to $P_2$. The equality at $(b)$ holds due to the
design of our procedure that only tests the selected reference against other
candidate models to decide whether they are worse than the reference model.  
By design, we will not test the selected reference model against itself.
So, $\mathbb{P}(\text{decide that } P_2 \text{ is worse} \mid P_2 \text{ is
selected} ) = 0$.
%
%\begin{align*}
%\text{{TPR}}_{\text{RelPSI}} & =\mathbb{P}[\text{Reject } H_{0}:D(P_{1},R)\ge D(P_{2},R)\,|\,P_{1}\ \text{is selected}]\mathbb{P}(P_{1}\ \text{is selected})\\
%\end{align*}
%
Using Equation \ref{eq:relpsi_rej_correct} and Lemma \ref{lemma:selection_event}, we have 
\begin{align*}
\text{{TPR}}_{\text{RelPSI}}
& \approx\Bigg[1-\frac{\Phi(\Phi^{-1}(1-\frac{{\alpha}}{2})-\frac{{\sqrt{{n}}\mu}}{\sigma})-\Phi(\frac{{-\sqrt{{n}}\mu}}{\sigma})}{1-\Phi(\frac{{-\sqrt{{n}}\mu}}{\sigma})}\Bigg]\Bigg[1-\Phi(-\frac{\sqrt{{n}}{\mu}}{\sigma})\Bigg]\\
& =1-\Phi(-\frac{\sqrt{{n}}{\mu}}{\sigma})-\Phi(\Phi^{-1}(1-\frac{{\alpha}}{2})-\frac{{\sqrt{{n}}\mu}}{\sigma})+\Phi(\frac{{-\sqrt{{n}}\mu}}{\sigma})\\
& =1-\Phi(\Phi^{-1}(1-\frac{{\alpha}}{2})-\frac{{\sqrt{{n}}\mu}}{\sigma}) \\
& =\Phi \left(\frac{{\sqrt{{n}}\mu}}{\sigma}-\Phi^{-1}(1-\frac{\alpha}{2}) \right).\stepcounter{equation}\tag{\theequation}\label{eq:tpr_psi}
\end{align*}

\paragraph{\relmul}
For \relmul{}, we perform data splitting to create independent sets of our data
for testing and selection. Suppose
we have $n$ samples and a proportion of samples to be used for testing $\rho\in(0,1)$, we have $m_{1}=\rho n$ samples used 
for testing and $m_{0}={n(1-\rho)}$ samples for selection. Then TPR for \relmul{} can be
derived as
\begin{align*}
\text{TPR}_{\text{RelMulti}}
&= \mathbb{P}(\text{decide that } P_2 \text{ is worse} \mid P_1 \text{ is
selected} )  \mathbb{P}(P_1 \text{ is selected}) \\ 
&= \mathbb{P}(\text{decide that } P_2 \text{ is worse} )  \mathbb{P}(P_1 \text{ is selected}) \\ 
&=\mathbb{P}(\sqrt{m_{1}}\hat{\mu}>t^{\text{RelMulti}}(\alpha))\mathbb{P}(\sqrt{m_{0}} \hat{\mu} >0).
\end{align*}

Using Equation \ref{eq:relmulti_rej} (with $n\rho$ samples) and Lemma
\ref{lemma:selection_event} (with $n(1-\rho)$ samples), we have
\begin{align*}
\text{TPR}_{\text{RelMulti}}
%& \approx\Bigg[1-\Phi(\frac{t^{\text{RelMulti}}(\alpha)-\sqrt{n\rho}\mu}{\sigma})\Bigg]\Bigg[1-\Phi(-\frac{\sqrt{n(1-\rho)}\mu}{\sigma})\Bigg]\\
%& =\Bigg[1-\Phi(\frac{\hat{\sigma}\Phi^{-1}(1-\alpha)-\sqrt{n\rho}\mu}{\sigma})\Bigg]\Phi(\frac{\sqrt{n(1-\rho)}\mu}{\sigma})\\
& \approx \Bigg[1-\Phi(\Phi^{-1}(1-\alpha)-\frac{\sqrt{n\rho}\mu}{\sigma})\Bigg]\Phi(\frac{\sqrt{n(1-\rho)}\mu}{\sigma})\stepcounter{equation}\tag{\theequation}
\label{eq:tpr_split}.
\end{align*}
%
\begin{comment}
Similarly, \textrm{FPR} is given by,
%
\begin{align*}
\text{FPR}_{\text{RelMulti}} & =\mathbb{P}(\sqrt{m_{1}}\hat{\mu}>t^{\text{RelMulti}}({\alpha}))\mathbb{P}(\sqrt{m_{0}}[\hat{D}(P_2,R)-\hat{D}(P_1,R)]<0).
\end{align*}
Using 
\begin{align*}
 & \approx\Bigg[1-\Phi(\frac{u(\alpha)-\sqrt{n\rho}\mu}{\sigma})\Bigg]\Bigg[\Phi(-\frac{\sqrt{n(1-\rho)}\mu}{\sigma})\Bigg]\\
 & =\Bigg[1-\Phi(\Phi^{-1}(1-\alpha)-\frac{\sqrt{n\rho}\mu}{\sigma})\Bigg]\Phi(-\frac{\sqrt{n(1-\rho)}\mu}{\sigma}).
\end{align*}
\end{comment}
We note that both $\mathrm{TPR}_\relpsi \to 1$ and $\mathrm{TPR}_\relmul \to 1$, as $n \to \infty$.
We are ready to prove Theorem \ref{theorem:tpr}. We first recall the theorem from the main text:
\tpr*
\begin{proof}
    \label{proof:tpr} 
    Assume without loss of generality that $D(P_1, R) < D(P_2,R)$, i.e., $P_1$ is the best model,
    and $\mu:= D(P_2,R) - D(P_1,R) > 0$ is the population difference of two
    discrepancy measures (which can be either \textrm{MMD} or \textrm{KSD}) and
    $\sigma$ is the standard deviation of our test statistic. 
    Since $n > N$, we have
    \begin{align*}
        \frac{{\sqrt{{n}}\mu}}{\sigma}(1-\sqrt{{\rho}}) & \ge\Phi^{-1}(1-\frac{{\alpha}}{2})\\
        \stackrel{(a)}{\implies} \frac{{\sqrt{{n}}\mu}}{\sigma}-\frac{{\sqrt{{n\rho}}\mu}}{\sigma} & \ge\Phi^{-1}(1-\frac{\alpha}{2})-\overbrace{{\Phi^{-1}(1-\alpha)}}^{\ge0}\\
        \equiv \frac{{\sqrt{{n}}\mu}}{\sigma}-\Phi^{-1}(1-\frac{\alpha}{2}) & \ge\frac{{\sqrt{{n\rho}}\mu}}{\sigma}-\Phi^{-1}(1-\alpha)\\
        \stackrel{(b)}{\implies} \Phi \left(\frac{{\sqrt{{n}}\mu}}{\sigma}-\Phi^{-1}(1-\frac{\alpha}{2}) \right) & \ge \Phi \left(\frac{{\sqrt{n{\rho}}\mu}}{\sigma}-\Phi^{-1}(1-\alpha) \right),
    \end{align*}
    where at $(a)$, we have $\Phi^{-1}(1-\alpha) \ge 0$ because $\alpha \in [0,
    1/2]$. At $(b)$, we use the fact that $a \mapsto \Phi(a)$ is increasing. We note that the left hand side is the same as Equation \eqref{eq:tpr_psi} and it follows that
    \begin{align*}
         %\mathbb{P}[\text{RelPSI rejects } H_{0}\, |\, P_{1}\ \text{is selected}]\mathbb{P}(P_{\hat{J}} = P_{1})
        \mathrm{TPR}_\relpsi
         & \gtrapprox \Phi \left(\frac{{\sqrt{n{\rho}}\mu}}{\sigma}-\Phi^{-1}(1-\alpha)\right) 
         \underbrace{{\Phi(\frac{{\sqrt{n(1-\rho)}\mu}}{\sigma})}}_{\in (0,1) }\\
         & = \Bigg[1-\Phi(\Phi^{-1}(1-\alpha)-\frac{{\sqrt{{n}\rho}\mu}}{\sigma})\Bigg]\Bigg[\Phi(\frac{{\sqrt{{n(1-\rho)}}\mu}}{\sigma})\Bigg]\\
         & \stackrel{(c)}{\gtrapprox} \mathrm{TPR}_\relmul, 
         %\mathbb{P}[\text{RelMulti rejects } H_{0}\, |\, P_{1}\ \text{is selected}]\mathbb{P}(P_{\hat{J}} = P_{1})\\
        %\text{\text{{TPR}}}_{\text{{RelPSI}}} & \ge\text{\text{{TPR}}}_{\text{{RelMulti}}},
    \end{align*}
    %where $H_0:D(P_1, R) \ge D(P_2,R)$, the RHS is Equation~\eqref{eq:tpr_psi} and the LHS is Equation~\eqref{eq:tpr_split}.
    where at $(c)$ we use Equation \eqref{eq:tpr_split}.
\end{proof}
%-----------------------------
\section{Test consistency}
In this section, we describe and prove the consistency result of our proposal
\relpsi{} and \relmul{} for both MMD and KSD.

\label{proof:consistency}
\mmdcons*
\begin{proof}
    %\wjsay{Remove all $\widehat{\cdot}$ above $\eta^\top z$.}
    Let $\hat{t}_\alpha$ and $t_\alpha$ be $(1-\alpha)$ quantiles of distributions
    $\mathcal{TN}(\bm{0}, \bm{\eta}^\top\hat{\bm\Sigma}\bm\eta, \mathcal{V}^-, \mathcal{V}^+)$
    and $\mathcal{TN}(\bm{0}, \bm{\eta}^\top \bm\Sigma\bm\eta, \mathcal{V}^-, \mathcal{V}^+ )$
    respectively. Given that $\hat{\bm\Sigma} \overset{p}{\rightarrow} \bm\Sigma$, 
    $\mathcal{TN}(\bm{0}, \bm{\eta^\top\hat\Sigma\eta}, \mathcal{V}^-, \mathcal{V}^+)$
    converges to $\mathcal{TN}(\bm{0}, \bm{\eta^\top\Sigma\eta}, \mathcal{V}^-, \mathcal{V}^+)$
    in probability, hence, $\hat{t}_\alpha$ converges to $t_\alpha$. Note that $\hat{t}_\alpha$ is random and is determined by which model is selected to be $P_{\hat{J}}$ (which changes the truncation points $\mathcal{V}^-$ and $\mathcal{V}^+$).
    
    Under $H_0:\bm{\eta}^\top\bm{\mu} \le 0\,|\,P_{\hat{J}} \text{ is selected}$, for some sufficiently large $n$ the rejection rate is 
    \begin{align*}
        \lim_{n\rightarrow \infty}\Prob_{H_0}(\bm{\eta}^\top\bm{z}>\hat{t}_\alpha)
&=\lim_{n\rightarrow \infty}\Prob_{H_0}(\bm{\eta}^\top\bm{z}>t^{\text{RelPSI}}_1({\alpha})\, |\, P_{\hat{J}} = P_1) \Prob(P_{\hat{J}}=P_1) \\
        &+ \lim_{n\rightarrow \infty}\Prob_{H_0}(\bm{\eta}^\top\bm{z}>t^{\text{RelPSI}}_2({\alpha}) | P_{\hat{J}} = P_2) \Prob(P_{\hat{J}}=P_2).
    \end{align*}
    Using Lemma \ref{lemma:selection_event} with Equation \ref{eq:relpsi_rej_correct} and Equation \ref{eq:relpsi_rej_incorrect}, we have
    \begin{align*}
        \lim_{n\rightarrow \infty}\Prob_{H_0}(\bm{\eta}^\top\bm{z}>\hat{t}_\alpha)
        &= 1-\lim_{n\rightarrow \infty}\Phi(\Phi^{-1}(1-\frac{{\alpha}}{2})-\frac{{\sqrt{{n}}\bm{\eta}^\top\bm{\mu} }}{\sigma})\\
        &+ \lim_{n\rightarrow \infty}\Phi(-\frac{\sqrt{{n}}\bm{\eta}^\top\bm{\mu} }{\sigma})-\lim_{n\rightarrow \infty}\Phi(\Phi^{-1}(\frac{{1}}{2}-\frac{{\alpha}}{2})-\frac{\sqrt{{n}}{\bm{\eta}^\top\bm{\mu}}}{\sigma}).\\
        &\le 1-\Phi(\Phi^{-1}(1-\frac{{\alpha}}{2}))\\
        &+ \frac{1}{2}-\Phi(\Phi^{-1}(\frac{{1}}{2}-\frac{{\alpha}}{2}))\\
        &\le \alpha.
    \end{align*}

    Under $H_1:\bm{\eta}^\top\bm{\mu} > 0\,|\,P_{\hat{J}} \text{ is selected}$, similarly to $H_0$ we have
    \begin{align*}
        \lim_{n\rightarrow \infty} \Prob_{H_1}(\bm{\eta}^\top\bm{z} > \hat{t}_\alpha)
        &= \lim_{n\rightarrow \infty} \Prob_{H_1}(\bm{\eta}^\top\bm{z}>\hat{t}_\alpha | P_{\hat{J}} = P_1) \Prob(P_{\hat{J}}=P_1) \\
        &+ \lim_{n\rightarrow \infty} \Prob_{H_1}(\bm{\eta}^\top\bm{z}>\hat{t}_\alpha | P_{\hat{J}} = P_2) \Prob(P_{\hat{J}}=P_2) \\
        &= 1-\lim_{n\rightarrow \infty}\Phi(\Phi^{-1}(1-\frac{{\alpha}}{2})-\frac{{\sqrt{{n}}\bm{\eta}^\top\bm{\mu}}}{\sigma})\\
        &+ \lim_{n\rightarrow \infty} \Phi(-\frac{\sqrt{{n}}\bm{\eta}^\top\bm{\mu}}{\sigma})-
        \lim_{n\rightarrow \infty}\Phi(\Phi^{-1}(\frac{{1}}{2}-\frac{{\alpha}}{2})-\frac{\sqrt{{n}}{\bm{\eta}^\top\bm{\mu}}}{\sigma}),
    \end{align*}
    where $\bm{\eta}^\top\bm{\mu}$ is the population difference of the two discrepancy measures,
    and 
    $\sigma$ the standard deviation.
    
    Since the alternative hypothesis is true, i.e., $\bm{\eta}^\top\bm{\mu} > 0$, we have
        $\lim_{n\rightarrow \infty} \Prob_{H_1}(\bm{\eta}^\top\bm{z} > \hat{t}_\alpha) = 1.$
\end{proof}
\begin{theorem}[Consistency of \textrm{RelPSI-KSD}]
\label{theorem:consksd}
    Given two models $P_1$, $P_2$ and reference distribution $R$ (which are all distinct). Let $\hat{\bm\Sigma}$ be the covariance matrix defined in Theorem \ref{theorem:ksd}
and $\bm\eta$ be defined such that $\bm{\eta}^\top\bm{z} = \sqrt{n}[\bighat{\ksd}^2_u(P_1,R) - \bighat{\ksd}^2_u(P_2,R)]$.
    Suppose that the threshold $\hat{t}_\alpha$ is the $(1-\alpha)$-quantile of the distribution
of $\mathcal{TN}(\bm{0}, \bm{\eta}^\top\hat{\bm\Sigma}\bm\eta, \mathcal{V}^-, \mathcal{V}^+)$ where
    $\mathcal{V}^+$ and $\mathcal{V}^-$ is defined in Theorem \ref{theorem:poly}.
Under
    $H_{0}: \bm{\eta}^\top\bm{\mu} \le 0 \,|\,P_{\hat{J}} \text{ is selected}$, the asymptotic type-I error is bounded above by $\alpha$.
    Under $H_{1}:\bm{\eta}^\top\bm{\mu} > 0\,|\,P_{\hat{J}} \text{ is selected}$, we have
    $\Prob(\bm{\eta}^\top\bm{z}> \hat{t}_\alpha) \rightarrow 1$ as $n \rightarrow \infty$.
\end{theorem}